\documentclass[10pt,journal,compsoc]{IEEEtran}
\usepackage{amsfonts}
\usepackage{bm}

\ifCLASSOPTIONcompsoc
  \usepackage[nocompress]{cite}
\else
  \usepackage{cite}
\fi
\ifCLASSINFOpdf
\else
\fi
\usepackage{amsmath, amssymb}
\usepackage{amsthm}
\usepackage{xcolor}
\usepackage{algorithm}
\usepackage{algorithmicx}

\usepackage{titletoc}
\usepackage{tikz}
\usepackage[framemethod=tikz]{mdframed}
\usepackage{thm-restate}
\usepackage{booktabs}
\usepackage{hyperref}
\usepackage{subfigure} 
\usepackage{multirow}
\usepackage{ragged2e}
\usepackage{soul}
\hypersetup{
	colorlinks=true,
	linkcolor=blue,
	urlcolor=orange,
	citecolor=orange,
}

\newtheorem{lemma}{Lemma}
\newtheorem{definition}{Definition}
\newtheorem{proposition}{Proposition}
\newtheorem{remark}{Remark}
\newtheorem{corollary}{Corollary}

\newtheorem{assumption}{Assumption}

\newtheorem*{sketch}{Proof Sketch}

\newcommand{\atopk} {\mathsf{AUTKC}}

\newcommand{\topk} {\mathsf{TOP}\text{-}\mathsf{k}}

\newcommand{\size}[1]{\left\lvert #1 \right\rvert}
\newcommand{\I}[1] {\mathbf{1} \left[ #1 \right]}
\newcommand{\pp}[1] {\mathbb{P} \left[ #1 \right]}
\newcommand{\E}[2] {\mathop{\mathbb{E}}\limits_{#1} \left[ #2 \right]}

\definecolor{Top2}{RGB}{102, 171, 221}
\definecolor{Top1}{RGB}{245, 137, 112}

\definecolor{gt}{RGB}{136, 197, 75}
\definecolor{am}{RGB}{117, 170, 255}

\begin{document}
  %
\title{Optimizing Partial Area Under the Top-k Curve: \\ Theory and Practice}

\author{Zitai~Wang,
        Qianqian~Xu*,~\IEEEmembership{Senior ~Member,~IEEE,}
        Zhiyong~Yang,
        Yuan He, \\
        Xiaochun~Cao,~\IEEEmembership{Senior ~Member,~IEEE,}
        and~Qingming~Huang*,~\IEEEmembership{Fellow,~IEEE}
        \IEEEcompsocitemizethanks{
            \IEEEcompsocthanksitem Zitai Wang is with State Key Laboratory of Information Security (SKLOIS), Institute of Information Engineering, Chinese Academy of Sciences, Beijing 100093, China, and also with School of Cyber Security, University of Chinese Academy of Sciences, Beijing 100049, China (email: \texttt{wangzitai@iie.ac.cn}).
            \IEEEcompsocthanksitem Qianqian Xu is with the Key Laboratory of Intelligent Information Processing, Institute of Computing Technology, Chinese Academy of Sciences, Beijing 100190, China (email: \texttt{xuqianqian@ict.ac.cn}). 
            \IEEEcompsocthanksitem Zhiyong Yang is with School of Computer Science and Technology, University of Chinese Academy of Sciences, Beijing 100049, China (email: \texttt{ yangzhiyong21@ucas.ac.cn}).
            \IEEEcompsocthanksitem Yuan He is with the Security Department of Alibaba Group, Hangzhou 311121, China (e-mail: \texttt{heyuan.hy@alibaba-inc.com}).
            \IEEEcompsocthanksitem Xiaochun Cao is with School of Cyber Science and Technology, Shenzhen Campus of Sun Yat-sen University, Shenzhen 518107, China (email:\texttt{caoxiaochun@mail.sysu.edu.cn}).
            \IEEEcompsocthanksitem Qingming Huang is with the School of Computer Science and Technology, University of Chinese Academy of Sciences, Beijing 101408, China, also with the Key Laboratory of Big Data Mining and Knowledge Management (BDKM), University of Chinese Academy of Sciences, Beijing 101408, China, also with the Key Laboratory of Intelligent Information Processing, Institute of Computing Technology, Chinese Academy of Sciences, Beijing 100190, China, and also with Peng Cheng Laboratory, Shenzhen 518055, China (e-mail: \texttt{qmhuang@ucas.ac.cn}).
            \IEEEcompsocthanksitem *Corresponding authors.
        }
}

\markboth{TO APPEAR IN IEEE TRANSACTIONS ON PATTERN ANALYSIS AND MACHINE INTELLIGENCE}%
{Shell \MakeLowercase{\textit{et al.}}: Bare Demo of IEEEtran.cls for Computer Society Journals}
%

\IEEEtitleabstractindextext{%
    \begin{abstract}
    \justifying
        Top-$k$ error has become a popular metric for large-scale classification benchmarks due to the inevitable semantic ambiguity among classes. Existing literature on top-$k$ optimization generally focuses on the optimization method of the top-$k$ objective, while ignoring the limitations of the metric itself. In this paper, we point out that the top-$k$ objective lacks enough discrimination such that the induced predictions may give a totally irrelevant label a top rank. To fix this issue, we develop a novel metric named partial Area Under the top-$k$ Curve (AUTKC). Theoretical analysis shows that AUTKC has a better discrimination ability, and its Bayes optimal score function could give a correct top-$K$ ranking with respect to the conditional probability. This shows that AUTKC does not allow irrelevant labels to appear in the top list. Furthermore, we present an empirical surrogate risk minimization framework to optimize the proposed metric. Theoretically, we present (1) a sufficient condition for Fisher consistency of the Bayes optimal score function; (2) a generalization upper bound which is insensitive to the number of classes under a simple hyperparameter setting. Finally, the experimental results on four benchmark datasets validate the effectiveness of our proposed framework.
    \end{abstract}
    
}
  \maketitle
  \IEEEdisplaynontitleabstractindextext
  \IEEEpeerreviewmaketitle

  \IEEEraisesectionheading{\section{Introduction}\label{sec:introduction}}
\begin{figure}[!t]
    \centering
    \includegraphics[width=0.9\linewidth]{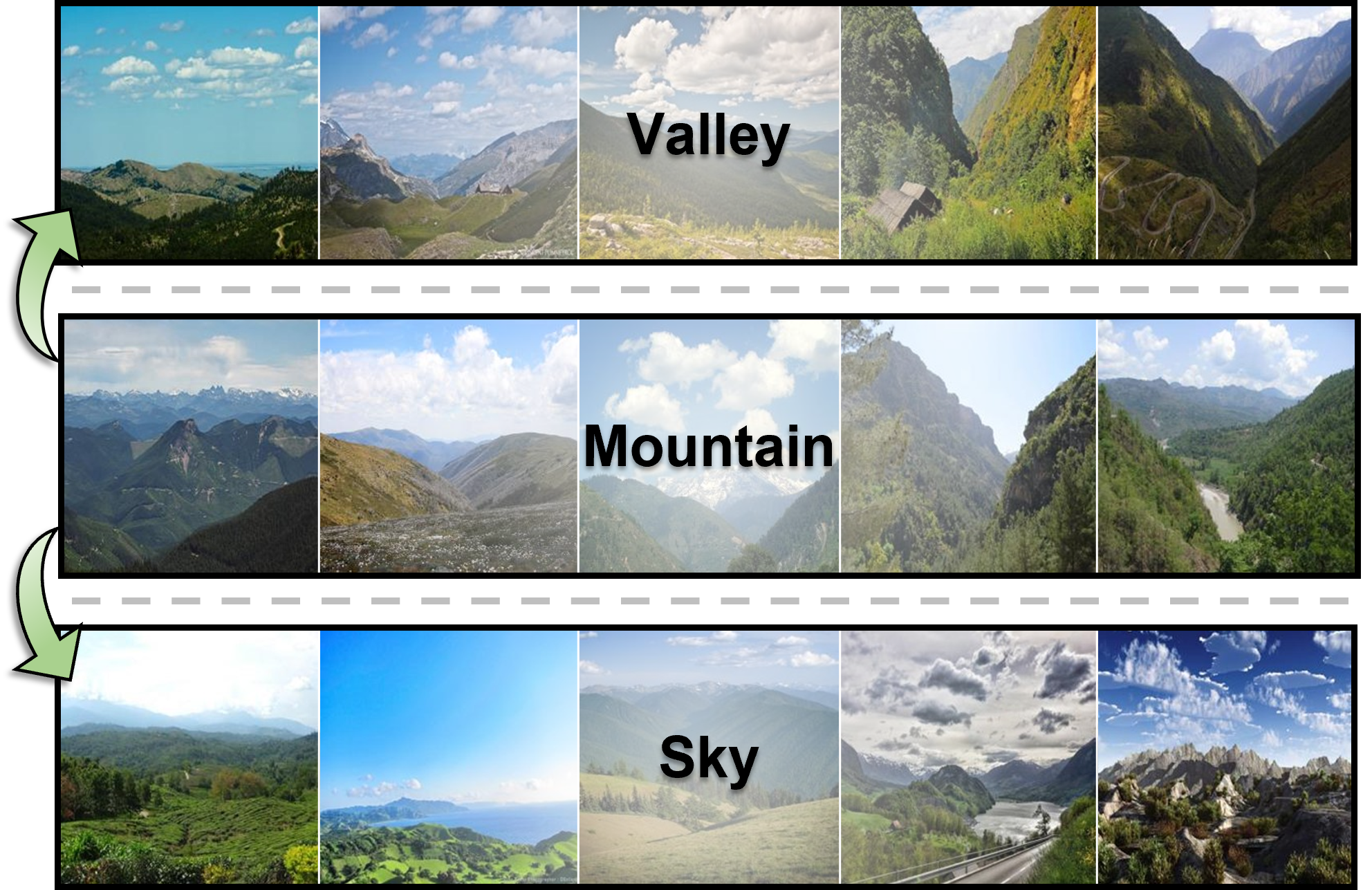}
    \caption{Label ambiguity on the Places-365 dataset \cite{zhou2017places}. On one hand, the semantic similarity between \texttt{Mountain} and \texttt{Valley} makes it easy to make wrong predictions even for humans. On the other hand, many instances are inherently relevant with multiple classes such as  \texttt{Mountain} and \texttt{Sky}.}
    \label{fig:label_ambiguity}
\end{figure}

    \IEEEPARstart{T}{op-1} error, which checks whether the annotated label has the highest score, is a ubiquitous measure in classification tasks. Implicitly, this metric assumes that \textbf{(a)}  the annotated label is exactly the ground-truth, \textit{i.e.}, \textbf{(no label noise \cite{DBLP:journals/tnn/FrenayV14})}, and \textbf{(b)} the ground-truth label is easy to discriminate from the others, \textit{i.e.},  \textbf{(no label ambiguity \cite{DBLP:conf/nips/JinG02,DBLP:conf/cvpr/Lapin0S16})}. To ensure \textbf{(a)}, we can either improve the data quality to clean the possible noises, or  improve the quality of the training model to resist the adverse effect of noises.  Different from the noises, label ambiguity is an intrinsic property of the label semantic. Consequently, it cannot be eliminated by simply improving the quality of data once the label space is fixed. In fact, \textbf{(b)} does not necessarily hold as the number of classes in benchmark datasets grows rapidly \cite{DBLP:conf/cvpr/DengDSLL009,zhou2017places}. As shown in Fig.\ref{fig:label_ambiguity}, when facing numerous classes, there are two major issues that make the ground-truth label ambiguous. One is that the semantic similarity among classes inevitably emerges as the size of label space grows larger, leading to high error rates of the top-1 predicted labels even for humans \cite{DBLP:conf/cvpr/Lapin0S16}. The other is that, when the number of classes is larger, each instance becomes more likely to be relevant with multiple classes, making each relevant class a reasonable prediction. Dealing with label ambiguity is thus a harder problem than label noises. In this sense, our goal is to seek a better solution against label ambiguity in this paper.

    Historically, there exist two solutions against label ambiguity. One solution is to perform multi-label learning \cite{DBLP:conf/nips/JinG02,DBLP:journals/tkde/ZhangZ14}, where each instance is annotated with multiple labels to cover ambiguous labels as much as possible. The other one is to relax the top-1 error and perform top-$k$ optimization \cite{DBLP:conf/nips/LapinHS15,DBLP:conf/cvpr/Lapin0S16,DBLP:journals/pami/LapinHS18,DBLP:conf/icml/YangK20}, where ambiguous labels are allowed to be ranked higher than the ground-truth. In this way, one can approximate the multi-label setting with a single-label dataset. Although the first solution is natural and straightforward, multi-label annotation is generally intractable for large-scale datasets. In practice, single-label datasets \cite{DBLP:conf/cvpr/DengDSLL009,zhou2017places} enjoy much higher scalability than multi-label datasets \cite{DBLP:journals/ijcv/EveringhamGWWZ10,DBLP:conf/eccv/LinMBHPRDZ14} in terms of both the number of classes (hundreds \textit{v.s.} dozens) and the number of instances (million \textit{v.s.} thousands). Considering this fact, this paper focuses on the second solution, with a special interest in directly optimizing the top-$k$ measure. 

    In this direction, existing literature has proposed a series of methods to optimize the top-$k$ objective \cite{DBLP:conf/nips/LapinHS15,DBLP:conf/cvpr/Lapin0S16,DBLP:journals/pami/LapinHS18,DBLP:conf/icml/YangK20}. However, the limitations of the metric itself have not been well understood. In this paper, we argue that \textbf{the top-$k$ error lacks enough discrimination, leaving some fatal errors free from penalty}. Specifically, when the top-$k$ error allows ambiguous labels to be ranked higher, it also relaxes the punishment on  irrelevant labels. As shown in Fig.\ref{fig:label_ambiguity}, even if an irrelevant label \texttt{Dog} is ranked higher than the ground-truth label \texttt{Mountain} and the ambiguous labels \texttt{Valley} and \texttt{Sky}, the top-$k$ error can still be zero, as long as \texttt{Mountain} appears in the predicted top-$k$ list. To fix this issue, this paper takes a further step toward a more discriminating metric. 

    First of all, in Sec.\ref{sec:motivation}, the analysis suggests that \textbf{ looking at the performance at a single fixed $k$ is insufficient to ensure the true prediction quality}. In view of this, a natural solution is to evaluate the performance from a comprehensive perspective. To this end, we design a new measure named partial Area Under the Top-$k$ Curve (AUTKC) in Sec.\ref{sec:metric_autkc}, as an analogy with AUC \cite{DBLP:conf/ijcai/LingHZ03}. The comparison in Sec.\ref{sec:autkc_topk} reveals that AUTKC is strictly consistent and more discriminating than the top-$k$ error, which helps it overcome the limitations of the top-$k$ measure. 

    Considering its advantages, it is appealing to design learning algorithms that optimize AUTKC efficiently. To this end, we have the following abstract formulation of the optimization problem, whose details are shown in Eq.(\ref{eq:op1}).
    \begin{equation}
        \label{eq:abstract}
        \min_{f} \E{ \boldsymbol{x}}{ \sum_{k} \ell_{k} \left( f; \boldsymbol{x}\right)},
    \end{equation}
    where the loss $\ell_{k}$ takes the model prediction $f(\boldsymbol{x})$ as the input. As presented in Sec.\ref{sec:roadmap}, the main challenges are two-fold: \textbf{(C1)} $\ell_{k}$ is discrete and non-differentiable, making the objective difficult to optimize; \textbf{(C2)} the data distribution is generally unavailable, making it impossible to calculate the expectation.

    Targeting \textbf{(C1)}, one common strategy is to replace $\ell_{k}$ with a differentiable surrogate loss. But, what kind of surrogate loss should we select? The basic requirement is Fisher consistency, that is, we should  recover the Bayes optimal solution of Eq.(\ref{eq:abstract}). In view of this, in Sec.\ref{sec:bayes_optimality}, we present a necessary and sufficient property for the Bayes-optimal function induced by Eq.(\ref{eq:abstract}). It suggests that the Bayes optimal score function must \textbf{preserve the ranking of the underlying top-$K$ labels \textit{w.r.t.} the conditional probability $\mathbb{P}(y=i|\bm{x})$}. Since the conditional probability for the irrelevant labels cannot surpass the relevant ones, the optimal solution naturally removes the irrelevant classes from the top list. On top of this, we develop a sufficient condition for the Fisher consistency for AUTKC in Sec.\ref{sec:consist_loss}. Based on this condition, we find that common surrogate losses are consistent with AUTKC as long as the model outputs are bounded, including the square loss, the exponential loss and the logit loss. However, the hinge loss, which is the standard surrogate loss of the top-$k$ optimization methods, is inconsistent.

    In Sec.\ref{sec:generalization}, we further construct an empirical optimization framework for AUTKC to relieve \textbf{(C2)}. Specifically, we turn to minimize the unbiased empirical estimation of the surrogate loss over the training set, thus avoiding the direct calculation of the expectation. A basic question is whether the performance on the training set can generalize well to the unseen data. To answer this question in the context of label ambiguity, we construct the upper bound of the generalization error in terms of the number of classes $C$ and deep neural networks. The result shows that we can obtain a generalization bound that is insensitive to $C$.

Finally, we conduct the experiments on four benchmark datasets in Sec.\ref{sec:experiment}, where the number of classes ranges from 10 to 365. 

In summary, the contribution of this paper is three-fold:
\begin{itemize}
    \item \textbf{New measure}: We provide a novel performance measure named AUTKC to handle the label ambiguity in large-scale classification. 
    \item \textbf{Theoretical guarantee}: The proposed framework for AUTKC optimization is supported by consistent surrogate losses and a generalization bound insensitive to the number of classes.
    \item \textbf{Empirical validation}: The empirical results not only show the superiority of the proposed framework, but also validate the theoretical results.
\end{itemize}

\section{Related Work}
\subsection{TOP-k Optimization}
The early literature on top-$k$ optimization focuses on the binary scenarios where only \ul{the top-ranked  instances} are of interest, such as retrieval (relevant/irrelevant) \cite{DBLP:journals/ftir/Liu09} and recommend (like/dislike) \cite{DBLP:conf/kdd/RendleMNS09}. These methods minimize a pairwise ranking loss that punishes more on the errors related to the top-ranked instances \cite{DBLP:conf/sdm/Agarwal11,DBLP:journals/jmlr/Rudin09,DBLP:conf/icml/UsunierBG09}.

As the scale of classes grows rapidly in modern benchmark datasets \cite{DBLP:conf/cvpr/DengDSLL009,zhou2017places}, the semantic overlap between classes inevitably emerges. In this light, the top-$k$ error has been extended to multiclass classification \cite{DBLP:conf/cvpr/PerronninAHS12,DBLP:journals/ijcv/RussakovskyDSKS15,DBLP:conf/cvpr/HeZRS16,zhou2017places}. Similar to the binary case, this metric evaluates the model performance on the top-ranked items. The difference lies in it measuring \ul{the ranking of the ground-truth label}, rather than that of instance. The early work in this direction adopts an alternative approach to minimize the top-$k$ error. For example, \cite{DBLP:journals/ijcv/McAuleyRC13} models the classes missing from the annotation by the structured learning technique with latent variables. \cite{DBLP:conf/icml/RossZYDB13} analyzes the problem from the perspective of submodular reward functions. \cite{DBLP:conf/nips/SwerskyTAZF12} proposes a probabilistic method explicitly modeling the total number of objects in an image. Besides, \cite{DBLP:conf/iccv/GuillauminMVS09,DBLP:journals/pami/MensinkVPC13} use metric learning to find potential labels, especially when the number of classes is dynamic.

Recently, more efforts focus on directly minimizing the top-$k$ error. To this end, \cite{DBLP:conf/nips/LapinHS15} relaxes the multiclass hinge loss to penalize the top-$k$ highest prediction. Although this approach is intuitive, the induced loss functions are not top-$k$ consistent \cite{DBLP:conf/cvpr/Lapin0S16}. In other words, minimizing these losses does not guarantee an ideal top-$k$ performance. To this end, \cite{DBLP:conf/cvpr/Lapin0S16} analyzes the calibration property of some top-$k$ loss functions truncated from the multiclass hinge loss and the cross-entropy loss. The theoretical results are further extended to multi-label classification \cite{DBLP:journals/pami/LapinHS18}. Most recently, \cite{DBLP:conf/icml/YangK20} points out that the condition for top-$k$ calibration proposed by \cite{DBLP:conf/cvpr/Lapin0S16} is invalid when ties exist in the prediction. To fix this issue, \cite{DBLP:conf/icml/YangK20} develops a necessary and sufficient property for top-$k$ optimality named top-$k$ preserving. Based on this property, a top-$k$ consistent hinge-like loss is finally established.

    Compared with existing efforts on the optimization method of the top-$k$ objective, this paper focuses on the limitations of the metric itself, whose details are shown in Sec.\ref{sec:motivation}. To solve this problem, we propose a novel metric named AUTKC. Moreover, an efficient framework for AUTKC optimization is also established.

\subsection{Label Ranking}
\label{sec:related_label_ranking}
    The setting we discuss is naturally associated with the label ranking problem, where the predictions are also evaluated by the ranking of labels. Specifically, given the label space $\mathcal{Y}$, label ranking aims to assign each instance $\boldsymbol{x}$ with \ul{the correct ranking of all the labels}, that is, a complete/partial, transitive, and asymmetric relation $\succ_{\boldsymbol{x}}$ defined on $\mathcal{Y}$, where $i \succ_{\boldsymbol{x}} j$ means that the label $i$ precedes the label $j$ in the ranking associated with $\boldsymbol{x}$. According to the taxonomy established by \cite{DBLP:journals/jcp/ZhouLYHL14}, label ranking methods can be divided into four categories: the ones decomposing the original problem to multiple simple objectives such as pointwise function \cite{DBLP:conf/nips/Har-PeledRZ02,DBLP:conf/nips/DekelMS03} and pairwise ranking loss \cite{DBLP:journals/ai/HullermeierFCB08,DBLP:books/daglib/p/FurnkranzH10a}; probabilistic methods including tree-based model \cite{DBLP:conf/icml/ChengHH09,DBLP:books/daglib/p/YuWL10,DBLP:journals/inffus/AledoGM17,DBLP:journals/es/SaSKC17,DBLP:conf/alt/ClemenconKS18}, Gaussian mixture model \cite{grbovic2012learning} and structured learning \cite{DBLP:conf/nips/KorbaGd18}; the ones based on similarity \cite{DBLP:conf/dis/AiguzhinovSS10,DBLP:conf/pakdd/SaSJAC11,DBLP:conf/icann/RibeiroDSK12}; and the rule-based ones \cite{DBLP:conf/ipmu/GurrieriSFGS12,DBLP:journals/inffus/SaASJK18}. Please refer to the surveys for more details \cite{DBLP:books/daglib/p/VembuG10,DBLP:journals/jcp/ZhouLYHL14}.

    Compared with label ranking, top-$k$ optimization follows the setting of classification, where each instance $\boldsymbol{x}$ is assigned with a single label $y \in \mathcal{Y}$, \textit{i.e.}, the ground-truth label. As a result, the common measures used in label ranking become invalid, such as Spearman’s rank \cite{spearman1904proof} and Kendall’s tau \cite{kendall1948rank}. And top-$k$ optimization only considers \ul{the ranking associated with the ground-truth label}, \textit{i.e.}, ensuring $y \succ_{\boldsymbol{x}} j, j \in \mathcal{Y}$. Concretely, top-$k$ optimization requires that $| \{j \in \mathcal{Y}: j \succ_{\boldsymbol{x}} y\} | < k$, where $\size{\cdot}$ represents the cardinality of a set. Besides, our analysis in Sec.\ref{sec:bayes_optimality} shows that optimizing AUTKC encourages the prediction to preserve \ul{the ranking of the underlying top-$K$ labels}. In this sense, AUTKC optimization essentially aims to recover the ranking of partial labels under a restricted condition where only the ground-truth label is available.

\subsection{AUC Optimization}
\label{sec:related_auc}
Our proposed framework is also related to AUC optimization, where the optimization objective is also a pairwise ranking loss. Specifically, AUC, the Area Under the receiver operating characteristic (ROC) Curve, measures the probability that the positive instances are ranked higher than the negative ones, with the assumption that the possibility to observe ties in the comparisons equals zero \cite{DBLP:conf/ijcai/LingHZ03}. Being insensitive to the label distribution, AUC has become a popular metric in the class-imbalanced applications such as disease prediction \cite{DBLP:conf/isbi/ZhouGCGFTYZ020} and rare event detection \cite{DBLP:conf/cvpr/LiuLLG18,DBLP:conf/kdd/LiuZASLFHT20,DBLP:conf/aaai/WuH0B0C20}. 

Traditional research on AUC optimization targets the whole area under the ROC curve. At the early stage, the major studies in this direction focus on the off-line setting \cite{DBLP:conf/icml/HerschtalR04,DBLP:conf/pkdd/CaldersJ07,DBLP:journals/jmlr/FreundISS03,DBLP:conf/kdd/Joachims06,DBLP:journals/jmlr/ZhangSV12}. As the scale of datasets increases rapidly, it becomes infeasible to optimize the pairwise loss in a full-batch manner. For this reason, more studies explore the extension of AUC optimization methods in the online setting \cite{DBLP:conf/icml/ZhaoHJY11,DBLP:conf/icml/GaoJZZ13,DBLP:conf/nips/YingWL16,DBLP:conf/icml/NatoleYL18,DBLP:journals/fams/NatoleYL19,DBLP:conf/aaai/GuHH19,yang2020stochastic,9200544}. Besides, the theoretical guarantee of the learning framework is another appealing topic, including generalization bound \cite{DBLP:journals/jmlr/AgarwalGHHR05,usunier2005data,DBLP:conf/nips/UsunierAG05,DBLP:journals/jmlr/RalaivolaSS10,DBLP:journals/jmlr/WangKPJ12} and consistency property \cite{DBLP:journals/jmlr/Agarwal14,DBLP:conf/ijcai/GaoZ15}. Recently, more literature focuses on the partial area under the ROC curve since only the order of the top-ranked instances are of interest in many practical applications \cite{DBLP:conf/kdd/NarasimhanA13,DBLP:conf/icml/NarasimhanA13,DBLP:journals/neco/Narasimhan017,DBLP:conf/icml/0001XBHCH21}. In this sense, partial AUC optimization is related to top-$k$ optimization in binary scenarios.

Compared with partial AUC optimization, our proposed framework focuses on the ranking of multiple labels. Moreover, we assume that ties exist in the scores. These two differences will lead to totally different theoretical results, whose details are shown in Sec. \ref{sec:surrogate} and Sec. \ref{sec:generalization}.

\begin{table}[t]
    \renewcommand{\arraystretch}{1.2}
    \caption{Some important Notations used in this paper.}
    \label{table:notation}
    \centering
    \begin{tabular}{ll}
        \toprule
        Notation & Description \\
        \midrule
        $C$ & the number of classes \\
        $K$ & the maximum number of relevant labels \\ 
        $\mathcal{X}, \mathcal{Y}$ & the input space and the label space \\
        $\mathcal{D}$ & the joint distribution defined on $\mathcal{Z} = \mathcal{X} \times \mathcal{Y}$ \\
        $\eta(\boldsymbol{x})_y$ & the conditional distribution $\pp{y | \boldsymbol{x}}$ \\
        $f$ & the score function mapping $\mathcal{X}$ to $\mathbb{R}^{C}$ \\
        $\mathcal{F}$ & the set of score functions \\
        $L$ & the loss function mapping $\mathcal{Y} \times \mathbb{R}^C$ to $\mathbb{R}_{+}$ \\
        $\pi_{\boldsymbol{s}}(y)$ & the index of $s_y$ when sorting $\boldsymbol{s} \in \mathbb{R}^C$ descendingly \\
        $s_{[k]}$ & the $k$-th greatest entry of $\boldsymbol{s} \in \mathbb{R}^C$  \\
        $(s_{\setminus y})_{[k]}$ & the $k$-th greatest entry of $\boldsymbol{s} \in \mathbb{R}^C$ except $s_y$ \\
        $\ell_{0-1}, \ell$ & the 0-1 loss and its surrogate loss \\
        $\mathcal{R}_K(f)$ & the original $\atopk$ risk  \\
        $\mathcal{R}_K(f^*)$ & the optimal $\atopk$ risk\\
        $\mathsf{RP}_K$ & top-$K$ ranking-preserving property (Def.\ref{def:rp}) \\
        $\mathcal{R}_K^{\ell}(f)$ & the surrogate $\atopk$ risk \\
        $\mathcal{S}$ & the dataset sampled from $\mathcal{D}$ \\
        $\hat{\mathcal{R}}_K^{\ell}(f, \mathcal{S})$ & the empirical $\atopk$ risk \\
        $\mathfrak{G}_\mathcal{S}(\mathcal{F})$ & the empirical Gaussian Complexity \\
        $\Phi(\mathcal{S}, \delta)$ & the generalization error of $\atopk$ \\ 
        $\mathfrak{C}(\Theta, \epsilon, d)$ & covering number defined on the metric space $(\Theta, d)$ \\
        $\mathcal{F}_{\beta, \nu}$ & the set of neural networks we discuss in Sec.\ref{sec:bound_cnn} \\
        $\mathsf{vec}(\cdot)$ & the vectorization operation \\
        $\mathsf{mt}(\cdot)$ & the operator matrix \\
        \bottomrule
    \end{tabular}
\end{table}

\begin{figure*}[t]
    \centering
    \includegraphics[width=0.9\linewidth]{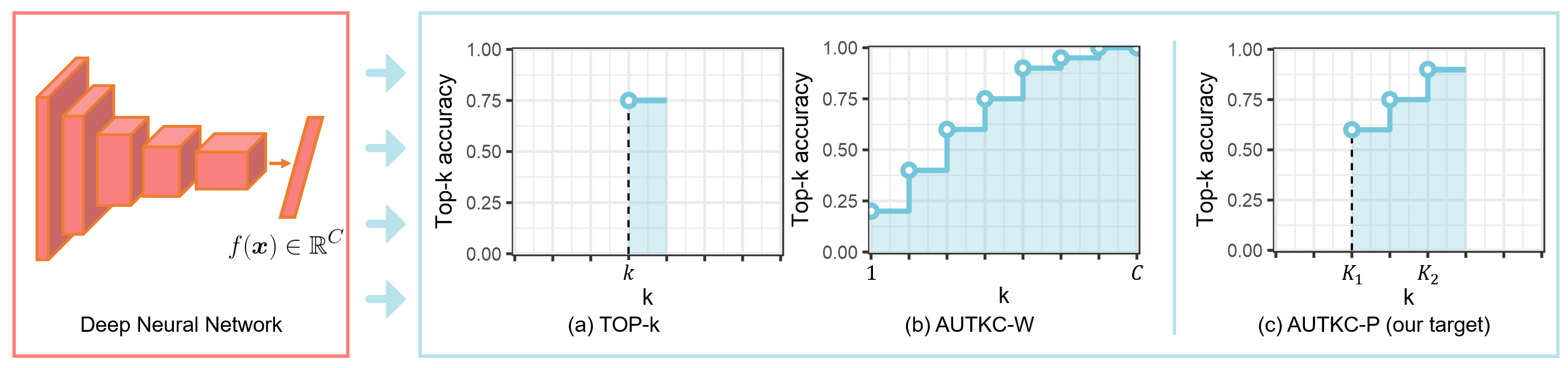}
    \caption{Comparisons of $\topk$ and $\atopk$: (a) $\topk$ focuses on the performance at a single point of the $\topk$ curve; (b) $\mathsf{AUTKC}\text{-}\mathsf{W}$ considers the entire area under the $\topk$ curve; (c) $\mathsf{AUTKC}\text{-}\mathsf{P}$ emphasizes the partial area with $k$ ranging in $[K_1, K_2]$. We denote $\mathsf{AUTKC}\text{-}\mathsf{P}$ as $\atopk$ in the rest discussion for convenience.}
    \label{fig:metrics}
\end{figure*}

\section{Preliminaries}
\label{sec:pre}
In this section, we first describe the notations of the top-$k$ error in Sec.\ref{sec:top_k_metric}. Then, a detailed analysis on the limitations of this metric will be presented in Sec.\ref{sec:motivation}. Motivated by this, we present the definition of AUTKC in Sec.\ref{sec:metric_autkc} and further reveal its advantages over the top-$k$ error in Sec.\ref{sec:autkc_topk}.

\subsection{Standard TOP-k Metric}
\label{sec:top_k_metric}
    In multi-class classification, one generally assumes that the samples are drawn \textit{i.i.d.} from a product space $\mathcal{Z} = \mathcal{X} \times \mathcal{Y}$, where $\mathcal{X}$ is the input space and $\mathcal{Y} = \{1, \cdots, C\}$ is the label space. Let $\mathcal{D}$ be the joint distribution defined on $\mathcal{Z}$, and $\eta(\boldsymbol{x})_y = \pp{y | \boldsymbol{x}}$ denotes the conditional probability of instance $\boldsymbol{x}$ according to distribution $\mathcal{D}$. We assume that there exist no ties in conditional probability $\eta(\boldsymbol{x}) \in \mathbb{R}^C$. In other words, we have $\eta(\boldsymbol{x})_{[1]} > \cdots > \eta(\boldsymbol{x})_{[C]}$, where $\eta(\boldsymbol{x})_{[k]}$ is the $k$-th greatest element of $\eta(\boldsymbol{x})$.
Then, our task is to learn a score function $f: \mathcal{X} \to \mathbb{R}^{C}$ to estimate the conditional probability for each class. Generally, a loss function $L:  \mathcal{Y} \times \mathbb{R}^C \to \mathbb{R}_{+}$ is used to measure the quality of the score function $f$ at $\boldsymbol{z} \in \mathcal{Z}$. For example, a standard measure for classification is top-$1$ error, which compares the ground-truth label with the only one guess, \textit{i.e.,} the class with the highest score. 

\begin{figure}[!t]
        \centering
        \includegraphics[width=0.95\linewidth]{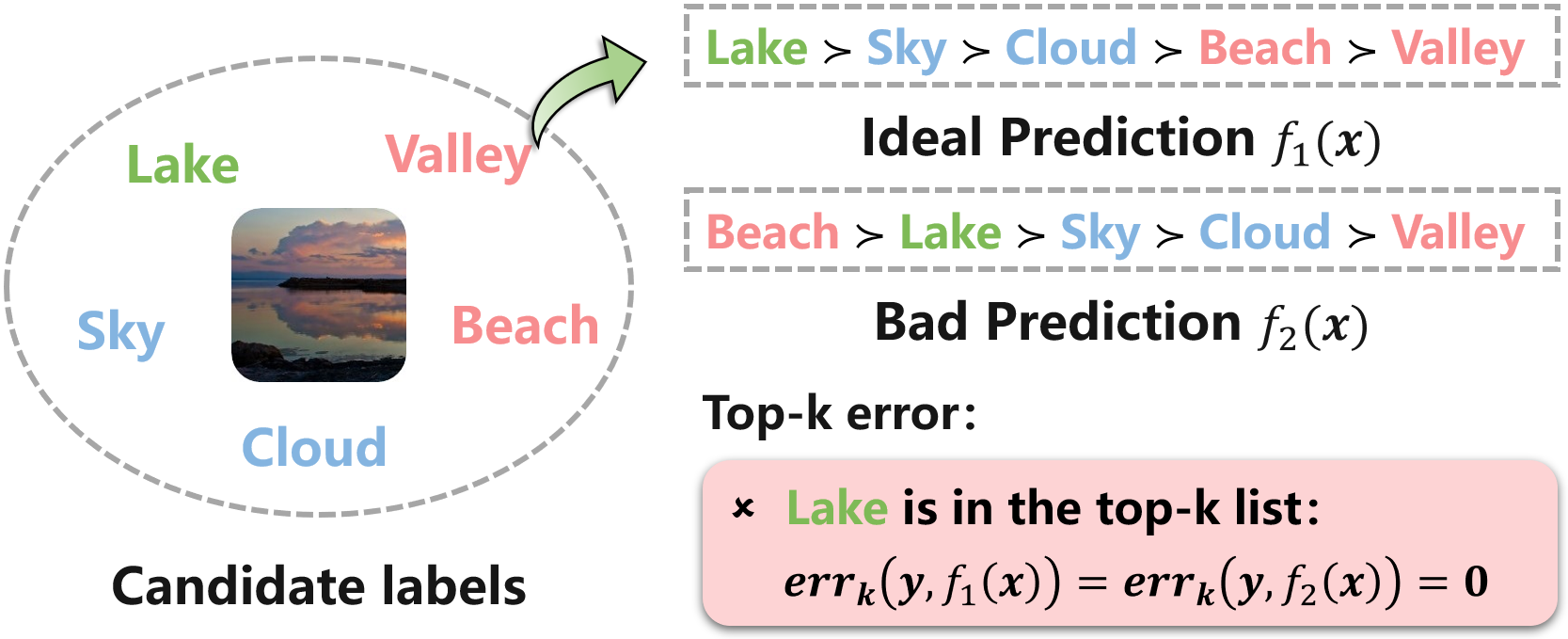}
        \caption{The limitation \textbf{(L1)} of $\topk$. $f_1(\boldsymbol{x})$ is an ideal prediction, while $f_2(\boldsymbol{x})$ is a \textit{bad} prediction since it ranks the irrelevant label \texttt{Beach} higher than the ground-truth label \texttt{Lake}. We expect $ err(y, f_2(\boldsymbol{x})) > err(y, f_1(\boldsymbol{x}))$, but $err_k(y, f_1(\boldsymbol{x})) = err_k(y, f_2(\boldsymbol{x})) = 0$ when $k \ge 2$.}
        \label{fig:l2}
\end{figure}

Recently, top-$k$ error has raised more attention due to label ambiguity. Specifically, it indicates whether the ground-truth label appears in the top-$k$ ranking list:
\begin{equation}
    \topk^{\downarrow} (f) = \E{\boldsymbol{z} \sim \mathcal{D}}{ err_k(y, f(\boldsymbol{x})) },
\end{equation}
where 
\begin{equation}
        err_k(y, f(\boldsymbol{x})) = \I{\pi_{f(\boldsymbol{x})}(y) > k},
\end{equation}
$\I{\cdot}$ is the indicator function, and $\pi_{f(\boldsymbol{x})}(y) \in \left\{1, \cdots, C\right\}$ denotes the index of $f(\boldsymbol{x})_y$ when we sort $f(\boldsymbol{x})$ in a descending order. Meanwhile, top-$k$ accuracy, denoted as $\topk^{\uparrow}$, is defined as $1 - \topk^{\downarrow}$; thus the two metrics are equivalent. In the following discussion, we use $\topk$ to refer to top-$k$ error and accuracy when there exists no ambiguity.

Note that different ways of breaking ties will lead to different $\pi_{f(\boldsymbol{x})}(y)$. Traditional methods assume that no ties exist since the number of classes is relatively small. However, the problem becomes non-negligible when calculating $\topk$ since label ambiguity generally involves a large number of classes. Following the prior arts on $\topk$ optimization \cite{DBLP:conf/icml/YangK20}, we adopt the worst-case assumption:
\begin{assumption}
    \label{ass:wrongly_break_ties}
    There exist ties in $f(\boldsymbol{x})$, and all the ties will be wrongly broken. In other words, we have the following formulations:
    \begin{itemize}
        \item Given $y_1, y_2$ such that $\eta(\boldsymbol{x})_{y_1} > \eta(\boldsymbol{x})_{y_2}$, if $f(\boldsymbol{x})_{y_1} = f(\boldsymbol{x})_{y_2}$, then we have $\pi_{f(\boldsymbol{x})}(y_1) > \pi_{f(\boldsymbol{x})}(y_2)$.
        \item Given the ground-truth label $y$, for any $y' \neq y$ such that $f(\boldsymbol{x})_{y'} = f(\boldsymbol{x})_{y}$, we have $\pi_{f(\boldsymbol{x})}(y) > \pi_{f(\boldsymbol{x})}(y')$.
    \end{itemize}
\end{assumption}
    \begin{remark}
        The two formulations apply to different situations:
        \begin{itemize}
            \item The first formulation is a population-level assumption where the conditional probability $\eta(\boldsymbol{x})$ is available. It states that, when ties exist, the ranking results of $f(\boldsymbol{x})$ are opposite to those of $\eta(\boldsymbol{x})$. We apply this formulation in the proof of Thm.\ref{thm:bayes_optimal}, Thm.\ref{thm:condition_for_consistency}, and Thm.\ref{thm:hinge}.
            \item The second formulation is an empirical-level assumption where only the ground-truth label is available. It states that, when ties exist, the ground-truth label will be ranked lower. We apply this formulation in the proof of Thm.\ref{thm:reformulation_opzero}.
            \item In summary, these two conditions ensure that we will always choose the worst result when there is a tie in a conservative sense.
        \end{itemize}
    \end{remark}

\begin{figure}[!t]
    \centering
    \includegraphics[width=0.68\linewidth]{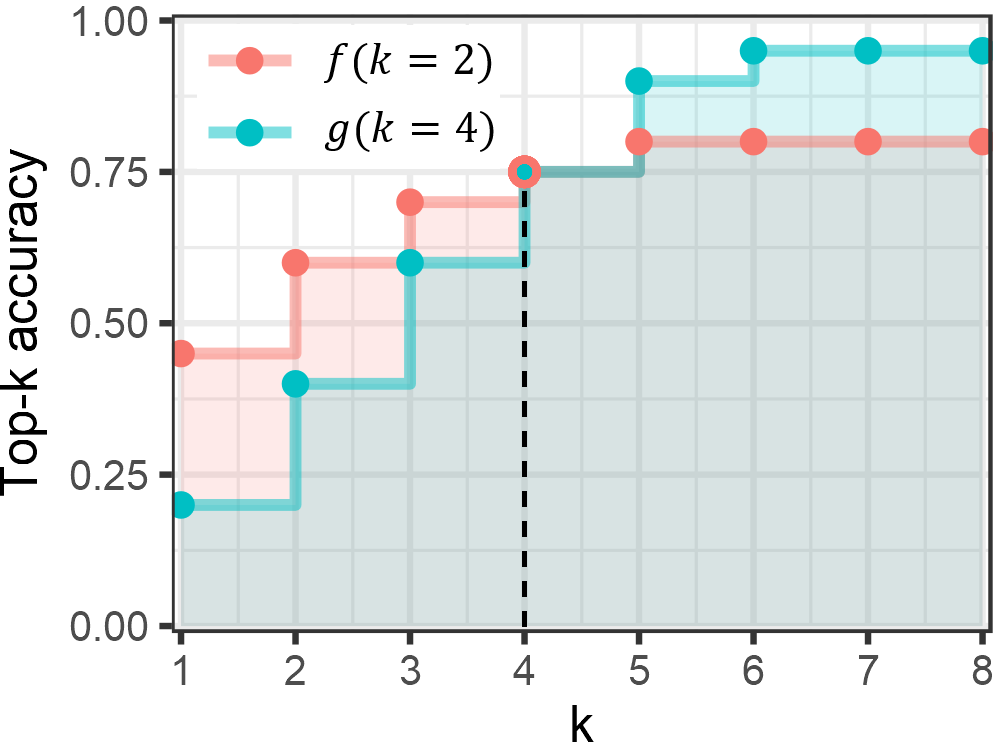}
    \caption{The limitation \textbf{(L2)} of $\topk$. $f$ and $g$ perform inconsistently at different $k$. It is difficult to select which model to deploy unless the participant knows the $k$ of interest ahead. In some scenarios, the $k$ of interest changes dynamically, and optimizing the performance at a specific $k$ is not a reasonable strategy.}
    \label{fig:topk_performance_inconsistent}
\end{figure}

\subsection{Motivation}
\label{sec:motivation}
    As mentioned in the introduction, $\topk$ has its limitations. In this subsection, we will make a further analysis to reveal its reasons.

    On one hand, $\topk$ \textbf{is not discriminating enough, leaving some fatal errors free from penalty (L1)}. As shown in Fig.\ref{fig:l2}, we assume that $\texttt{Lake}$ is the ground-truth label; $\texttt{Sky}$ and $\texttt{Cloud}$ are the labels with ambiguity; $\texttt{Beach}$ and $\texttt{Valley}$ are the irrelevant labels. Consider the following two score function such that:
    $$\begin{aligned}
        f_1: & \textbf{\texttt{  Lake}} \succ \underline{\texttt{Sky}} \succ \underline{\texttt{Cloud}} \succ \texttt{Beach} \succ \texttt{Valley}, \\
        f_2: & \texttt{  Beach} \succ \textbf{\texttt{Lake}} \succ \underline{\texttt{Sky}} \succ \underline{\texttt{Cloud}} \succ \texttt{Valley}. \\
    \end{aligned}$$
    Obviously, $f_1$ is an ideal score function since it ranks $\texttt{Lake}$ first, while $f_2$ is a \textit{bad} score function since it ranks the irrelevant label $\texttt{Beach}$ higher than the ground-truth label. Based on this observation, we expect 
    $$
        err(y, f_2(\boldsymbol{x})) > err(y, f_1(\boldsymbol{x})) .
    $$
    However, when $k \ge 2$, we have
    $$
        err_k(y, f_1(\boldsymbol{x})) = err_k(y, f_2(\boldsymbol{x})) = 0.
    $$
    In other words, although using the top-$k$ error can relax the punishment on label ambiguity, it also relaxes the punishment on the fatal errors. In this sense, \textbf{we need to seek out a more discriminating metric to selectively relax the punishment}.

    One may wonder whether a clever selection of $k$ could eliminate this limitation. Unfortunately, \textbf{the number of ambiguous labels generally differs among classes, while the hyperparameter $k$ is universal for the classes}. This contradiction leads to a dilemma during the training period: if the value of $k$ is too small for some classes, label ambiguity might make these classes hard to optimize, which is the motivation of $\topk$ optimization. But, if the value of $k$ is too large, irrelevant labels will be more likely to be ranked higher than the ground-truth label, which is also unexpected. This suggests that merely looking at a single $k$ is not enough to address \textbf{(L1)}.

On the other hand, $\topk$ \textbf{is a static measure, ignoring the changes in operator's decision condition (L2)}. To be specific, empirical results suggest that models trained with different hyperparameters $k$ generally perform inconsistently on the top-$k$ curve \cite{DBLP:conf/cvpr/Lapin0S16,DBLP:journals/pami/LapinHS18}. As shown in Fig. \ref{fig:topk_performance_inconsistent}, model $f$ outperforms model $g$ when $k < 4$, while the results reverse when $k > 4$. In view of this, the operator might ask: which model should we select to deploy? In practice, the answer depends on which $k$ is of interest during the deployment phase. However, $\topk$ optimization methods fix the value of $k$ during the training phase. As a result, the model performance cannot be guaranteed if the $k$ of interest changes after the model is deployed. 

Being aware of the limitations of $\topk$, we have the following question: 
$$
    \text{(\textbf{Q1}) \textit{Whether there exists a "better" metric for label ambiguity?}}
$$

\subsection{AUTKC and its Partial Variant}
\label{sec:metric_autkc}
Our answer to \textbf{(Q1)} is to evaluate the score function from a more comprehensive perspective, rather than a single point of the top-$k$ curve. Taking the inspiration from Area Under the ROC Curve (AUC) \cite{DBLP:conf/ijcai/LingHZ03}, we propose to adopt \textit{the area under the top-$k$ curve (AUTKC)} to evaluate the quality of the given score function:
\begin{equation}
    \label{equ:original_autkc}
    \mathsf{AUTKC}\text{-}\mathsf{W}^{\downarrow}(f) := \E{\boldsymbol{z} \sim \mathcal{D}} { \frac{1}{C} \int_{1}^{C} err_k(y, f(\boldsymbol{x})) dk } ,
\end{equation}

The original definition of $\atopk$ considers the whole area under the top-$k$ curve. However, we generally focus on the performance within a given range of $k$. In this case, Eq. (\ref{equ:original_autkc}) provides a biased estimation of the expected performance by considering the unrelated regions. This motivates us to adopt its variant involving the given range of $k$:
\begin{equation}
    \label{equ:partial_autkc}
    \begin{split}
        & \mathsf{AUTKC}\text{-}\mathsf{P}^{\downarrow}(f) \\
        & \phantom{-} := \E{\boldsymbol{z} \sim \mathcal{D}} { \frac{1}{K_1 - K_2 + 1} \int_{K_1}^{K_2} err_k(y, f(\boldsymbol{x})) dk } ,
    \end{split}
\end{equation}
where $K_1, K_2$ are the minimum and the maximum number of ambiguous labels, respectively. Since $\mathsf{TOP}\text{-}\mathsf{1}$ is the most common metric of interest, and there generally exist some instances involving no ambiguous labels, we set $K_1 = 1$ and denote $K_2$ as $K$ for the sake of conciseness in the following discussion.

Furthermore, since $k$ is discrete, Eq. (\ref{equ:partial_autkc}) enjoys a much simpler formulation:
\begin{equation}
    \mathsf{AUTKC}\text{-}\mathsf{P}^{\downarrow}(f) := \E{\boldsymbol{z} \sim \mathcal{D}} {aerr_K(y, f(\boldsymbol{x}))},
\end{equation}
where 
\begin{equation}
    aerr_K(y, f(\boldsymbol{x})) = \frac{1}{K} { \sum_{k=1}^{K} err_k\left(y, f(\boldsymbol{x})\right)}.
\end{equation}
Note that we \ul{denote $\mathsf{AUTKC}\text{-}\mathsf{P}^{\downarrow}$ as $\atopk^{\downarrow}$} in the following discussion for the sake of convenience. Similar to $\topk$, the partial area under top-$k$ accuracy curve, denoted as $\mathsf{AUTKC}^{\uparrow}$, is equivalent to $\mathsf{AUTKC}^{\downarrow}$, and we use $\atopk$ to refer to the two metrics when there exists no ambiguity.

\subsection{AUTKC vs TOP-k}
\label{sec:autkc_topk}
    Intuitively, $\atopk$ is a better measure than $\topk$. On one hand, $\atopk$ is more discriminating than $\topk$, which helps overcome the limitation \textbf{(L1)}. For example, when we set $K=3$, $\atopk$ tells us $f_1$ is the ideal score function and indeed better than $f_2$:
    $$aerr_K(y, f_1(\boldsymbol{x})) = \frac{1}{3} \left( 0 + 0 + 0 \right) = 0, $$
    and 
    $$aerr_K(y, f_2(\boldsymbol{x})) = \frac{1}{3} \left( 1 + 0 + 0 \right) = \frac{1}{3}. $$
On the other hand, $\atopk$ only requires a universal hyperparameter $K$ for all the instances, that is, the maximum number of ambiguous labels. Such a simplification is beneficial to overcoming the limitation \textbf{(L2)}. 

Although the aforementioned example indicates the advantages of AUTKC, more theoretical clues are necessary to validate our argument. To this end, we first present the definition of \textbf{consistent} and \textbf{discriminating}, which is a classic technique for comparing measures \cite{DBLP:conf/ijcai/LingHZ03}. 

    \begin{definition}[Degree of consistency  \cite{DBLP:conf/ijcai/LingHZ03}]
        Given two measures $f$ and $g$ on domain $\Psi$ and two predictions $\boldsymbol{a}, \boldsymbol{b} \in \Psi$, let $R = \{ (\boldsymbol{a}, \boldsymbol{b}) | f(\boldsymbol{a}) > f(\boldsymbol{b}), g(\boldsymbol{a}) > g(\boldsymbol{b}) \}, S = \{ (\boldsymbol{a}, \boldsymbol{b}) | f(\boldsymbol{a}) > f(\boldsymbol{b}), g(\boldsymbol{a}) < g(\boldsymbol{b}) \}$. The degree of consistency between $f$ and $g$ is defined as $\mathbf{C}=\frac{\size{R}}{\size{R} + \size{S}} \in [0, 1]$, where $\size{\cdot}$ is the cardinality of a set.
    \end{definition}
\begin{remark}
    The degree of consistency estimates the probability that the two measures make the same judgment on the given domain. A relatively high degree of consistency is a necessary condition for measure comparison.
\end{remark}
\begin{definition}[Degree of discriminancy \cite{DBLP:conf/ijcai/LingHZ03}]
    Given two measures $f$ and $g$ on domain $\Psi$ and two predictions $\boldsymbol{a}, \boldsymbol{b} \in \Psi$, let $P = \{ (\boldsymbol{a}, \boldsymbol{b}) | f(\boldsymbol{a}) > f(\boldsymbol{b}), g(\boldsymbol{a}) = g(\boldsymbol{b}) \}, S = \{ (\boldsymbol{a}, \boldsymbol{b}) | g(\boldsymbol{a}) > g(\boldsymbol{b}), f(\boldsymbol{a}) = f(\boldsymbol{b}) \}$. The degree of discriminancy between $f$ and $g$ is $\mathbf{D}=\frac{\size{P}}{\size{S}}$.
\end{definition}
\begin{remark}
    Degree of discriminancy describes whether a measure could find the discrepancy that the other measure fails to discover. For example, the top-1 error is intuitively more discriminating than the top-$k$ error when $k > 1$.
\end{remark}
\begin{definition}[Consistent and discriminating \cite{DBLP:conf/ijcai/LingHZ03}]
    The measure $f$ is consistent and more discriminating than $g$ if and only if $\mathbf{C} > 0.5$ and $\mathbf{D} > 1$. In this case, we conclude, intuitively, that $f$ is a "better" measure than $g$. Especially, if $\mathbf{C}=1.0$ and $\mathbf{D} = \infty$, we say the measure $f$ is strictly consistent and more discriminating than $g$.
\end{definition}
 
On top of consistency and discrimination, the following theorem suggests that $\atopk$ is a better measure than $\topk$. Please see Appendix.\ref{sec_app:metric_comparison} for the proof.
\begin{restatable}{theorem}{consistencydiscriminating}
    \label{thm:consistencydiscriminating}
    For any $k < K$, $\atopk$ is strictly consistent and more discriminating than $\topk$.
\end{restatable}
\begin{remark}
    Being strictly consistent with $\topk$ for any $k < K$ means that $\atopk$ is insensitive to the $k$ of interest, which overcomes the limitation \textbf{(L2)}. Meanwhile, being more discriminating implies that $\atopk$ might be a better optimization objective than $\topk$, thus providing a clue to overcoming the limitation \textbf{(L1)}. Further analysis in Sec.\ref{sec:bayes_optimality} will support these conjectures.
\end{remark}

\section{Roadmap for AUTKC Optimization}
\label{sec:roadmap}
So far, we have known the advantage of $\atopk$. Then, a natural question is
$$\begin{aligned}
    \text{(\textbf{Q2}) \textit{How to design }} & \text{\textit{a learning algorithm }} \\
    & \text{\textit{that optimizes $\atopk$ effectively?}}
\end{aligned}$$
Following the standard machine learning paradigm \cite{10.5555/2371238}, we first reformulate the metric optimization problem to a risk minimization problem:
\begin{equation}
    \label{eq:raw}
    \begin{split}
    (OP_0)\phantom{|} \min_{f} \ & \mathcal{R}_K(f) := \atopk^{\downarrow} (f)\\
    & = \E{\boldsymbol{z} \sim \mathcal{D}} {\frac{1}{K} \sum_{k \le K} \ell_{0-1} \left(s_y - (s_{\setminus y})_{[k]}\right)},
    \end{split}
\end{equation}
where $s_y := f(\boldsymbol{x})_y$ denotes the score of the class $y$, $(s_{\setminus y})_{[k]}$ is the $k$-th greatest entry of $f(\boldsymbol{x})$ except $s_y$, and $\ell_{0-1}(t) := \I{t \le 0}$ is the 0-1 loss. Note that we set $\ell_{0-1}(0)=1$ since according to Asm.\ref{ass:wrongly_break_ties}, any ties will lead to an error. 

For the sake of conciseness, we provide an equivalent reformulation of $(OP_0)$, where the annoying operation $\setminus y$ is removed. See Appendix.\ref{sec_app:reformulation_opzero} for the proof.
\begin{restatable}{theorem}{reformulationofopzero}
    \label{thm:reformulation_opzero}
    When Asm.\ref{ass:wrongly_break_ties} holds, the original optimization problem $(OP_0)$ is equivalent to the following problem:
    \begin{equation}
        \label{eq:op1}
        (OP_1)\phantom{|} \min_{f} \E{\boldsymbol{z} \sim \mathcal{D}}{\frac{1}{K} \sum_{k = 1}^{K + 1} \ell_{0-1} \left(s_y - s_{[k]}\right)}.
    \end{equation}
\end{restatable}

\begin{remark}
    Thm. \ref{thm:reformulation_opzero} bridges $\atopk$ optimization with two classic problems: AUC optimization and label ranking. Specifically, AUC optimization is a binary classification problem with a pairwise ranking objective \cite{DBLP:conf/ijcai/LingHZ03}:
    $$
        (OP_{AUC})\phantom{|} \min_{f} \E{\boldsymbol{x^{+}} \sim \mathcal{D}^{+} \atop \boldsymbol{x^{-}} \sim \mathcal{D}^{-}} { \ell_{0-1}( f(\boldsymbol{x^{+}}) - f(\boldsymbol{x^{-}}) ) },
    $$
    where $\mathcal{D}^{+}, \mathcal{D}^{-}$ are the distribution of positive and negative samples, respectively. If we view the ground-truth label as the positive instance and the other labels as the negative ones, then $\atopk$ optimization is similar to \textbf{a partial variant of AUC optimization in the context of label ranking}. Although partial AUC optimization has been well studied \cite{DBLP:conf/kdd/NarasimhanA13,DBLP:conf/icml/NarasimhanA13,DBLP:journals/neco/Narasimhan017,DBLP:conf/icml/0001XBHCH21}, there still exist two main differences:
    \begin{itemize}
        \item The objective of AUC optimization assumes that there exist no ties in $f(\boldsymbol{x})$, which contradicts Asm.\ref{ass:wrongly_break_ties}.
        \item The expectation operation of AUC involves two samples, while in $\atopk$ it only involves one label.
    \end{itemize}
    These differences will lead to \textbf{different theoretical results and practical implementation}. Please refer to Sec. \ref{sec:surrogate} and Sec. \ref{sec:generalization} for the details.
\end{remark}

According to Eq.(\ref{eq:op1}), the main challenges for directly minimizing $\mathcal{R}_K(f)$ are two-fold: 
\begin{itemize}
    \item[\textbf{(C1)}] The loss function $\ell_{0-1}$ is not differentiable;
    \item[\textbf{(C2)}] The data distribution $\mathcal{D}$ is unavailable.
\end{itemize}
Next, we will present the solutions to \textbf{(C1)} and \textbf{(C2)} in Sec.~\ref{sec:surrogate} and Sec.~\ref{sec:generalization}, respectively. Finally, we will obtain an end-to-end framework that optimizes the empirical risk of $\atopk$ efficiently.

\section{Surrogate Risk Minimization for AUTKC}
\label{sec:surrogate}
    In this section, our task is to handle \textbf{(C1)}. Specifically, the discrete $\ell_{0-1}$ in $(OP_1)$ will be replaced by a differentiable surrogate loss $\ell$, leading to the following surrogate risk minimization problem:
\begin{equation}
    (OP_2)\phantom{|} \min_{f} \mathcal{R}_{K}^{\ell}(f) := \E{\boldsymbol{z} \sim \mathcal{D}} {\frac{1}{K} \sum_{k=1}^{K+1} \ell \left(s_y - s_{[k]}\right)}.
\end{equation}
But,
$$
    \text{(\textbf{Q2-1}) \textit{What kind of surrogate loss should we select?}}
$$
We next answer this question in three steps. First, we study the Bayes optimal function for $\atopk$, that is, the optimal solution of $(OP_1)$. Next is the condition of $\ell$ such that optimizing $(OP_2)$ leads to the Bayes optimal function. In other words, when the surrogate risk is Fisher consistent with the $\ell_{0-1}$-based $\atopk$ risk. Finally, we answer the question \textbf{(Q2-1)} by discussing the consistency property of some common surrogate losses. 

\subsection{AUTKC Bayes Optimality}
\label{sec:bayes_optimality}
We first define the Bayes optimal score function for $\atopk$, which is the best score function we could approximate:
\begin{definition}[$\atopk$ Bayes optimal]
    The score function $f^*: \mathcal{X} \to \mathbb{R}^{C}$ is $\atopk$ Bayes optimal if 
    \begin{equation}
        \label{eq:bayes_optimal}
        f^* \in \arg \inf_f \mathcal{R}_K(f).
    \end{equation}
\end{definition}

\noindent Our goal is to find the solution to Eq.(\ref{eq:bayes_optimal}). To this end, we first present the top-$K$ ranking-preserving property:
\begin{definition}[Top-$K$ ranking-preserving property] 
    \label{def:rp}
    Given $\boldsymbol{a}, \boldsymbol{b} \in \mathbb{R}^{C}$, we say that $\boldsymbol{b}$ is \textit{top-$K$ ranking-preserving} with respect to $\boldsymbol{a}$, denoted as $\mathsf{RP}_K(\boldsymbol{b}, \boldsymbol{a})$, if for any $k \le K$, 
    $$\pi_{\boldsymbol{b}}(y) = k \Longrightarrow  \pi_{\boldsymbol{a}}(y) = k.$$ 
    The negation of this statement is $\lnot \mathsf{RP}_K(\boldsymbol{b}, \boldsymbol{a})$. Note that $\mathsf{RP}_K$ is transitive. In other words, if $\mathsf{RP}_K(\boldsymbol{c}, \boldsymbol{b})$ and $\mathsf{RP}_K(\boldsymbol{b}, \boldsymbol{a})$, then we have $\mathsf{RP}_K(\boldsymbol{c}, \boldsymbol{a})$.
\end{definition} 
\noindent Then, the following theorem illuminates the connection between $\mathsf{RP}_K$ and $\atopk$ optimality. Please see Appendix.\ref{sec_app:bayes_optimal} for the proof.
    \begin{restatable}{theorem}{bayesoptimal}
        \label{thm:bayes_optimal}
        The score function $f: \mathcal{X} \to \mathbb{R}^{C}$ is $\atopk$ Bayes optimal if and only if the score is top-$K$ ranking-preserving with respect to the conditional risk. More concretely, we have
        $$f(\boldsymbol{x})_{y_1} > f(\boldsymbol{x})_{y_2} \text{  if  } \eta(\boldsymbol{x})_{y_1} > \eta(\boldsymbol{x})_{y_2} \ge \eta(\boldsymbol{x})_{[K]}.$$ 
    \end{restatable}

\noindent It is worth noting that this theorem sheds more light on the advantages of optimizing $\atopk$:
\begin{figure}[!t]
        \centering
        \includegraphics[width=\linewidth]{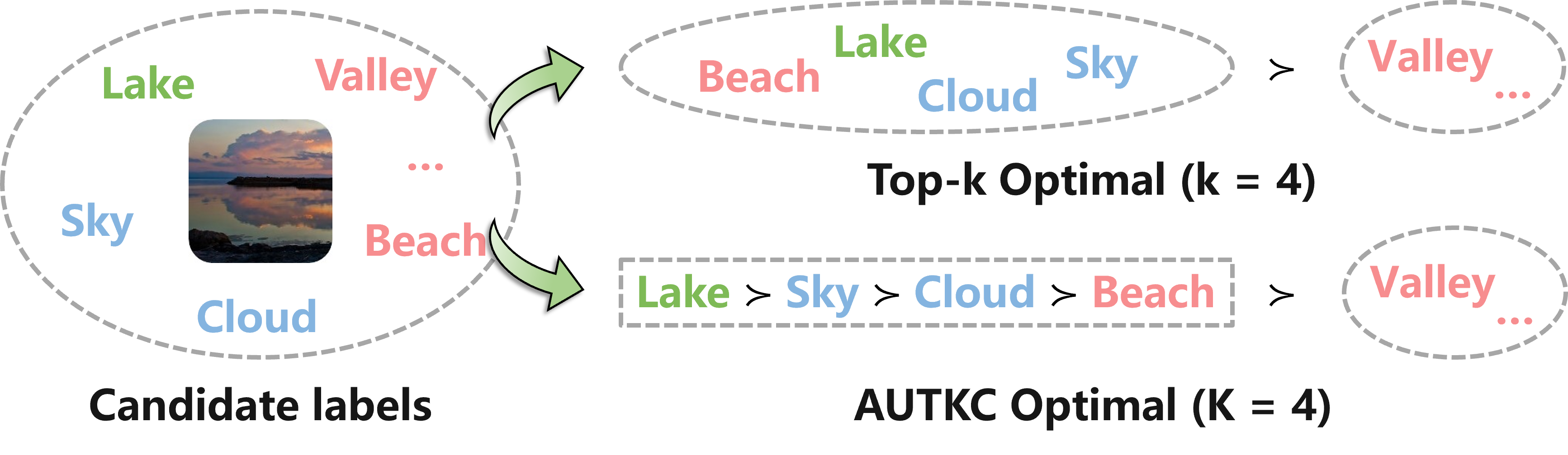}
        \caption{The comparison between $\atopk$ Bayes optimality and $\topk$ Bayes optimality. $\atopk$ optimality requires that $\textbf{\texttt{Lake}} \succ \underline{\texttt{Sky}} \succ \underline{\texttt{Cloud}} \succ \texttt{Beach} \succ \text{others}$, while $\topk$ optimality does not.}
        \label{fig:bayes}
\end{figure}

    \begin{remark}
        Top-$K$ ranking-preserving is stricter than the optimality property for $\topk$ optimization \cite{DBLP:conf/icml/YangK20}. To be specific, given any $\boldsymbol{x}$, the $\topk$ optimality property requires that 
        $$\begin{aligned}
            f(\boldsymbol{x})_y > f(\boldsymbol{x})_{[k+1]} & \text{  if  } \eta(\boldsymbol{x})_y > \eta(\boldsymbol{x})_{[k+1]}, \\
            f(\boldsymbol{x})_y < f(\boldsymbol{x})_{[k]} & \text{  if  } \eta(\boldsymbol{x})_y < \eta(\boldsymbol{x})_{[k]}. \\
        \end{aligned}$$
        In other words, this property only considers the ranking involving  $\eta(\boldsymbol{x})_{[k+1]}$ and $\eta(\boldsymbol{x})_{[k]}$. As a comparison, Top-$K$ ranking-preserving $\mathsf{RP}_K$ also considers the ranking between $f(\boldsymbol{x})_{y_1}$ and $f(\boldsymbol{x})_{y_2}$, where $\eta(\boldsymbol{x})_{y_1} > \eta(\boldsymbol{x})_{y_2} > \eta(\boldsymbol{x})_{[K]}$. As shown in Fig.\ref{fig:bayes}, $\atopk$ optimality requires that $$\textbf{\texttt{Lake}} \succ \underline{\texttt{Sky}} \succ \underline{\texttt{Cloud}} \succ \texttt{Beach} \succ \text{others},$$ while $\topk$ optimality only requires that $$\textbf{\texttt{Lake}}, \underline{\texttt{Sky}}, \underline{\texttt{Cloud}}, \texttt{Beach} \succ \text{others}.$$ 
        
        \noindent The advantage of this strict property is three-fold:
        \begin{itemize}
            \item It reveals how $\atopk$ overcomes limitation \textbf{(L1)}: optimizing $\atopk$ encourages the classes with a higher conditional probability being ranked higher, while the $\topk$ objective does not.
            \item It reveals how $\atopk$ overcomes limitation \textbf{(L2)} and escapes from the dilemma that the value of hyperparameter $k$ is hard to determine: as long as the hyperparameter $K$ is not so small, the induced optimal prediction will rank the ground-truth label and the ambiguous ones higher than the irrelevant ones.
            \item It helps us obtain a simpler criterion to select surrogate loss, whose details are presented in Sec.\ref{sec:consist_loss}.
        \end{itemize}
    \end{remark}

\begin{remark}
    If a score function $f$ is $\atopk$ Bayes optimal, no ties will exist in 
    $$
        \{f(\boldsymbol{x})_{[1]}, \cdots, f(\boldsymbol{x})_{[K]}, f(\boldsymbol{x})_{[K + 1]}\}.
    $$
    In other words, \ul{optimizing $\atopk$ encourages breaking the ties}, which is an appealing property. Meanwhile, ties might exist in 
    $$
        \{f(\boldsymbol{x})_{[K+1]}, \cdots, f(\boldsymbol{x})_{[C]}\}
    $$
    since it does not break the top-$K$ ranking-preserving property.
\end{remark}

\subsection{Consistency analysis for surrogate risk minimization}
\label{sec:consist_loss}
On top of Thm.\ref{thm:bayes_optimal}, we could investigate the answer to the question \textbf{(Q2-1)}. A basic requirement for $\ell$ is consistency. That is, the solution returned by minimizing $\mathcal{R}_{K}^{\ell}(f)$ should recover the Bayes optimal score function:
\begin{definition}[$\atopk$ consistency]
    The surrogate loss $\ell: \mathbb{R} \to \mathbb{R}_{+}$ is consistent with $\atopk$ if for every function sequence $\left\{ f_t \right\}_{t=1, 2, \cdots}$, we have:
    \begin{equation}
        \mathcal{R}_{K}^{\ell}(f_t) \to \inf_{f} \mathcal{R}_{K}^{\ell}(f) \Rightarrow \mathcal{R}_{K}(f_t) \to \inf_{f} \mathcal{R}_{K}(f).
    \end{equation}
\end{definition}

\noindent Next, we provide a sufficient condition for $\atopk$ consistency, which enjoys a simpler formulation than checking whether the optimal solutions of $(OP_2)$ are Bayes optimal. Please see Appendix.\ref{sec_app:condition_calibration} for the proof.
\begin{restatable}{theorem}{conditionforconsistency}
    \label{thm:condition_for_consistency}
    The surrogate loss $\ell(t)$ is $\atopk$ consistent if it is bounded, differentiable and strictly decreasing.
\end{restatable}

So far, we have discussed the consistency \textit{w.r.t.} all the measurable functions. However, common surrogate losses are not bounded. To this end, we next restrict the functions within a special function set $\mathcal{F}$, which induces the concept of $\mathcal{F}$-consistency:
\begin{definition}[$\atopk$ $\mathcal{F}$-consistency]
    The surrogate loss $\ell: \mathbb{R} \to \mathbb{R}_{+}$ is consistent with $\atopk$ if for every function sequence $\left\{ f_t \right\}_{t=1, 2, \cdots}, f_t \in \mathcal{F}$, we have:
    \begin{equation}
        \mathcal{R}_{K}^{\ell}(f_t) \to \inf_{f \in \mathcal{F}} \mathcal{R}_{K}^{\ell}(f) \Rightarrow \mathcal{R}_{K}(f_t) \to \inf_{f \in \mathcal{F}} \mathcal{R}_{K}(f).
    \end{equation}
\end{definition}
\noindent According to Thm.\ref{thm:condition_for_consistency}, the following proposition and corollary clearly hold:
\begin{proposition}
    Let $\mathcal{F}$ denote the set of functions whose outputs are bounded in $[0, 1]$. Then the surrogate loss $\ell(t)$ is $\mathcal{F}$-consistent with $\atopk$ if it is differentiable and strictly decreasing in $[0, 1]$.
\end{proposition}

\begin{corollary}
    \label{coll:consistent_loss}
    Let $\mathcal{F}$ denote the set of functions whose outputs are bounded in $[0, 1]$. Then the following statements hold:
    \begin{itemize}
        \item The square loss $\ell_{sq}(t) = (1 - t)^2$ is $\mathcal{F}$-consistent with $\atopk$.
        \item The exponential loss $\ell_{exp}(t) = \exp(-t)$ is $\mathcal{F}$-consistent with $\atopk$.
        \item The logit loss $\ell_{logit}(t) = \log (1 + \exp(-t))$ is $\mathcal{F}$-consistent with $\atopk$.
    \end{itemize}
\end{corollary}

Traditional $\topk$ optimization methods adopt the hinge loss $\ell_{hinge}\left(t\right) = \left[1 - t\right]_{+}$ as the surrogate loss, where $\left[t\right]_{+} = \max\left\{0, t\right\}$. However, we point out that its variant in $\atopk$ optimization is inconsistent, whose proof is shown in Appendix.\ref{sec_app:hinge_inconsistent}.
\begin{restatable}{theorem}{hingenotconsistent}
    \label{thm:hinge}
    The hinge loss is inconsistent with $\atopk$ when there exists $\boldsymbol{x}$ such that $ \sum_{k= K + 2}^{C} \eta(\boldsymbol{x})_{[k]} > \frac{K}{K+1}$.
\end{restatable}
\begin{sketch}
    Given any sample $\boldsymbol{x}$ satisfying the above condition and its Bayes optimal score $\boldsymbol{s}^*$, Thm.\ref{thm:hinge} is induced by the fact that we could construct a score $\boldsymbol{s}$ such that $\lnot \mathsf{RP}(\boldsymbol{s}, \eta(\boldsymbol{x}))$ but $\boldsymbol{s}$ has a smaller risk than $\boldsymbol{s}^*$ does. This construction holds by the property that $\ell_{hinge}(t) = 0$ for any $t \ge 1$. Please refer to Appendix.\ref{sec_app:hinge_inconsistent} for the details.
\end{sketch}

\section{Generalization Analysis for AUTKC Optimization}
\label{sec:generalization}
    In this section, we aim to relieve \textbf{(C2)}. Since the data distribution $\mathcal{D}$ is unknown, we turn to optimizing its empirical estimation based on the given dataset $\mathcal{S}=\{\boldsymbol{z}^{(i)}\}_{i=1}^{n}$ sampled \textit{i.i.d.} from $\mathcal{D}$. Let $s^{(i)}_{y}$ denote $f(\boldsymbol{x}^{(i)})_{y^{(i)}}$ for convenience. Then, we have the following optimization problem:
\begin{equation}
    (OP_3)\phantom{|} \min_{f \in \mathcal{F}} \hat{\mathcal{R}}_K^{\ell}(f, \mathcal{S}) := \frac{1}{n}\sum_{i=1}^{n}\sum_{k=1}^{K+1} \ell \left(s^{(i)}_{y} - s^{(i)}_{[k]}\right),
\end{equation}
The results in Sec. \ref{sec:surrogate} guarantee a small $\atopk$ risk $\mathcal{R}_K$ when we minimize the surrogate risk $\mathcal{R}_K^{\ell}$, as long as the surrogate loss $\ell$ is $\atopk$ consistent. But, 
$$\begin{aligned}
    \text{(\textbf{Q2-2}) \textit{Does }} & \text{\textit{minimizing the empirical risk }} \hat{\mathcal{R}}_K^{\ell}(f, \mathcal{S}) \\
    & \text{\textit{ lead to a small generalization error }} \mathcal{R}_K^{\ell}(f) \text{\textit{?}}
\end{aligned}$$
In other words, we should ensure $\mathcal{R}_K^{\ell}(f) - \hat{\mathcal{R}}_K^{\ell}(f, \mathcal{S})$ is bounded with high probability, and the upper bound goes to zero as the number of samples increases, \textbf{especially when the number of classes is large}. To this end, we next provide a Gaussian-complexity-based analysis for this bound in Sec. \ref{sec:general_bound}. Since deep neural networks have been one of the most popular function sets for large-scale classification, we further present its result in Sec. \ref{sec:cnn_bound}.

\subsection{General Result for AUTKC Generalization Bound}
\label{sec:general_bound}
Gaussian Complexity \cite{10.5555/2371238} is a classic measure of generalization error. Here we present its formal definition.
\begin{definition}[Gaussian Complexity] 
    Given the dataset $\mathcal{S} = \{\boldsymbol{z}^{(i)}\}_{i=1}^{n}$, let $\mathcal{F}$ be a family of functions defined on the sample space $\mathcal{Z}$. Then, the empirical Gaussian complexity is defined as:
    \begin{equation}
        \mathfrak{G}_\mathcal{S}(\mathcal{F}) := \E{\boldsymbol{g}}{\sup_{f \in \mathcal{F}} \frac{1}{n}\sum_{i=1}^{n} g^{(i)} f(\boldsymbol{z}^{(i)}) },
    \end{equation}
    where $\boldsymbol{g} = ( g^{(1)}, g^{(2)}, \cdots, g^{(n)} )$ are the independent random variables sampled from the standard normal distribution $N(0, 1)$. 
\end{definition}

\noindent Then, the following lemma provides the upper bound of $\mathcal{R}_K^{\ell}(f) - \hat{\mathcal{R}}_K^{\ell}(f, \mathcal{S})$ via Gaussian complexity.
\begin{lemma}[Generalization Bound with Gaussian Complexity \cite{10.5555/2371238}]
    \label{lem:generalization_gaussian}
    Suppose that the loss function $L$ is bounded by $B > 0$. Then, for any $\delta \in (0, 1)$, with a probability of at least $1 - \delta$, the following holds
    \begin{equation}
        \Phi(\mathcal{S}, \delta) \le \sqrt{2\pi} \mathfrak{G}_{\mathcal{S}}(L \circ \mathcal{F}),
    \end{equation}
    where 
    $$
        \Phi(\mathcal{S}, \delta) = \sup_{f \in \mathcal{F}} \left( \mathcal{R}_K^{\ell}(f) - \hat{\mathcal{R}}_K^{\ell}(f, \mathcal{S}) \right) - 3 B \sqrt{\frac{\log{\frac{2}{\delta}}}{2n}},
    $$
    and $L \circ \mathcal{F} := \{L \circ f | f \in \mathcal{F}\}$ is the set that consists of the composition of the loss function $L: \mathbb{R}^C \to \mathbb{R}$ and the score function $f \in \mathcal{F}$.
\end{lemma}

In multi-class classification, this classic bound scales \textit{w.r.t.} $\mathcal{O}(\sqrt{C})$ \cite{10.5555/2371238,10.1214/aos/1015362183,DBLP:journals/paa/Guermeur02}. This result is rather unfavorable when facing numerous classes. Thanks to the recent advance in learning theory, a Lipschitz assumption on the loss function $L$ can help eliminate this problem \cite{DBLP:journals/tit/LeiDZK19}:
\begin{assumption}[Lipschitz Continuity \textit{w.r.t.} a Variant of the $\ell_2$-Norm \cite{DBLP:journals/tit/LeiDZK19}]
    \label{ass:lipschitz}
    The loss function $L: \mathbb{R}^{C} \to \mathbb{R}$ is Lipschitz continuous \textit{w.r.t.} a variant of the $\ell_2$-norm involving index $y \in \mathcal{Y}$ if for all $\boldsymbol{s}, \boldsymbol{s}' \in \mathbb{R}^{C}$
    $$\begin{aligned}
        & \left| L(\boldsymbol{s}) - L(\boldsymbol{s}') \right| \\
        & \phantom{| L(s)} \le L_1 \left\| (s_1 - s'_1, \cdots, s_C - s'_C ) \right\|_2 + L_2 \left| s_y - s'_y \right|, 
    \end{aligned}$$
    where $L_1, L_2$ are Lipschitz constants.
\end{assumption}

\noindent Then, according to Lemma 1 in \cite{DBLP:journals/tit/LeiDZK19}, we obtain the following generalization bound: 
\begin{proposition}[Generalization Bound with the Lipschitz continuity variant of the $\ell_2$-Norm]
    \label{prop:generalization_gaussian}
    If the loss function $L$ satisfies Asm.\ref{ass:lipschitz}, then for any $\delta \in (0, 1)$, with probability of at least $1 - \delta$, we have
    \begin{equation}
        \label{eq:bound_lip}
        \Phi(S, \delta) \le 2 \sqrt{\pi} \left[L_1 C \mathfrak{G}_{\widetilde{\mathcal{S}}}(\widetilde{\mathcal{F}}) + L_2 \mathfrak{G}_{\mathcal{S}}(\widetilde{\mathcal{F}})\right].
    \end{equation}
    where 
    $$\begin{aligned}
        \widetilde{\mathcal{S}} = & \Big\{ \underbrace{(\boldsymbol{x}^{(1)}, 1), (\boldsymbol{x}^{(1)}, 2), \cdots, (\boldsymbol{x}^{(1)}, C)}_{\text{induced by } \boldsymbol{x}^{(1)}}, \\
        & \phantom{\{} \underbrace{(\boldsymbol{x}^{(2)}, 1), (\boldsymbol{x}^{(2)}, 2), \cdots, (\boldsymbol{x}^{(2)}, C)}_{\text{induced by } \boldsymbol{x}^{(2)}}, \cdots, \\
        & \phantom{\{} \underbrace{(\boldsymbol{x}^{(n)}, 1), (\boldsymbol{x}^{(n)}, 2), \cdots, (\boldsymbol{x}^{(n)}, C)}_{\text{induced by } \boldsymbol{x}^{(n)}} \Big\}, \\
        \widetilde{\mathcal{F}} = & \left\{ \tilde{f}(\boldsymbol{x}, y) := \boldsymbol{e}_y^T f(\boldsymbol{x}) | f \in \mathcal{F}, \boldsymbol{x} \in \mathcal{X}, y \in \mathcal{Y} \right\}, \\
    \end{aligned}$$
    and $\boldsymbol{e}_y = (\underbrace{0, \cdots, 0}_{y-1}, 1, \underbrace{0, \cdots, 0}_{C-y})^T$.
\end{proposition}
\noindent Note that this formulation generalizes the Thm.2 in \cite{DBLP:journals/tit/LeiDZK19}, where $f$ is restricted in linear models. Such generalization is necessary when we analyze the bound of neural networks. 

According to Proposition \ref{prop:generalization_gaussian}, we need to present the Lipschitz property of the surrogate losses mentioned in Corollary \ref{coll:consistent_loss}. To this end, we first show the following theorem, whose proof is shown in Appendix.\ref{sec_app:lip_surrogate}.
\begin{restatable}{theorem}{conditionoflipschitz}
    \label{thm:for_lip}
    Assume that $dom \ s_i \subset [0, B_s]$ for any $i=1, \cdots, C$, if $\ell(t)$ is strictly decreasing and $L_\ell$-Lipschitz continuous in the range $[-B_s, B_s]$, Then 
    $$
        L_K(\boldsymbol{s}, y) = \frac{1}{K} \sum_{k=1}^{K+1} \ell\left( s_{y} - s_{[k]} \right)
    $$
    is Lipschitz continuous \textit{w.r.t.} a variant of the $\ell_2$-norm with Lipschitz constant pair $( \frac{ L_\ell \sqrt{K+1} }{K}, \frac{ L_\ell (K+1) }{K} )$.
\end{restatable}

\noindent On top of Thm.\ref{thm:for_lip}, the following results are clear. We leave the rest of the proof in Appendix.\ref{sec_app:lip_surrogate}.
\begin{restatable}{corollary}{lipschitzofsurrogate}
    \label{coll:lipschitz_surrogate}
    Let $\overline{\boldsymbol{s}}$ denote the normalized score $\mathsf{soft}(\boldsymbol{s})$, where $\mathsf{soft}$ is the softmax function. Then, the following statements hold:
    \begin{itemize}
        \item The square loss 
        $$
            L_K^{sq}(\boldsymbol{s}, y) = \frac{1}{K} \sum_{k=1}^{K+1} \ell_{sq}(\overline{s}_y - \overline{s}_{[k]})
        $$ 
        is Lipschitz continuous \textit{w.r.t.} a variant of the $\ell_2$-norm with Lipschitz constant pair $$( \frac{2\sqrt{2(K+1)}}{K}, \frac{2\sqrt{2}(K+1)}{K} ).$$
        \item The exponential loss 
        $$
            L_K^{exp}(\boldsymbol{s}, y) = \frac{1}{K} \sum_{k=1}^{K+1} \ell_{exp}(\overline{s}_y - \overline{s}_{[k]})
        $$
        is Lipschitz continuous \textit{w.r.t.} a variant of the $\ell_2$-norm with Lipschitz constant pair $$( \frac{e\sqrt{2(K+1)}}{2K}, \frac{e\sqrt{2}(K+1)}{2K} ).$$
        \item The logit loss 
        $$
            L_K^{logit}(\boldsymbol{s}, y) = \frac{1}{K} \sum_{k=1}^{K+1} \ell_{logit}(\overline{s}_y - \overline{s}_{[k]})
        $$
        is Lipschitz continuous \textit{w.r.t.} a variant of the $\ell_2$-norm with Lipschitz constant pair $$( \frac{\sqrt{2(K+1)}}{2eK\ln 2}, \frac{\sqrt{2}(K+1)}{2eK \ln 2} ).$$
    \end{itemize}
\end{restatable}

\subsection{Generalization Bound for Convolutional Neural Networks}
\label{sec:cnn_bound}
According to Proposition \ref{prop:generalization_gaussian} and Thm.\ref{thm:for_lip}, we could obtain the final generalization bound as soon as a proper upper bound of Gaussian complexity is available. To this end, we first present an important technique for the bound of Gaussian complexity named chaining in Sec.\ref{sec:chaining}. Then, the final result of convolutional neural networks is available under the given regularization condition in Sec.\ref{sec:bound_cnn}.

\subsubsection{Chaining Bound for Gaussian Complexity}
\label{sec:chaining}
Similar to Gaussian complexity, covering number \cite{wainwright2019high} is another technique to measure the complexity of the given (pseudo)metric space:
\begin{definition}[$\epsilon$-covering]
    \label{def:covering}
    Let $\{\mathcal{F}, d\}$ be a (pseudo)metric space and $\mathcal{B}(f, r)$ denote the ball of radius $r$ centered at $f \in \mathcal{F}$. Given $\Theta \subset \mathcal{F}$, we say $\mathcal{F}_m = \{f_i\}_{i=1}^{m}$ is $\epsilon$-covering of $\Theta$ if $\Theta \subset \bigcup_{i=1}^{m} \mathcal{B}\left(f_i, \epsilon\right)$. In other words, for any $\theta \in \Theta$, there exists $f_i \in \mathcal{F}_m$ such that $d(\theta, f_i) < \epsilon$.
\end{definition}

\begin{definition}[Covering Number]
    On top of Def. \ref{def:covering}, covering number of $\Theta$ is defined as the minimum cardinality of any $\epsilon$-covering of $\Theta$:
    \begin{equation}
        \mathfrak{C}(\Theta, \epsilon, d) := \min \{m: \mathcal{F}_m \text{ is } \epsilon\text{-covering of } \Theta\}.
    \end{equation}
\end{definition}

Our task is to bound the Gaussian Complexity with covering number. To this end, we first check the sub-Gaussian property of Gaussian Complexity. See Appendix.\ref{sec_app:subgaussian} for the proof.
\begin{definition}[Sub-Gaussian Process \textit{w.r.t.} the Metric $d_X$] 
    \label{def:subgaussian}
    A collection of zero-mean random variables $\{X_f, f \in \mathcal{F}\}$ is a sub-Gaussian process \textit{w.r.t.} a metric $d_X$ on $\mathcal{F}$ if 
    $$
        \mathbb{E}[e^{\lambda (X_f - X_{f'})}] \le e^{\frac{\lambda^2 d_X^2(f, f')}{2}}, \forall f, f' \in \mathcal{F}, \lambda \in \mathbb{R}.
    $$
\end{definition}

\begin{restatable}[Sub-Gaussian Property of Gaussian Complexity]{proposition}{subgaussian}
    \label{prop:sub_gaussian}
    For any $f \in \mathcal{F}$, let
    $$
        T_{f}(\boldsymbol{g}) = \sum_{i=1}^{n} \sum_{j=1}^{C} g_{i, j} s^{(i)}_{j},
    $$
    Then, $Z := \frac{1}{\sqrt{nC}} (T_{f}(\boldsymbol{g}) - T_{f'}(\boldsymbol{g}))$ is a sub-Gaussian process with metric
    $$
        d_{\infty, \mathcal{S}} (s, s') = \max_{\boldsymbol{z} \in \mathcal{S}} |s_y - s'_y|,
    $$
    where $s^{(i)}_{j} = f(\boldsymbol{x}^{(i)})_j$.
\end{restatable}

\noindent Since $d_{\infty, \mathcal{S}} (\boldsymbol{s}, \boldsymbol{s}') \le 1$ and
$$
    \mathfrak{G}_{\tilde{S}}(\mathcal{F}) =  \E{\boldsymbol{g}}{\sup_{f \in \mathcal{F}} \frac{1}{nC} T_{f}(\boldsymbol{g}) },
$$ 
we have the chaining bound of Gaussian Complexity according to Dudley’s entropy integral \cite{wainwright2019high}:
\begin{proposition}[Chaining bound of Gaussian Complexity]
    \label{prop:chaining_gaussian}
    \begin{equation}
        \label{equ:chaining_gaussian}
        \mathfrak{G}_{\tilde{S}}(\mathcal{F}) \le \inf_{\alpha \ge 0} \left( C_1\alpha + \frac{C_2}{\sqrt{nC}} \int_\alpha^{1} \sqrt{\log{ \mathfrak{C}(\mathcal{F}, \epsilon, d_{\infty, \mathcal{S}}) }} d\epsilon \right),
    \end{equation}
    where $C_1, C_2$ are two constants. 
\end{proposition}

\subsubsection{Practical Result for Convolutional Neural Networks}
\label{sec:bound_cnn}
On top of Eq.(\ref{equ:chaining_gaussian}), we are now ready to present the generalization bound of a family of neural networks, where both convolution layers and fully-connected layers exist. In fact, it is an application of the recent advance in \cite{DBLP:conf/iclr/LongS20}.

\noindent \textbf{The family of neural networks.} The network consists of $L_c$ convolutional layers and followed by $L_f$ fully-connected layers. Each convolutional layer consists of a convolution followed by an activation function and an optional pooling operation. All the convolutions use zero-padding \cite{deeplearning} with the kernel ${KN}^{(l)} \in \mathbb{R}^{p_l \times p_l \times c_{l-1}\times c_{l}}$ for layer $l \in \{1, \cdots, L_c\}$. All the activation functions and pooling operations are $1$-Lipschitz continuous. Meanwhile, we denote the parameters of the $l$-the fully-connected layer as $V^{(l)}$. Finally, Let $\Theta = \{{KN}^{(1)}, \cdots, {KN}^{(L_c)}, V^{(1)}, \cdots, V^{(L_f)} \}$ denote the set of all the parameters in the given network. 

\noindent \textbf{Regularization of the parameters.} The model input $\boldsymbol{x} \in \mathbb{R}^{d\times d\times c}$ satisfies $||\mathsf{vec}(\boldsymbol{x})|| \le \chi$, where $c$ is the number of channels and $\mathsf{vec}(\cdot)$ is the vectorization operation. Note that convolution is essentially a linear operation, and we denote the operator matrix of the kernel ${KN}^{(l)}$ as $\mathsf{mt}({KN}^{(l)})$. We assume that the initial parameters, denoted as $\Theta_0$, satisfy 
$$\begin{aligned}
    ||\mathsf{mt}({KN}^{(l)}_0)||_2 \le 1 + \nu, l = 1, \cdots, L_c, \\
    ||\mathsf{mt}(V^{(l)}_0)||_2 \le 1 + \nu, l = 1, \cdots, L_f. \\
\end{aligned}$$
And the distance from $\Theta_0$ to the current parameters $\Theta$ is bounded:
\begin{equation}
    \begin{split}
        ||\Theta - \Theta_0|| & := \sum_{l=1}^{L} || \mathsf{mt}({KN}^{(l)}) - \mathsf{mt}({KN}_0^{(l)}) ||_2\\
        & \phantom{\sum_{l=1}^{L} || \mathsf{mt}({KN}^{(l)})} + \sum_{l=1}^{L} || V^{(l)} - V_0^{(l)} ||_2\\
        & \le \beta.
    \end{split}
\end{equation}

Finally, let $\mathcal{F}_{\beta, \nu}$ denote the set of neural networks discribed above. According to Proposition \ref{prop:chaining_bound} in Appendix.\ref{sec_app:bound_cnn}, where the covering number in Eq.(\ref{equ:chaining_gaussian}) and the Gaussian complexity in Eq.(\ref{eq:bound_lip}) are bounded, we have the final result for the generalization bound of $\mathcal{F}_{\beta, \nu}$. See Appendix.\ref{sec_app:bound_cnn} for the details of the proof.

\begin{restatable}[Generalization bound of Convolutional Neural Networks]{theorem}{boundcnn}
    \label{thm:final_bound}
    Given the set of convolution neural networks $\mathcal{F}_{\beta, \nu}$, if the loss function $L$ satisfies Asm.\ref{ass:lipschitz}, then for any $\delta \in (0, 1)$, with a probability of at least $1 - \delta$, we have
    $$\begin{aligned}
        \Phi(S, \delta) \precsim \mathcal{O} & \left( \frac{ WL_1\sqrt{C}\log(nC B_{\beta,\nu, \chi}) }{\sqrt{n}}\right) \\ 
        & \phantom{ WL_1\sqrt{C}} + \mathcal{O}\left(\frac{ WL_2 \log(nB_{\beta, \nu, \chi}) }{\sqrt{n}}\right), 
    \end{aligned}$$
    where $B_{\beta, \nu, \chi} := \chi \beta (1 + \nu + \beta / L_{a}) ^ {L_{a}}$, $L_{a} = L_{c} + L_{f}$, and $W$ is the number of parameters in $\mathcal{F}_{\beta, \nu}$.
\end{restatable}

    \begin{remark}
        \label{rem:insensitive}
        According to Thm.\ref{thm:bayes_optimal} and the generalization bound, it seems that we can simply set $K$ as the number of classes. However, only the ranking of the ground-truth label and the ambiguous labels is of interest, which are all top-ranked classes. In this sense, a large $K$ will provide a biased estimation of the performance by including more uninformative loss terms when taking the average. Motivated by this, a moderate $K$, which is exactly equal to or slightly larger than the maximum number of relevant labels, might be a better choice. Since the maximum number of relevant labels is proportional to the number of classes, we set $K \ge p \cdot C$, where $p$ is a universal constant. Note that $L_1 \sim \mathcal{O}(\frac{1}{\sqrt{K}})$ according to Thm.\ref{thm:for_lip}. In this sense, $\Phi(\mathcal{S}, \delta)$ becomes irrelevant to $C$.
    \end{remark}

    \textbf{Computational Complexity.} On one hand, the proposed method shares a similar convergence rate with prior arts since it can be optimized efficiently via common optimizers such as SGD \cite{DBLP:conf/icml/SutskeverMDH13}. On the other hand, the additional complexity of the proposed method comes from the top-$K$ ranking operator. With the help of heapsort, the average complexity of each sample is in order of $\mathcal{O}(C\log K)$. Then, the total complexity is in order of $\mathcal{O}(nC\log K + nK)$, where $n$ is the number of samples. As analyzed in the revised remark for Thm.9, we prefer choosing a moderate $K$. Thus, compared with top-$k$ optimization, whose complexity is in order of $\mathcal{O}(nC + n)$, the complexity of the proposed method is still moderate. 

\section{Experiment}
\label{sec:experiment}
In this section, we analyze the empirical results on four benchmark datasets to validate the proposed $\atopk$ optimization framework as well as the theoretical results.
\subsection{Datasets}
We first describe the datasets used in the following discussion:
\begin{itemize}
    \item Cifar-10 \cite{krizhevsky2009learning}: This dataset\footnote{\url{https://www.cs.toronto.edu/~kriz/cifar.html} \label{foot:cifar}} consists of 60,000 32 $\times$ 32 color images in 10 classes, with 6,000 images per class. There are 50,000 and 10,000 images in the training set and the test set, respectively. At the first glance, Cifar-10 is not suitable for the context of label ambiguity since all the classes are completely mutually exclusive. For example, the class \texttt{Automobile} includes sedans, SUVs, things of that sort while the class \texttt{Truck} includes only big trucks. However, we have shown that $\atopk$ is a better metric than $\topk$ for any $k < K$ in Sec.\ref{sec:autkc_topk}. We adopt this dataset to justify this argument even when no label ambiguity involves.
    \item Cifar-100 \cite{krizhevsky2009learning}: This dataset\textsuperscript{\ref{foot:cifar}} is similar to the Cifar-10, except it has 100 classes each containing 600 images. There are 500 training images and 100 testing images per class. As the number of classes increases, label ambiguity inevitably happens. For example, the class \texttt{Sea} tends to overlap with the class \texttt{Dolphin} and \texttt{Seal}; the class \texttt{Cloud} is widely distributed in outdoor classes such as \texttt{House}, \texttt{Mountain}, \texttt{Plain} and \texttt{Skyscraper}.
    \item Tiny-imagenet-200\footnote{\url{https://www.kaggle.com/c/tiny-imagenet}}: This dataset contains 100,000 64 $\times$ 64 color images in 200 classes. Each class has 500 training images, 50 validation images and 50 test images. The classification challenge on this dataset adopts both $\mathsf{TOP}\text{-}\mathsf{1}$ and $\mathsf{TOP}\text{-}\mathsf{5}$ as the evaluation metrics. Since the test set does not provide annotations, we report the performances on the validation set.
    \item Places-365 \cite{zhou2017places}: The standard Places-365\footnote{\url{http://places2.csail.mit.edu/download.html}} dataset contains 1.8 million training images in 365 scene classes. There are 50 images per class in the validation set and 900 images per class in the test set. We use this dataset to validate the proposed framework on large-scale datasets.
\end{itemize}

\begin{table}[t]
    \renewcommand{\arraystretch}{1.7}
    \caption{The prior arts of $\topk$ optimization, where $\overline{\mathbf{1}}_y := \mathbf{1} - e_y$. In these methods, TCE is a truncated variant of the traditional cross-entropy loss; and the others are induced by the multiclass hinge loss.}
    \label{table:arts_topk}
    \centering
    \begin{tabular}{ccc}
        \toprule
        Name & Formulation & References \\
        \midrule
        \textbf{L1} & $\left[ 1 + (s_{\setminus y})_{[k]} - s_y \right]_+$ & \cite{DBLP:conf/nips/LapinHS15,DBLP:conf/iclr/BerradaZK18} \\
        \textbf{L2} & $\left[ \frac{1}{k} \sum_{j=1}^{k} ( s + \overline{\mathbf{1}}_y )_{[j]} - s_y \right]_+$ & \cite{DBLP:conf/nips/LapinHS15,DBLP:conf/cvpr/Lapin0S16,DBLP:journals/pami/LapinHS18} \\
        \textbf{L3} & $\frac{1}{k} \sum_{j=1}^{k} \left[ ( s + \overline{\mathbf{1}}_y )_{[j]} - s_y \right]_+$ & \cite{DBLP:conf/nips/LapinHS15,DBLP:conf/cvpr/Lapin0S16,DBLP:journals/pami/LapinHS18} \\
        \textbf{L4} & $\left[ \frac{1}{k} \sum_{j=1}^{k} ( 1 + (s_{\setminus y})_{[j]}) - s_y \right]_+$ & \cite{DBLP:conf/icml/YangK20} \\
        \textbf{L5} & $\left[ 1 + s_{[k+1]} - s_y \right]_+$ & \cite{DBLP:conf/icml/YangK20} \\
        \textbf{TCE} & $\log \left( 1 + \sum_{j=1}^{k} \exp (s_{[j]} - s_y) \right)$ & \cite{DBLP:conf/cvpr/Lapin0S16,DBLP:journals/pami/LapinHS18} \\
        \bottomrule
    \end{tabular}
\end{table}

\begin{table*}[htbp]
    \centering
    \renewcommand{\arraystretch}{1.05}
    \caption{The empirical results of $\atopk^\uparrow @ K$ on the four datasets. The Best and the runner-up methods on each metric are marked with {\color{Top1}\textbf{red}} and {\color{Top2}\textbf{blue}}, respectively. The best method of the competitors is marked with \underline{underline}.}
      \begin{tabular}{c|c|cc|ccc|ccc|ccc}
      \toprule
      \multirow{2}{*}[-2pt]{Type} & Dataset & \multicolumn{2}{c|}{Cifar-10} & \multicolumn{3}{c|}{Cifar-100} & \multicolumn{3}{c|}{Tiny-imagenet-200} & \multicolumn{3}{c}{Places-365} \\
      \cmidrule(){2-2}  \cmidrule(){3-4}  \cmidrule(){5-7} \cmidrule(){8-10} \cmidrule(){11-13}
      & $K$ & 3     & 5     & 3     & 5     & 10    & 3     & 5     & 10    & 3     & 5     & 10 \\
      \midrule
      \multirow{2}{*}{Traditional} & CE    & \underline{92.87}  & \underline{95.24}  & 71.81  & 76.70  & 83.05  & 81.56  & 85.34  & 89.66  & 67.72 & 74.16  & \underline{82.31} \\
      & Hinge & 92.75  & 95.20  & 71.33  & 76.65  & 83.35  & 79.85  & 83.95  & 88.74  & 66.84  & 73.38  & 81.75  \\
      \midrule
      \multirow{6}{*}{Top-$k$} & L1    & 89.58  & 93.25  & 69.90  & 76.19  & 83.19  & 81.41  & 85.22  & 89.54  & 66.50  & 72.94  & 80.55  \\
      & L2    & 92.04  & 94.70  & 71.34  & 77.19  & 83.98  & 81.43  & 85.30  & 89.64  & 67.19  & 73.78  & 81.98  \\
      & L3    & 92.56  & 94.92  & \underline{71.98}  & 76.78  & 82.81  & \underline{81.87}  & \underline{85.48}  & \underline{89.66} & \underline{67.86}  & \underline{74.22}  & 82.15  \\
      & L4    & 91.64  & 94.48  & 71.47  & \underline{77.28}  & \underline{84.05} & 81.16  & 85.11  & 89.48  & 67.27  & 73.85  & 82.06  \\
      & L5    & 88.88  & 92.79  & 68.88  & 75.27  & 82.31  & 81.24  & 85.10  & 89.43  & 65.93  & 72.84  & 81.12  \\
      & TCE   & 90.29  & 93.74  & 71.41  & 76.99  & 83.73  & 79.45  & 83.82  & 88.76  & 67.52  & 74.06  & 82.19  \\
      \midrule
      \multirow{4}{*}[-4pt]{AUTKC (Ours)} & AUTKC-Hinge & 92.93  & 95.27  & 73.29  & 78.33  & 84.57  & {\color{Top2}\textbf{82.20}} & 85.73  & 89.99 & 67.89  & 74.28  & 82.33  \\
      \cmidrule(){2-2}  \cmidrule(){3-4}  \cmidrule(){5-7} \cmidrule(){8-10} \cmidrule(){11-13}
      & AUTKC-Sq & 93.00  & 95.01 & {\color{Top2}\textbf{73.39}}  & {\color{Top2}\textbf{78.43}}  & {\color{Top2}\textbf{84.76}}  & 82.05  & 85.70  & 89.99  & 67.92  & 74.33  & {\color{Top2}\textbf{82.38}}  \\
      & AUTKC-Exp & {\color{Top2}\textbf{93.01}}  & {\color{Top2}\textbf{95.30}}  & {\color{Top1}\textbf{73.60}}  & {\color{Top1}\textbf{78.70}}  & {\color{Top1}\textbf{84.87}} & {\color{Top1}\textbf{82.87}}  & {\color{Top2}\textbf{85.77}} & {\color{Top2}\textbf{90.01}} & {\color{Top1}\textbf{67.95}}  & {\color{Top1}\textbf{74.36}}  & {\color{Top1}\textbf{82.38}}  \\
      & AUTKC-Logit & {\color{Top1}\textbf{93.05}}  & {\color{Top1}\textbf{95.30}}  & 73.29  & 78.42  & 84.69 & 82.12  & {\color{Top1}\textbf{85.81}}  & {\color{Top1}\textbf{90.05}} & {\color{Top2}\textbf{67.95}}  & {\color{Top2}\textbf{74.34}}  & 82.37  \\
      \bottomrule
      \end{tabular}
    \label{tab:overall_results_autkc}
\end{table*}

\subsection{Competitors}
Specifically, we hope to validate three arguments: (1) $\topk$ is not an ideal optimization objective, as presented in Sec.\ref{sec:motivation}; (2) $\atopk$ is a better metric than $\topk$ since the models induced by $\atopk$ optimization consistently achieve better performance; (3) The hinge loss is inconsistent in $\atopk$ optimization. Motivated by this, the competitors include: 
\begin{itemize}
    \item Traditional loss functions for multiclass classification, such as cross-entropy loss (\textbf{CE}) and multi-class hinge loss (\textbf{Hinge}).
    \item Prior arts of $\topk$ optimization, whose details are presented in Tab.\ref{table:arts_topk}. Note that \textbf{TCE} and \textbf{L1} are induced by \textbf{CE} and \textbf{Hinge} via simple top-$k$ truncation, respectively. Meanwhile, \textbf{L2}, \textbf{L3} and \textbf{L4} are hinge-like top-$k$ loss functions consistent with $\topk$ when no ties exist in the scores. \textbf{L5} is the recent consistent variant even when ties exist. The value of $k$ is searched in $\{3, 5\}$ on Cifar-10 and in $\{3, 5, 10\}$ on the other datasets.
    \item Our $\atopk$ optimization framework implemented with the hinge surrogate loss (\textbf{AUTKC-Hinge}):
            $$
                L_K^{hinge}(s, y) = \frac{1}{K} \sum_{k=1}^{K+1} \ell_{hinge}(s_y - s_{[k]}),
            $$
          where $K$ is searched in the same space as $k$.
    \item Our $\atopk$ optimization framework implemented with the square surrogate loss (\textbf{AUTKC-Sq}), the exponential surrogate loss (\textbf{AUTKC-Exp}) and the logit surrogate loss (\textbf{AUTKC-Logit}), that is, $L_K^{exp}$, $L_K^{logit}$ and $L_K^{sq}$ mentioned in Corollary \ref{coll:lipschitz_surrogate}. The search space of $K$ is same as that of $k$, and the corresponding ablation study is shown in Sec.\ref{sec:sensitivity}.
\end{itemize}

\subsection{Implementation details}
\noindent \textbf{Infrastructure.} We carry out all the experiments on an ubuntu 16.04 server equipped with Intel(R) Xeon(R) Silver 4110 CPU and an Nvidia(R) TITAN RTX GPU. The version of CUDA is 10.2 with GPU driver version 440.44. The codes are implemented via \texttt{python} (v-3.8.11) \cite{DBLP:conf/nips/PaszkeGMLBCKLGA19}, and the main third-party packages include \texttt{pytorch} (v-1.9.0), \texttt{numpy} (v-1.20.3), \texttt{scikit-learn} (v-0.24.2) and \texttt{torchvision} (v-0.10.0).

\noindent \textbf{Evaluation Metric.} Given a trained score function $f$, we evaluate its performance with $\atopk^\uparrow$ (the up the better):
$$
    \atopk^\uparrow = \frac{1}{nK} \sum_{i=1}^{n} \left| \left\{ k: f(\boldsymbol{x}_i)_{y_i} > f(\boldsymbol{x}_i)_{[k]}, k \le K+1 \right\} \right|.
$$
The value of $K$ ranges in $\{3, 5\}$ on Cifar-10 and $\{3, 5, 10\}$ on the other datasets. Meanwhile, we also record the performance at some points of the top-$k$ curve:
$$
    \topk^\uparrow = \frac{1}{n} \sum_{i=1}^{n} \I{ f(\boldsymbol{x}_i)_{y_i} > f(\boldsymbol{x}_i)_{[k]} },
$$
where the value of $k$ is in the range of $[1, 5]$ on Cifar-10 and $[1, 10]$ on the other datasets.

\noindent \textbf{Backbone and Optimization Method.} We utilize ResNet-18 \cite{DBLP:conf/cvpr/HeZRS16} as the backbone on Cifar-10, Cifar-100 and Tiny-imagenet-200. The parameters of this neural network are initialized with the checkpoint pre-trained on ILSVRC2012 \cite{DBLP:conf/cvpr/DengDSLL009}, and we finetine the whole model by \texttt{Stochastic Gradient Descent} (SGD) \cite{DBLP:conf/icml/SutskeverMDH13} with Nesterov momentum. Empirically, the batch size and momentum are set as 128 and 0.9, respectively; the learning rate at the first epoch is searched in $\{0.01, 0.001, 0.0001\}$, and decays exponentially with a ratio of 0.1 every 30 epochs; the $\ell_2$ regularization term is searched in $\{0.01, 0.001, 0.0001, 0.00001\}$. On Places-365, to scale up the training period, we extract the original images to the features with 2,048 dimensions via the official pre-trained ResNet-50\footnote{\url{https://github.com/CSAILVision/places365}}. Then, the backbone is implemented with a three-layer fully-connected neural network, whose parameters are randomly initialized. The optimization strategy is similar to that on the other datasets, except that the batch size is 4096 and the initial learning rate is searched in $\{0.1, 0.01, 0.001, 0.0001\}$. The total number of epochs is set as 50 on Tiny-imagenet-200 and 90 on the other datasets. Note that Tiny-imagenet-200 has the same domain as ILSVRC2012, which helps the training process converge in fewer epochs.

\noindent \textbf{Warm-Up Training.} Focusing on a few classes during the early training period makes it easy to suffer from over-fitting. Thus, we adopt a warm-up strategy to encourage the attention on all the classes. Specifically, the model is trained with CE at the first $E_w$ epochs. Afterward, we start the $\atopk$ training phase by optimizing the proposed objectives. The corresponding ablation study is shown in Sec.\ref{sec:sensitivity}, where $E_w$ is searched in $\{5, 10, 20, 30, 40\}$.

\begin{figure}[t]
    \centering
    \subfigure[Cifar-10]{
        \begin{minipage}[t]{0.45\linewidth}
        \centering
        \includegraphics[width=\linewidth, height=0.8\linewidth]{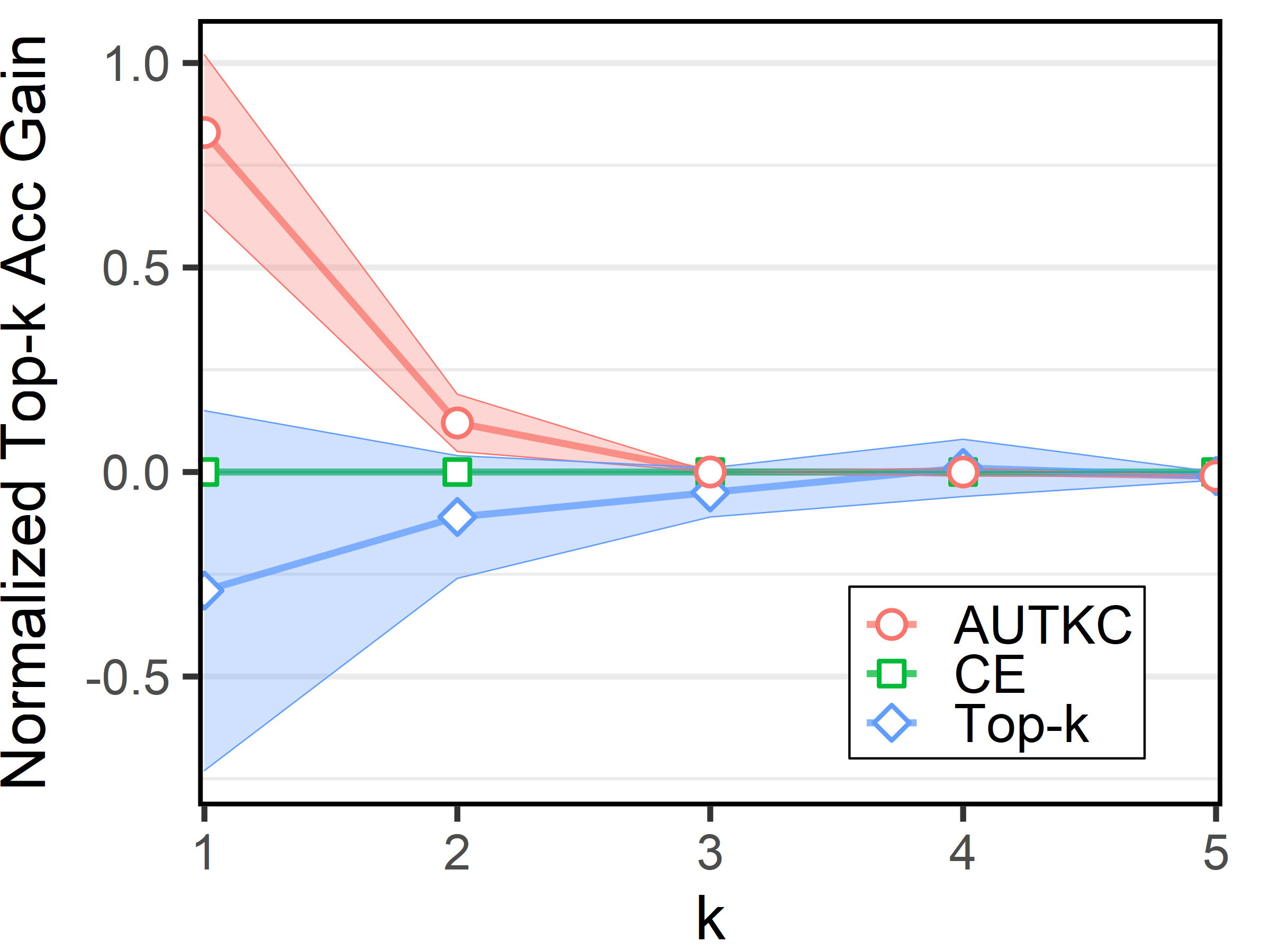}
        \end{minipage}%
    }
    \subfigure[Cifar-100]{
        \begin{minipage}[t]{0.45\linewidth}
        \centering
        \includegraphics[width=\linewidth, height=0.8\linewidth]{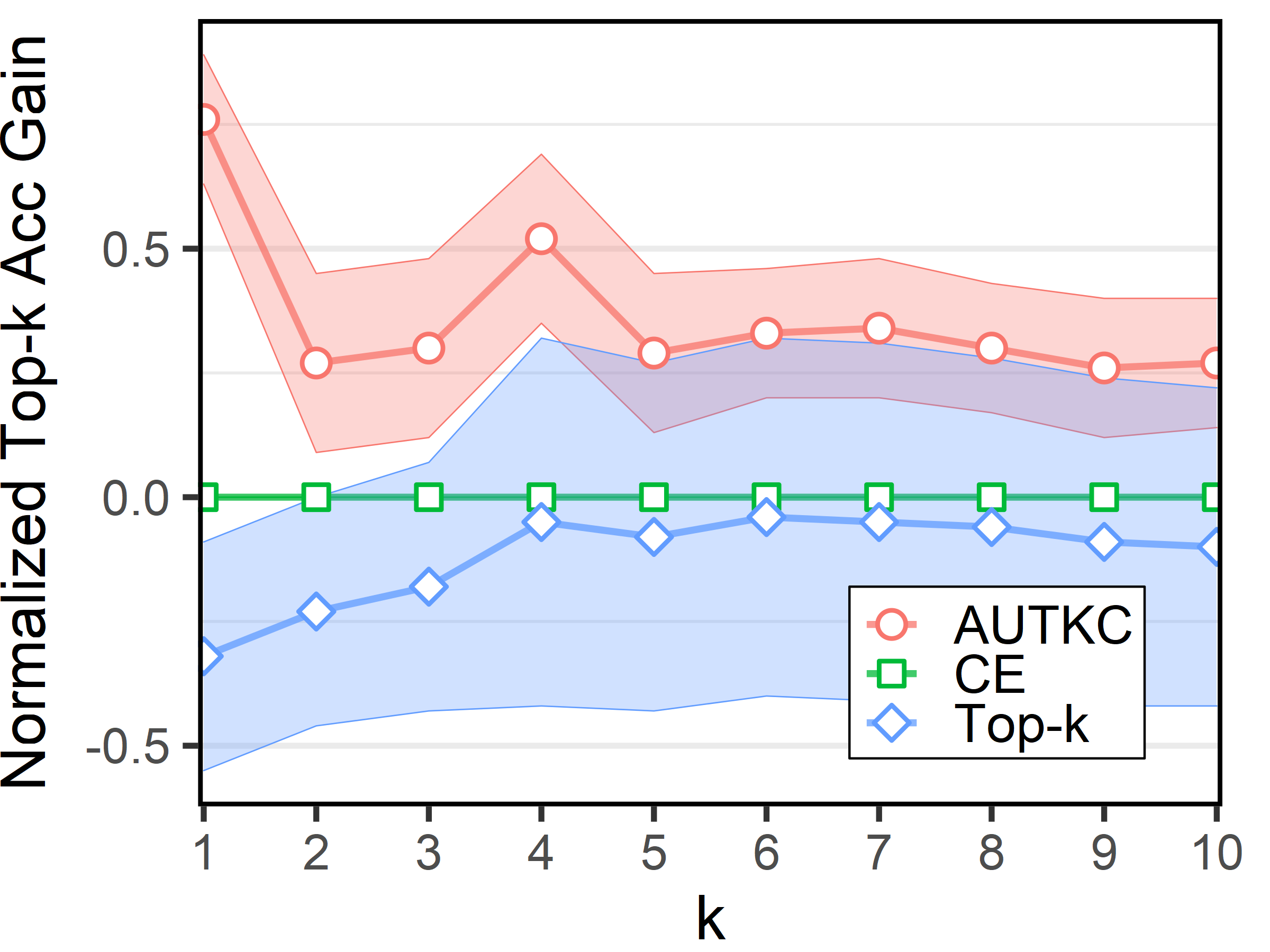}
        \end{minipage}%
    }

    \subfigure[Tiny-imagenet-200]{
        \begin{minipage}[t]{0.45\linewidth}
        \centering
        \includegraphics[width=\linewidth, height=0.8\linewidth]{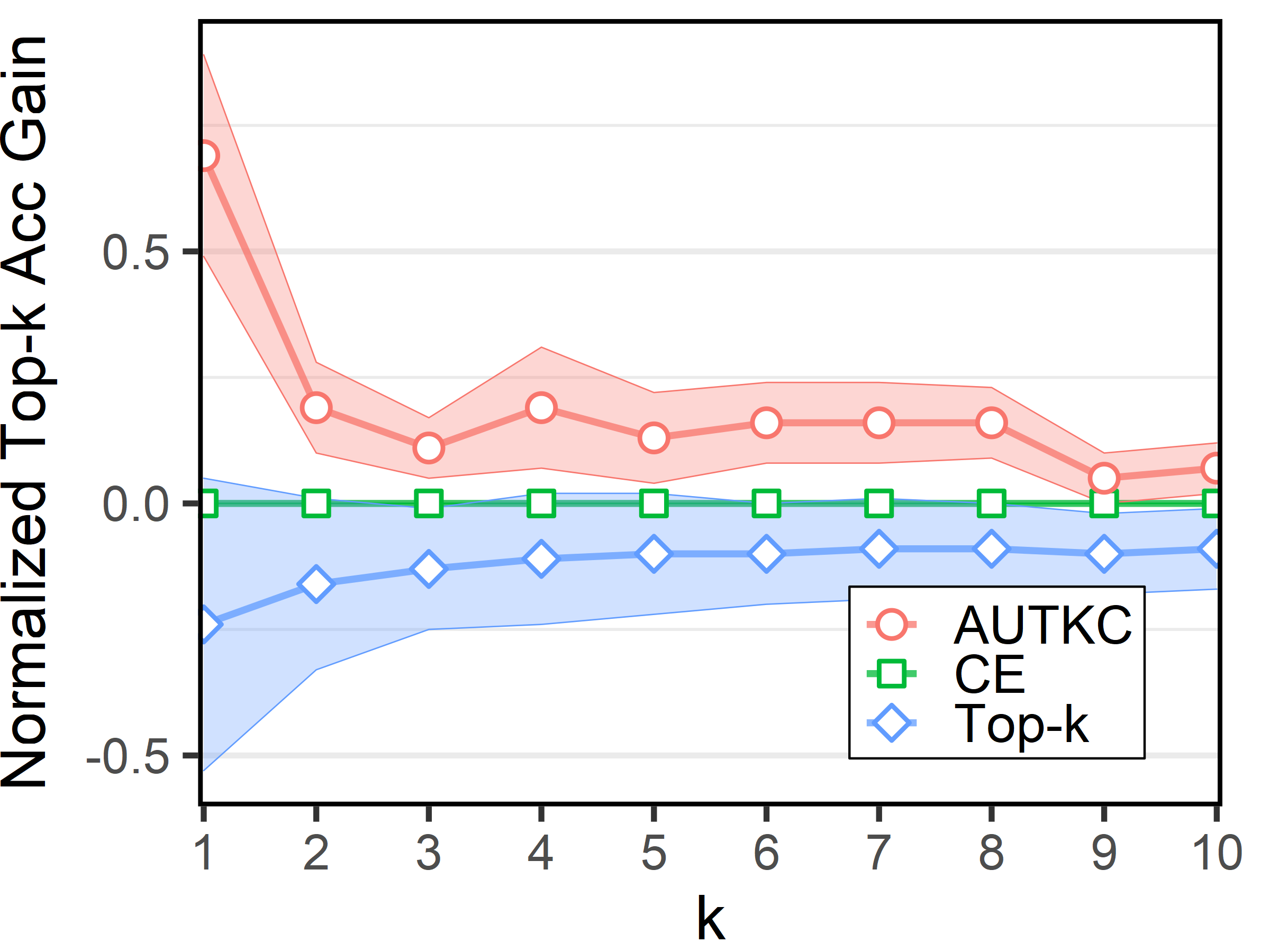}
        \end{minipage}
    }%
    \subfigure[Places-365]{
        \begin{minipage}[t]{0.45\linewidth}
        \centering
        \includegraphics[width=\linewidth, height=0.8\linewidth]{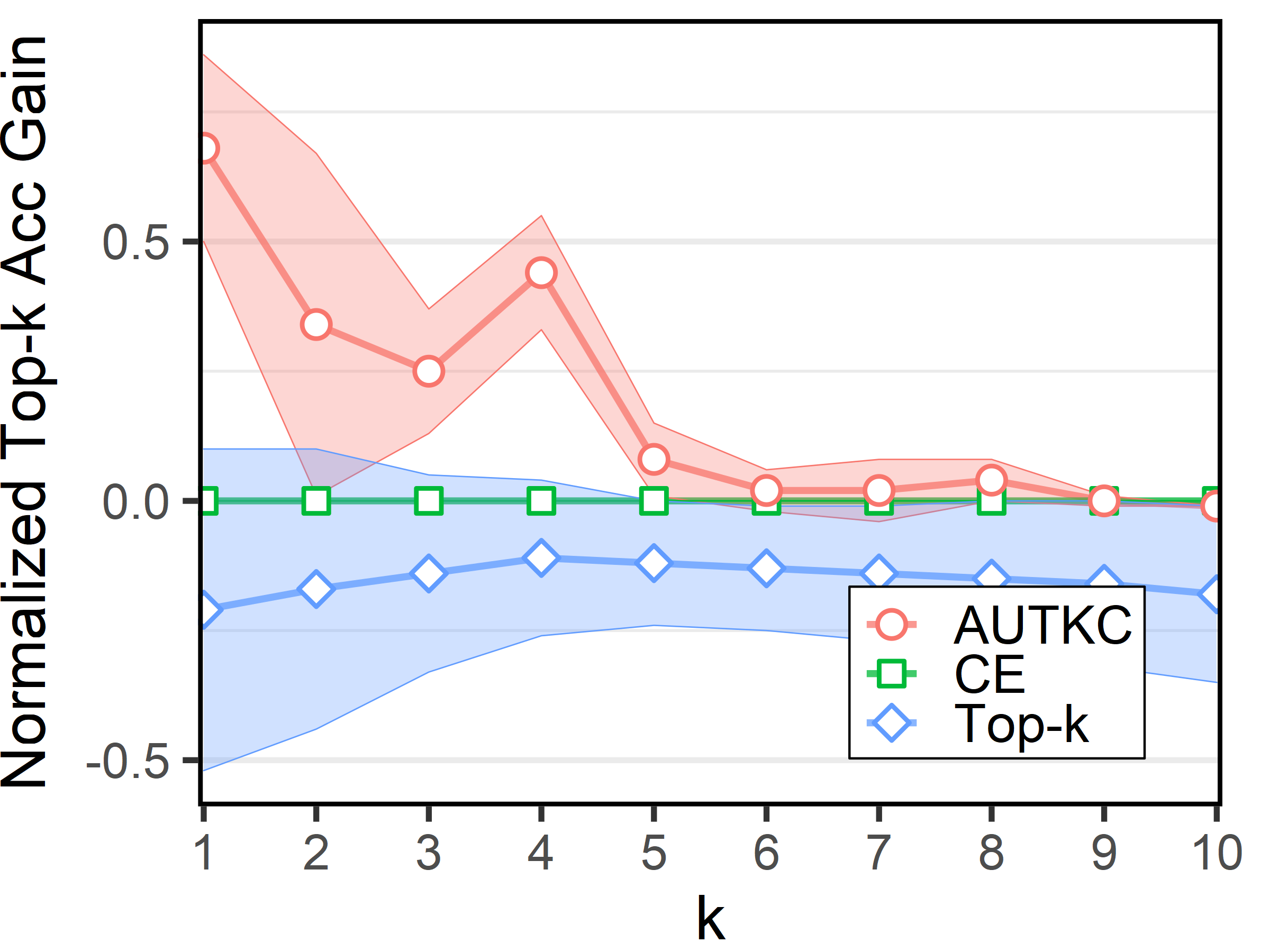}
        \end{minipage}%
    }
\centering
\caption{The normalized top-$k$ accuracy gain with respect to CE on the four datasets. The scattered points represent the average performance of each method type, and the shadow represents the standard deviation.}
\label{fig:topk_curve}
\end{figure}

\begin{figure*}[t]
    \centering
    \includegraphics[width=0.85\linewidth]{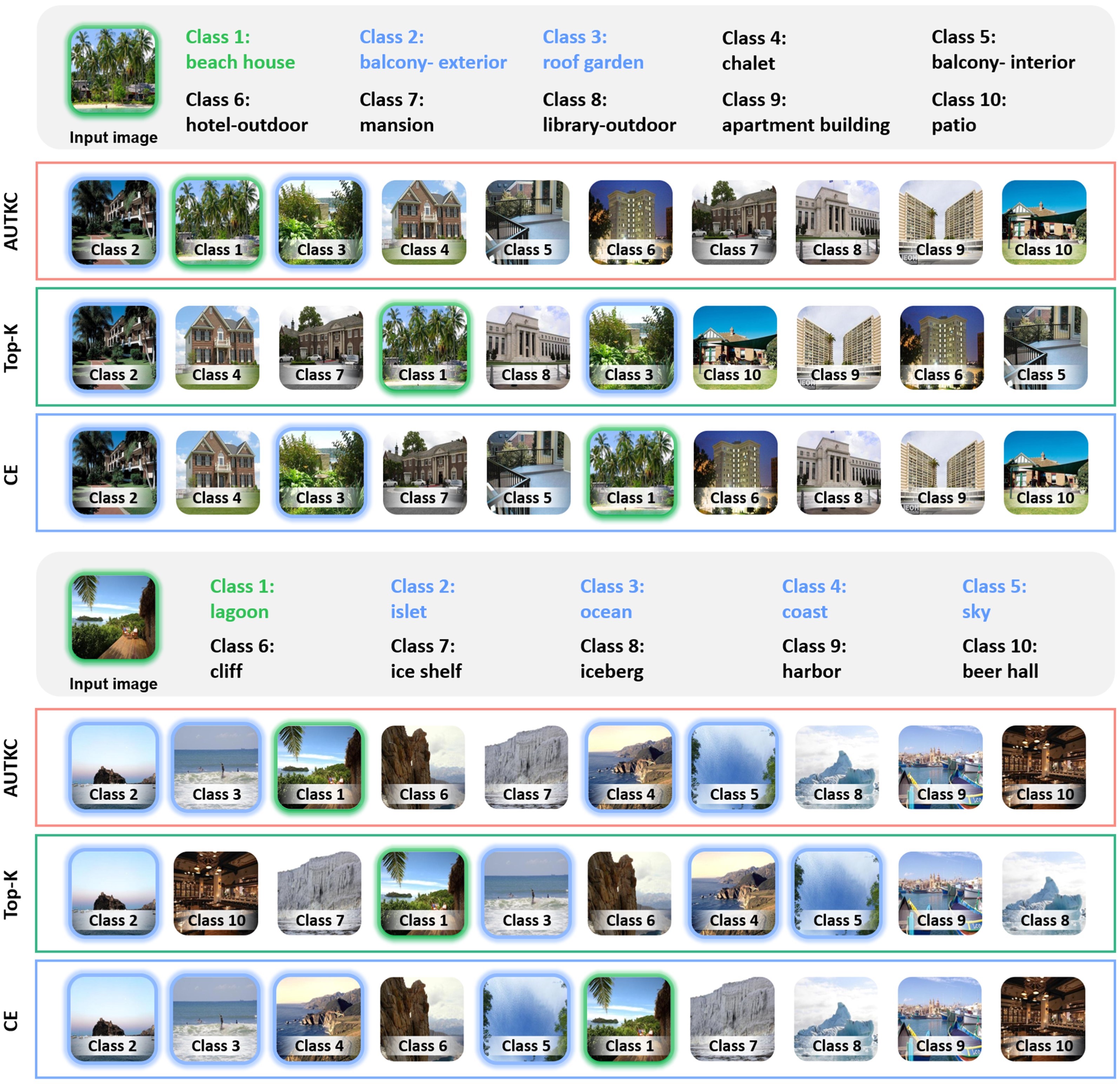}
    \caption{Case study on Places-365. For each input, we select {\color{gt} \textbf{the annotated label}}, {\color{am} \textbf{ambiguous labels}} and \textbf{some irrelevant labels} to visualize the ranking result produced by different types of method. In each ranking list, the annotated label and the ambiguous labels are highlighted with {\color{gt} \textbf{green}} and {\color{am} \textbf{blue}} box, respectively.}
    \label{fig:case}
\end{figure*}

\subsection{Results and Analysis}
\subsubsection{Overall Performance}
\noindent \textbf{The Comparison on AUTKC.} First, we record the best performance of each method on each dataset in Tab.\ref{tab:overall_results_autkc}. From the results, we have the following observations:
\begin{itemize}
    \item The AUTKC methods outperform all the competitors consistently on all the metrics on Cifar-100, Tiny-imagenet-200 and Places-365. This validates the effectiveness of the proposed optimization framework on the datasets where label ambiguity exists.
    \item A slight improvement is also attainable with AUTKC optimization methods on Cifar-10, where the number of classes is small. This result evidences our theoretical results in Sec.\ref{sec:autkc_topk} and Sec.\ref{sec:bayes_optimality}, that is, $\atopk$ is more discriminating than $\topk$.
    \item Among the four surrogate losses for AUTKC optimization, the hinge loss performs inferior to the others, which is consistent with our theoretical analysis in Sec.\ref{sec:surrogate}.
    \item Top-$k$ optimization methods achieve a slight improvement \textit{w.r.t.} CE on Cifar-100, Tiny-imagenet-200 and Places-365. This empirical result is consistent with the observations in \cite{DBLP:journals/pami/LapinHS18,DBLP:conf/icml/YangK20}.
    \item On Cifar-10, the performances of top-$k$ optimization methods fail to surpass CE, which also validates the limitation of $\topk$ shown in Sec.\ref{sec:motivation}.
\end{itemize}

\noindent \textbf{The Comparison on the Top-$k$ Curve.} Then, we plot the top-$k$ curve in Fig.\ref{fig:topk_curve} for more fine-grained comparison. Note that we summarize the results of each type of method for concise illustration, where the scattered points represent the average performance of each type of method, and the shadow represents the standard deviation. Meanwhile, the top-$k$ accuracy ranges from 50 to 90 at different $k$ points, while the performance gap between different methods is significantly smaller. This scale difference makes the performance gap almost invisible. To fix this issue, we normalize the performance gap on each dataset and each $k$ point by the following strategy:
$$\begin{aligned}
    \text{Norm} & \text{alized Top-} k \text{ Acc Gain of method } m \\ 
    & = \left\{  
        \begin{array}{ll}  
            \text{Gain}(m, k) / G_+, & \text{Gain}(m, k) > 0;\\  
            \text{Gain}(m, k) / G_-, & \text{otherwise}.\\  
        \end{array}  
    \right.  
\end{aligned}$$
where $\text{Gain}(m, k)$ is the top-$k$ accuracy gain of method $m$ with respect to CE, and
$$\begin{aligned}
    G_+ & := | \max_{m, k}\left\{ \text{Gain}(m, k) \right\} |, \\
    G_- & := | \min_{m, k}\left\{ \text{Gain}(m, k) \right\} | 
\end{aligned}
$$
are the maximum and minimum of all the gains on the data-set, respectively. According to Fig.\ref{fig:topk_curve}, we have the following observations:
\begin{itemize}
    \item In most cases, the average gains of the proposed framework are positive, especially when $k$ is small. This again suggests that optimizing AUTKC is beneficial to the top-$k$ accuracy with $k < K$.
    \item The gains of top-$k$ optimization methods tend to be negative. Moreover, we can see that the gains are increasing at the first several $k$s. All these observations again validate our analysis of the limitation of $\topk$. Review Sec.\ref{sec:motivation} for the details.
\end{itemize}

\noindent \textbf{Case Study.} After the quantitative comparison, we further visualize the ranking results of different types of methods in Fig.\ref{fig:case} to further investigate the performance. As illuminated by the demo images, the class \texttt{beach house} is semantically related to the class \texttt{balcony-exterior} and \texttt{roof garden}, and the class \texttt{lagoon} naturally overlaps with the class \texttt{islet}, \texttt{ocean}, \texttt{coast} and \texttt{sky}. As a result, all these methods fail to rank the annotated label first due to label ambiguity. For these instances, CE ranks the ambiguous labels higher than the annotated label, thus degenerating the top-$k$ performance. Meanwhile, the top-$k$ method succeeds in improving the order of the annotated label, while the irrelevant labels are ranked higher, such as \texttt{chalet} and \texttt{mansion} in the first case and \texttt{ice shelf} and \texttt{beer hall} in the second one. In contrast, the annotated label and the ambiguous labels are both ranked higher by the proposed AUTKC method. These cases again validate the superiority of AUTKC and the proposed framework.

\begin{figure}[t]
    \centering
    \subfigure[AUTKC-Hinge]{
        \begin{minipage}[t]{0.45\linewidth}
        \centering
        \includegraphics[width=\linewidth, height=0.9\linewidth]{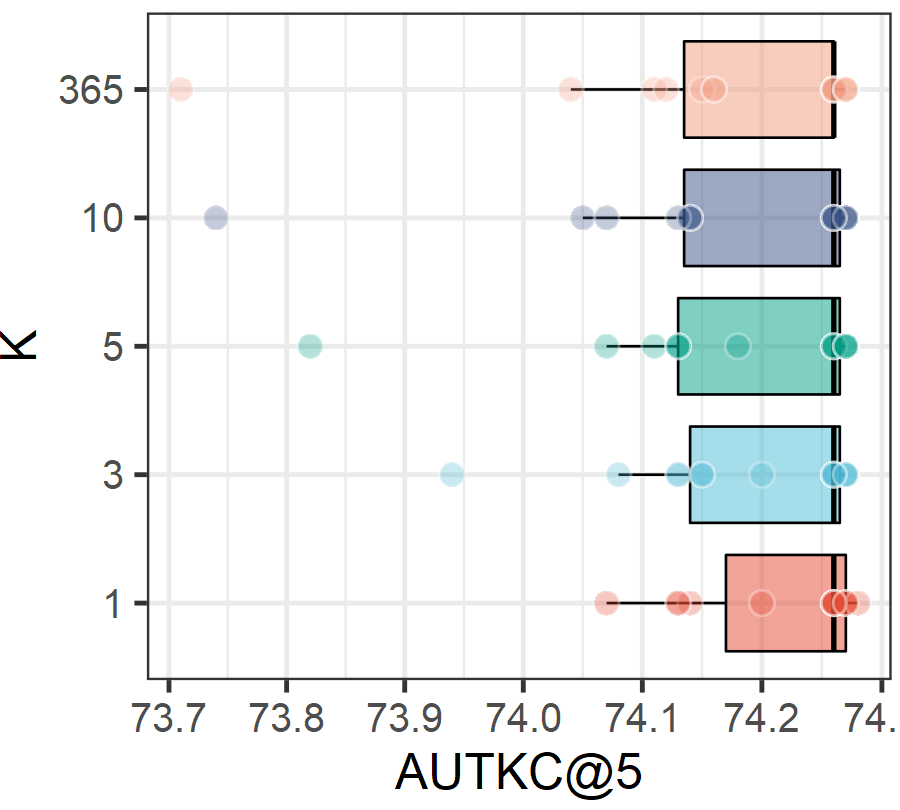}
        \end{minipage}%
    }
    \subfigure[AUTKC-Sq]{
        \begin{minipage}[t]{0.45\linewidth}
        \centering
        \includegraphics[width=\linewidth, height=0.9\linewidth]{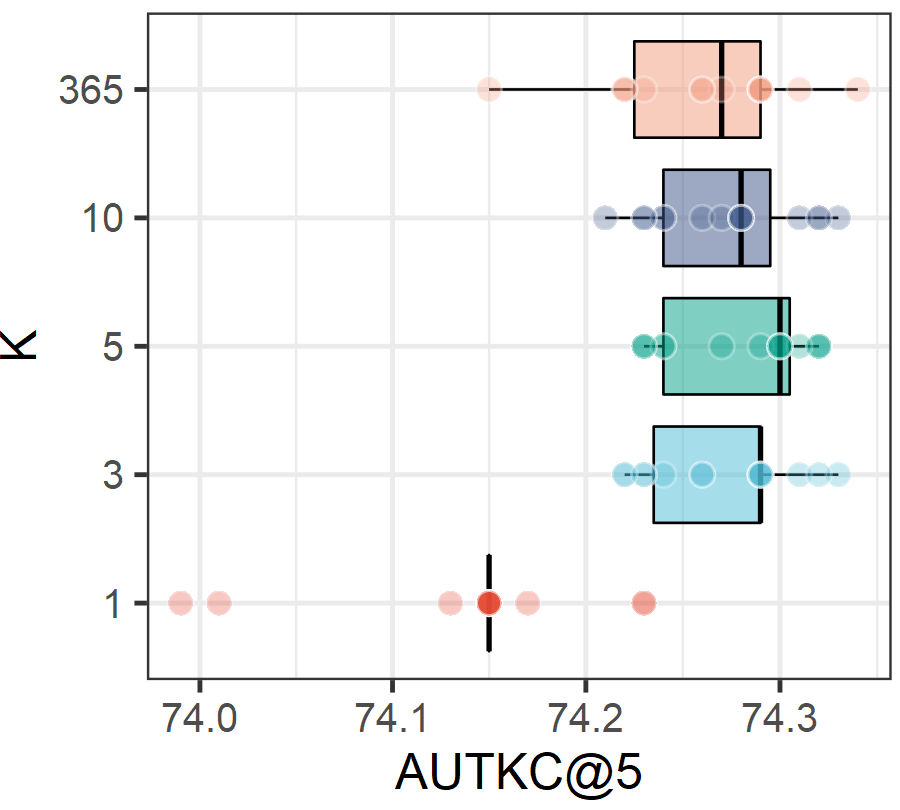}
        \end{minipage}%
    }

    \subfigure[AUTKC-Exp]{
        \begin{minipage}[t]{0.45\linewidth}
        \centering
        \includegraphics[width=\linewidth, height=0.9\linewidth]{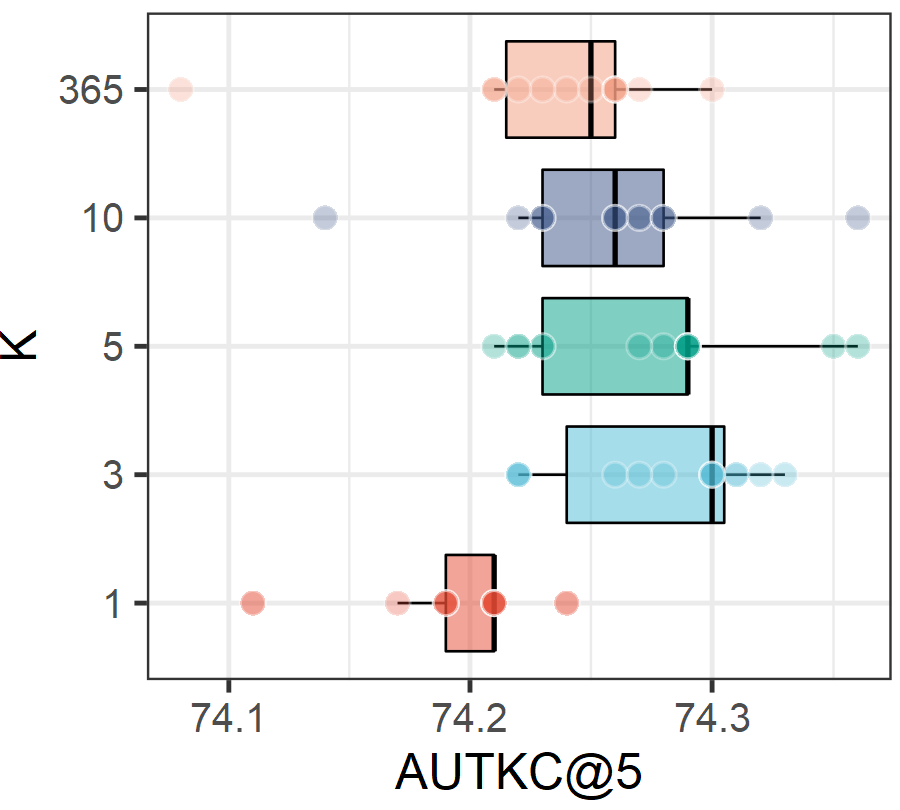}
        \end{minipage}
    }%
    \subfigure[AUTKC-Logit]{
        \begin{minipage}[t]{0.45\linewidth}
        \centering
        \includegraphics[width=\linewidth, height=0.9\linewidth]{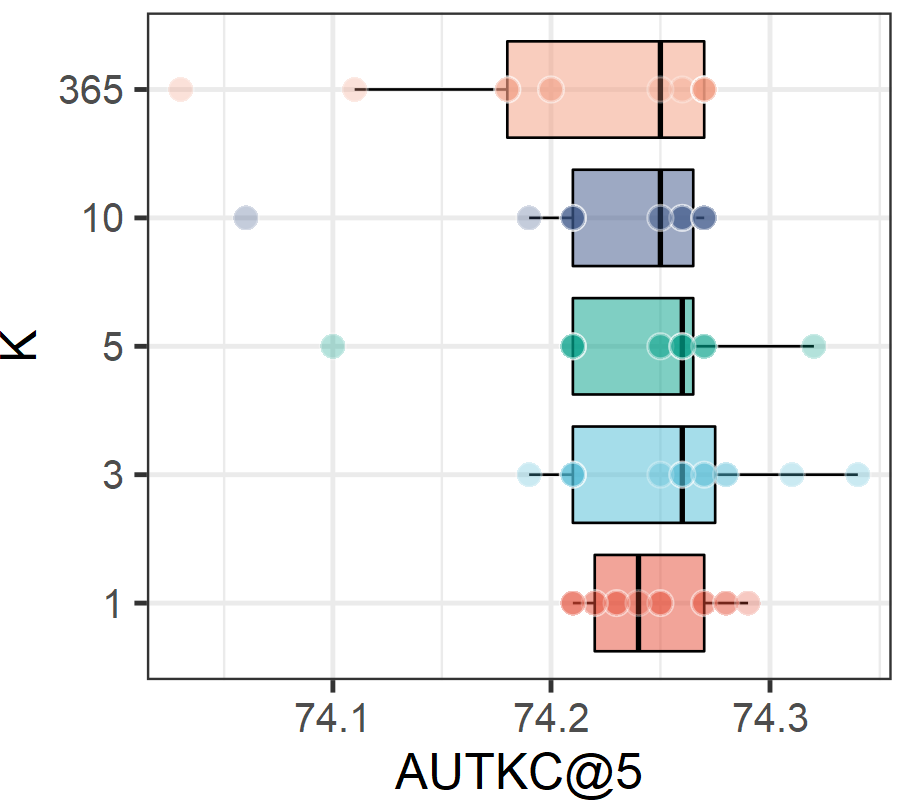}
        \end{minipage}%
    }
\centering
\caption{Sensitivity analysis of AUTKC optimization methods on $K$, where the experiments are conducted on Places-365. For each box, $K$ is fixed as the y-axis value, and the scattered points represent the performance with different hyperparameters.}
\label{fig:K}
\end{figure}

\begin{figure}[t]
    \centering
    \subfigure[AUTKC-Hinge]{
        \begin{minipage}[t]{0.45\linewidth}
        \centering
        \includegraphics[width=\linewidth, height=0.9\linewidth]{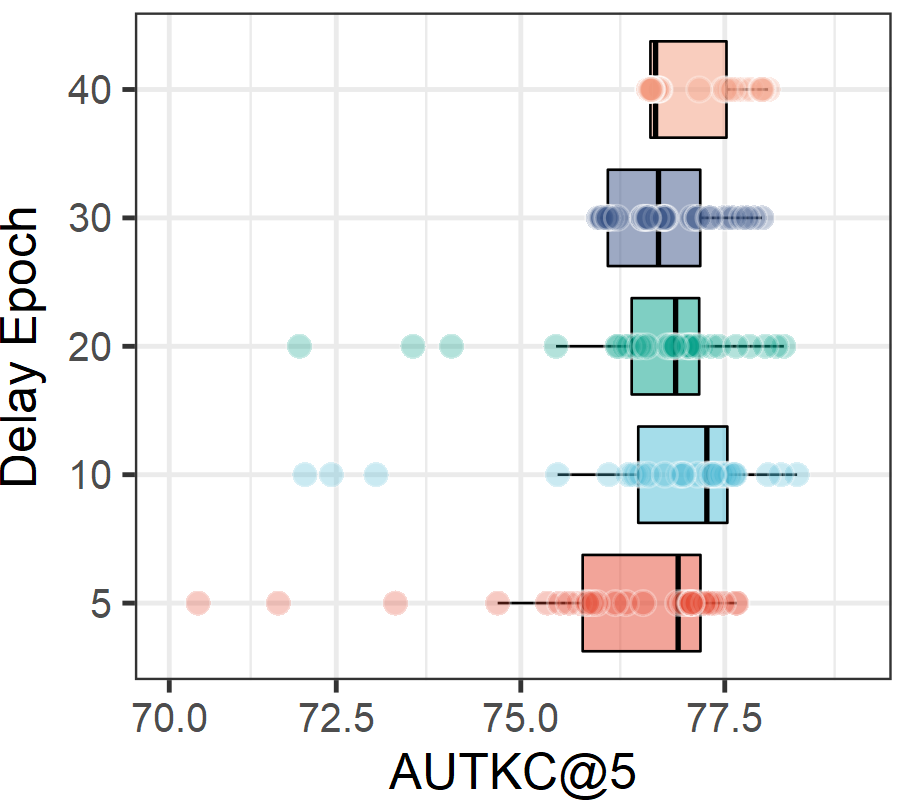}
        \end{minipage}%
    }
    \subfigure[AUTKC-Sq]{
        \begin{minipage}[t]{0.45\linewidth}
        \centering
        \includegraphics[width=\linewidth, height=0.9\linewidth]{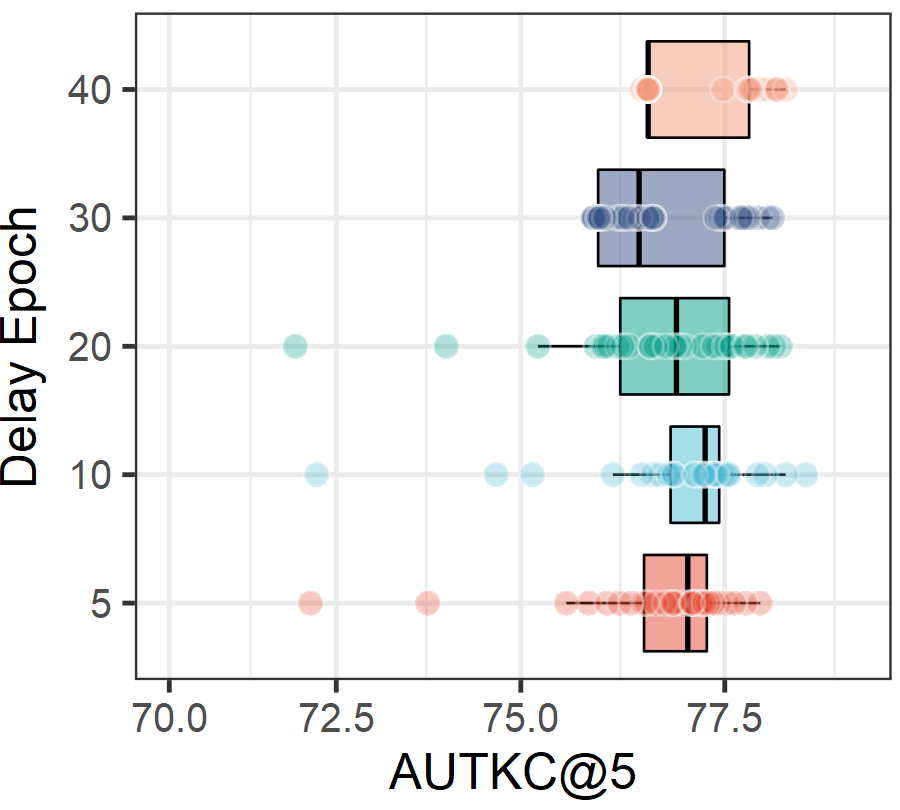}
        \end{minipage}%
    }

    \subfigure[AUTKC-Exp]{
        \begin{minipage}[t]{0.45\linewidth}
        \centering
        \includegraphics[width=\linewidth, height=0.9\linewidth]{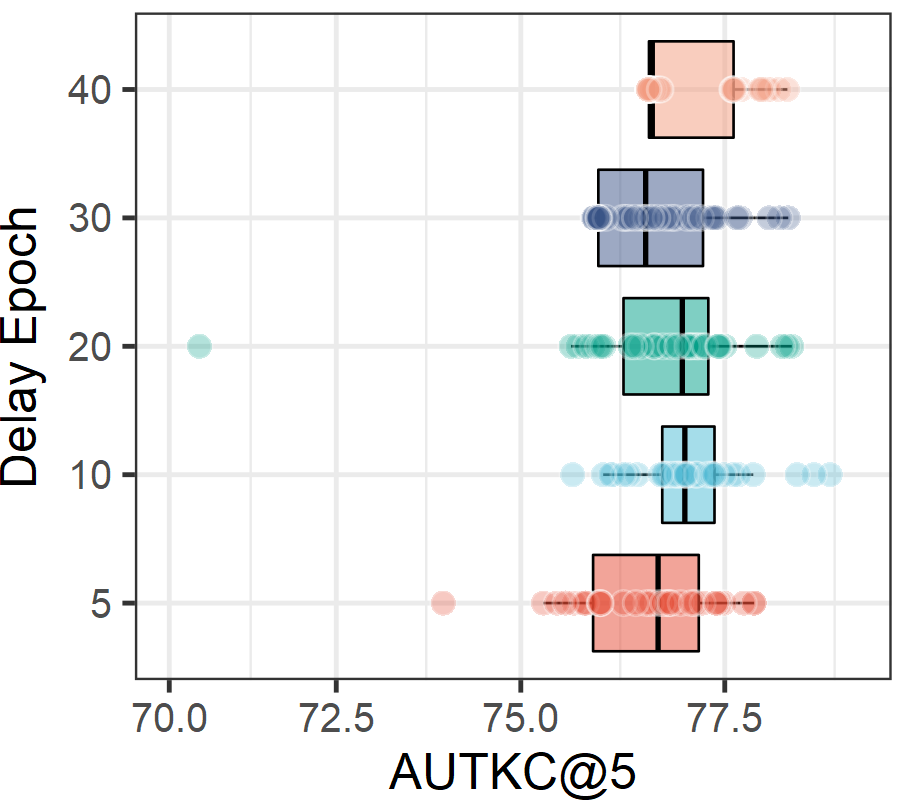}
        \end{minipage}
    }%
    \subfigure[AUTKC-Logit]{
        \begin{minipage}[t]{0.45\linewidth}
        \centering
        \includegraphics[width=\linewidth, height=0.9\linewidth]{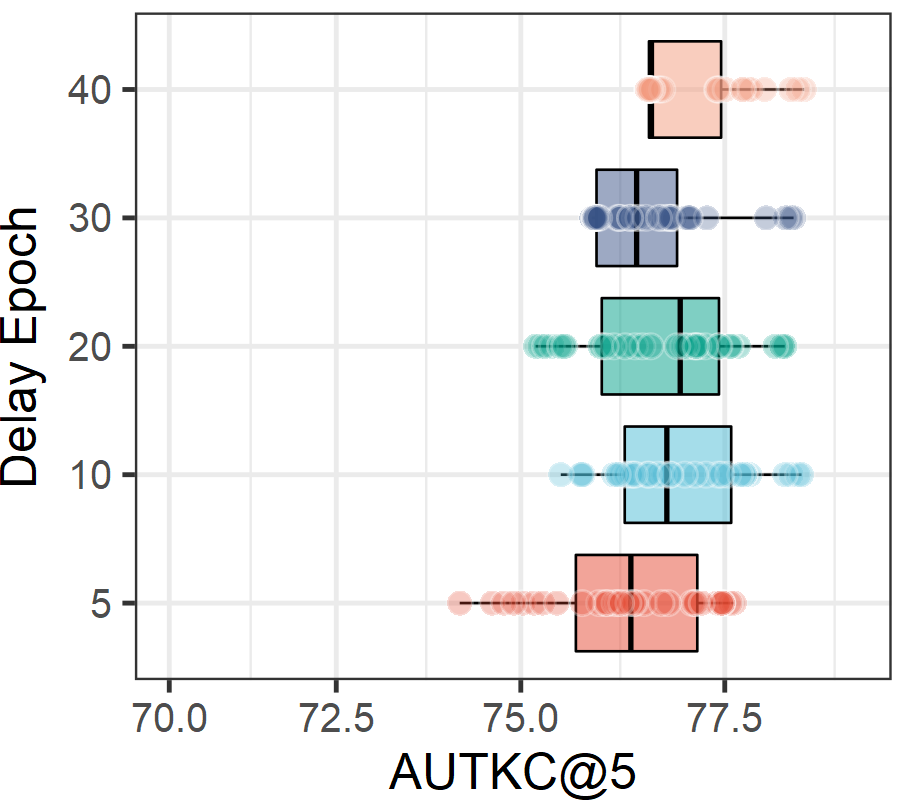}
        \end{minipage}%
    }
\centering
\caption{Sensitivity analysis of AUTKC optimization methods on $E_w$, where the experiments are conducted on Cifar-100. For each box, $E_w$ is fixed as the y-axis value, and the scattered points represent the performance with different hyperparameters.}
\label{fig:warm-up}
\end{figure}

\subsubsection{Sensitivity analysis}
\label{sec:sensitivity}
\noindent \textbf{The Effect of $K$.} In Fig.\ref{fig:K}, we show the sensitivity in terms of hyperparameter $K$ on Places-365 for the AUTKC methods with $E_w = 40$. Note that the space of $K$ includes the total number of classes. We have the following observations from the results:
\begin{itemize}
    \item AUTKC-Hinge has a different trend from the other methods. To be specific, its median performance is insensitive to the value of $K$, and the worst performance becomes worse as the value of $K$ increases. This counter-intuitive phenomenon again validates Thm.\ref{thm:hinge}, that is, AUTKC-Hinge is not AUTKC consistent.
    \item The consistent methods, \textit{e.g.}, AUTKC-Sq, AUTKC-Exp and AUTKC-Logit, achieve the best median performance on AUTKC@5 when $K$ is set as 3 or 5. This observation is consistent with our goal to design the AUTKC objective.
    \item When $K$ equals 1, the consistent methods generally have the worst median performance. This is not surprising since this case does not properly handle label ambiguity.
    \item The consistent methods with $K=365$ also perform competitively but are inferior to the ones with $K = 3$ or $K = 5$. This confirms the necessity to select a moderately large $K$.
\end{itemize}

\noindent \textbf{The Effect of $E_w$.} Furthermore, we show the sensitivity in terms of hyperparameter $E_w$ on Cifar-100 for the AUTKC methods in Fig.\ref{fig:warm-up}. Note that we record the results only when the method succeeds to converge. From the results, one can observe that:
\begin{itemize}
    \item On one hand, the increasing $E_w$ from 5 to 10 brings an increasing trend of the median performance. On the other hand, a larger $K$ generally leads to a decreasing trend of the performance deviation. This shows that the warm-up strategy is beneficial to the performance of the AUTKC methods. 
    \item The AUTKC methods achieve the best median performance when $E_w$ is equal to 10 or 20. This suggests that a moderate warm-up strategy might be a better strategy for AUTKC optimization. 
\end{itemize}

\section{Conclusion}
This paper proposes a new metric named AUTKC to overcome the limitations of the top-$k$ objective in the context of label ambiguity. An empirical surrogate risk minimization framework is further established for efficient AUTKC optimization. Theoretically, we show that the surrogate loss is Fisher consistent with the 0-1-loss-based AUTKC risk if it is differentiable, bounded, and strictly decreasing. Most common surrogate losses satisfy this condition when the model outputs are bounded, such as the square loss, the exponential loss, and the logit loss. However, the hinge loss, the standard surrogate loss for top-$k$ optimization, is inconsistent. Moreover, we construct the generalization bound for empirical AUTKC optimization, which is insensitive to the number of classes. Finally, the empirical results on four benchmark datasets justify the proposed framework and the theoretical results.

\ifCLASSOPTIONcompsoc
\section*{Acknowledgments}
\else
\section*{Acknowledgment}
\fi

This work was supported in part by the National Key R\&D Program of China under Grant 2018AAA0102000, in part by National Natural Science Foundation of China: U21B2038, 61931008, 62025604, 62132006, 6212200758 and 61976202, in part by the Fundamental Research Funds for the Central Universities, in part by Youth Innovation Promotion Association CAS, in part by the Strategic Priority Research Program of Chinese Academy of Sciences, Grant No. XDB28000000, and in part by the National Postdoctoral Program for Innovative
Talents under Grant BX2021298.

\ifCLASSOPTIONcaptionsoff
\newpage
\fi

  \bibliographystyle{IEEEtran}
  \bibliography{IEEEabrv,citations}


\begin{IEEEbiography}
    [{\includegraphics[width=1in,height=1.25in,clip,keepaspectratio]{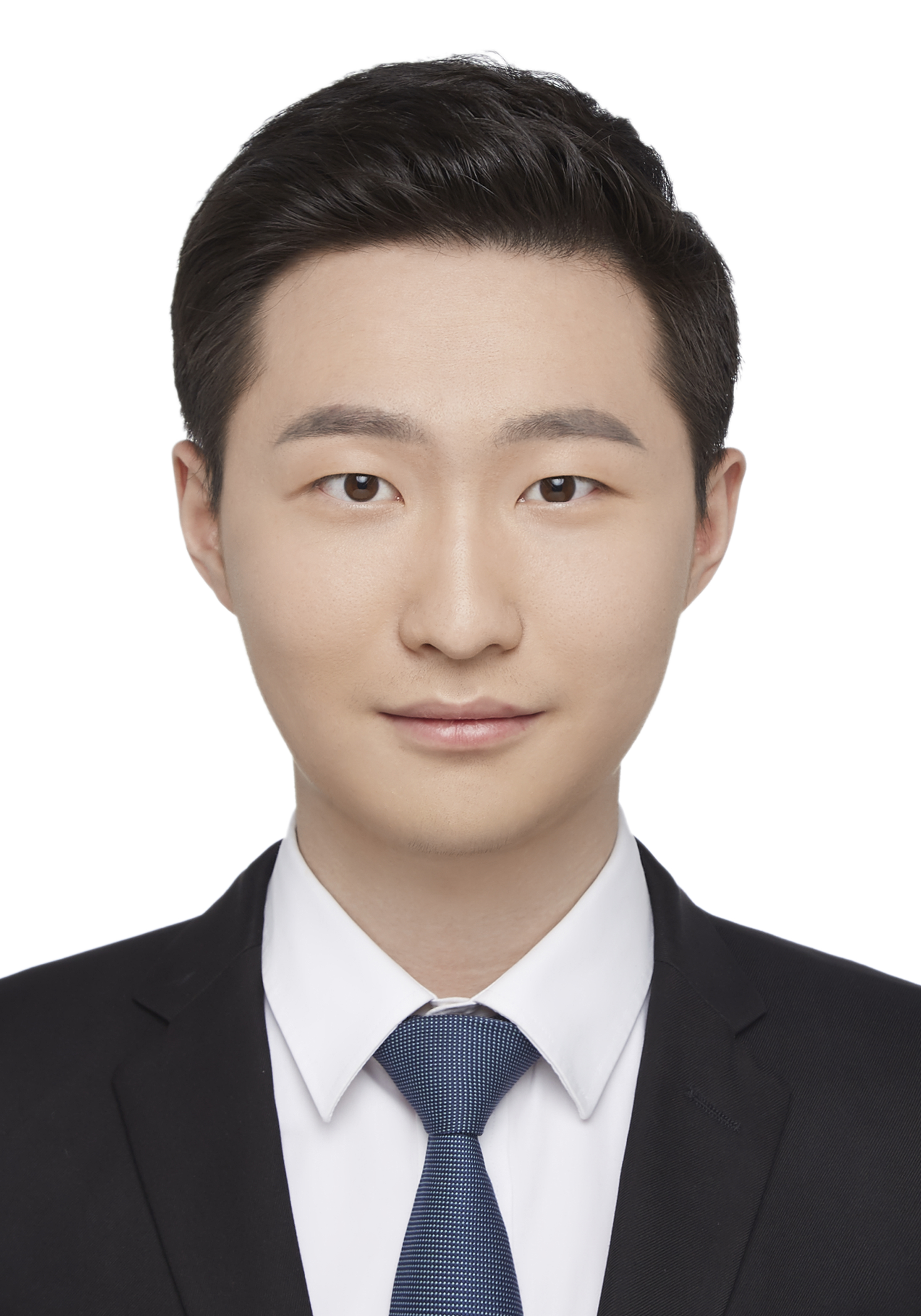}}]{Zitai Wang} received the B.S. degree in computer science and technology from Beijing Jiaotong University in 2019. He is currently pursuing a Ph.D. degree from University of Chinese Academy of Sciences. His research interests include machine learning and data mining, with special focus on top-k optimization and preference learning.
\end{IEEEbiography}

\begin{IEEEbiography}
	[{\includegraphics[width=1in,height=1.25in,clip,keepaspectratio]{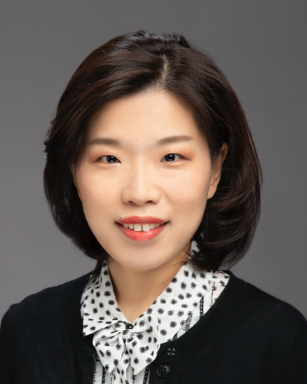}}]{Qianqian Xu} received the B.S. degree in computer science from China University of Mining and Technology in 2007 and the Ph.D. degree in computer science from University of Chinese Academy of Sciences in 2013. She is currently an Associate Professor with the Institute of Computing Technology, Chinese Academy of Sciences, Beijing, China. Her research interests include statistical machine learning, with applications in multimedia and computer vision. She has authored or coauthored 50+ academic papers in prestigious international journals and conferences (including T-PAMI, IJCV, T-IP, NeurIPS, ICML, CVPR, AAAI, etc). Moreover, she serves as an associate editor of IEEE Transactions on Circuits and Systems for Video Technology, ACM Transactions on Multimedia Computing, Communications, and Applications, and Multimedia Systems.
\end{IEEEbiography}

\begin{IEEEbiography}
	[{\includegraphics[width=1in,height=1.25in,clip,keepaspectratio]{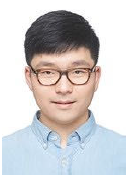}}]{Zhiyong Yang} received the M.Sc. degree in computer science and technology from University of Science and Technology Beijing (USTB) in 2017, and the Ph.D. degree from University of Chinese Academy of Sciences (UCAS) in 2021. He is currently a postdoctoral research fellow with the University of Chinese Academy of Sciences. His research interests lie in machine learning and learning theory, with special focus on AUC optimization, meta-learning/multi-task learning, and learning theory for recommender systems. He has authored or coauthored about 32 academic papers in top-tier international conferences and journals including T-PAMI/ICML/NeurIPS/CVPR. He served as a Senior PC member for IJCAI 2021 and a reviewer for several top-tier journals and conferences such as T-PAMI, TMLR, ICML, NeurIPS and ICLR.
\end{IEEEbiography}

\begin{IEEEbiography}
	[{\includegraphics[width=1in,height=1.25in,clip,keepaspectratio]{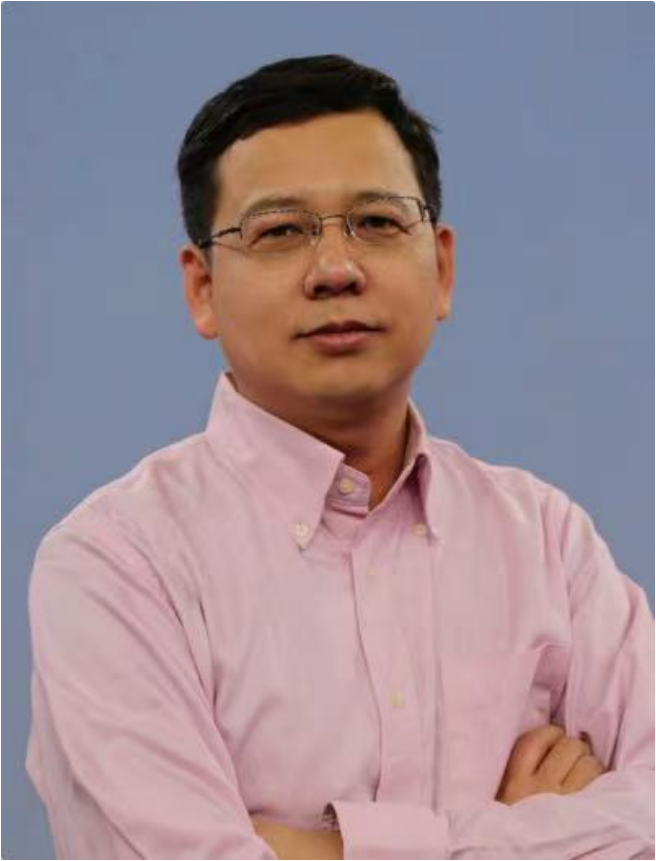}}]{Yuan He} received his B.S. degree and Ph.D. degree from Tsinghua University, P.R. China. He is a Senior Staff Engineer in the Security Department of Alibaba Group, and working on artificial intelligence-based content moderation and intellectual property protection systems. Before joining Alibaba, he was a research manager at Fujitsu working on document analysis system. He has published more than 30 papers in computer vision and machine learning related conferences and journals including CVPR, ICCV,	ICML, NeurIPS, AAAI and ACM MM. His research interests include computer vision, machine learning, and AI security.
\end{IEEEbiography}

\begin{IEEEbiography}
	[{\includegraphics[width=1in,height=1.25in,clip,keepaspectratio]{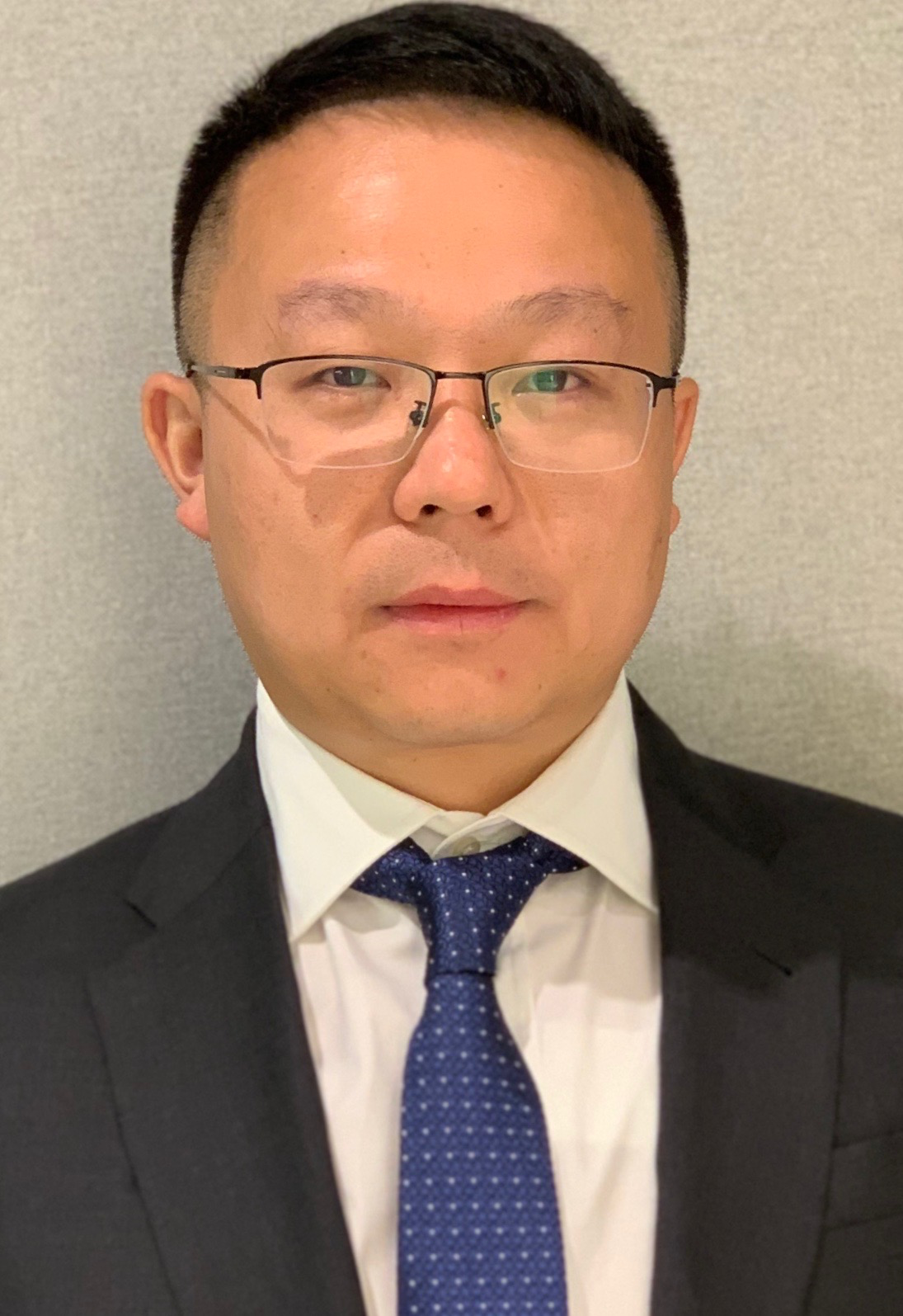}}]{Xiaochun Cao} is a Professor of School of Cyber Science and Technology, Shenzhen Campus of Sun Yat-sen University. He received the B.E. and M.E. degrees both in computer science from Beihang University (BUAA), China, and the Ph.D. degree in computer science from the University of Central Florida, USA, with his dissertation nominated for the university level Outstanding Dissertation Award. After graduation, he spent about three years at ObjectVideo Inc. as a Research Scientist. From 2008 to 2012, he was a professor at Tianjin University. Before joining SYSU, he was a professor at Institute of Information Engineering, Chinese Academy of Sciences. He has authored and coauthored over 200 journal and conference papers. In 2004 and 2010, he was the recipients of the Piero Zamperoni best student paper award at the International Conference on Pattern Recognition. He is on the editorial boards of IEEE Transactions on Image Processing and IEEE Transactions on Multimedia, and was on the editorial board of IEEE Transactions on Circuits and Systems for Video Technology.
\end{IEEEbiography}

\begin{IEEEbiography}
	[{\includegraphics[width=1in,height=1.25in,clip,keepaspectratio]{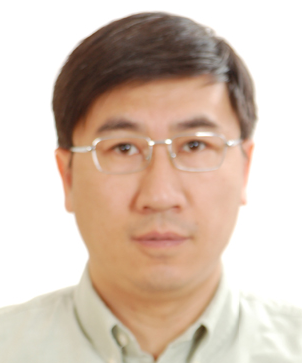}}]{Qingming Huang} is a chair professor in the University of Chinese Academy of Sciences and an adjunct research professor in the Institute of Computing Technology, Chinese Academy of Sciences. He graduated with a Bachelor degree in Computer Science in 1988 and Ph.D. degree in Computer Engineering in 1994, both from Harbin Institute of Technology, China. His research areas include multimedia computing, image processing, computer vision and pattern recognition. He has authored or coauthored more than 400 academic papers in prestigious international journals and top-level international conferences. He was the associate editor of IEEE Trans. on CSVT and Acta Automatica Sinica, and the reviewer of various international journals including IEEE Trans. on PAMI, IEEE Trans. on Image Processing, IEEE Trans. on Multimedia, etc. He is a Fellow of IEEE and has served as general chair, program chair, track chair and TPC member for various conferences, including ACM Multimedia, CVPR, ICCV, ICME, ICMR, PCM, BigMM, PSIVT, etc.
\end{IEEEbiography}


\clearpage
\onecolumn
\appendices
\section*{\textcolor{blue}{\Large{Contents}}}
\startcontents[sections]
\printcontents[sections]{l}{1}{\setcounter{tocdepth}{3}}
\newpage

\section{Metric Comparison between AUTKC and Top-k (Proof of Thm. \ref{thm:consistencydiscriminating})}
\label{sec_app:metric_comparison}
\consistencydiscriminating*
\begin{proof}
    Let $$f(\pi) := \frac{1}{K} \sum_{k=1}^{K} {\I{\pi(y) \le k}} $$ and $$g(\pi) := \I{\pi(y) \le k}.$$ 
    \textbf{Proof for consistency.} We first show $\atopk$ is consistent with $\topk$ for any $k < K$. Given two predictions $\boldsymbol{a}, \boldsymbol{b} \in \mathbb{R}^C$ and a label $y$, it is clear that $$f(\pi_{\boldsymbol{a}}) > f(\pi_{\boldsymbol{b}}) \iff \pi_{\boldsymbol{a}}(y) < \pi_{\boldsymbol{b}}(y) \text{ and } \pi_{\boldsymbol{a}}(y) \le K$$ and $$g(\pi_{\boldsymbol{a}}) > g(\pi_{\boldsymbol{b}}) \iff \pi_{\boldsymbol{a}}(y) \le k \text{ and } \pi_{\boldsymbol{b}}(y) > k.$$
    Thus, $$R = \left\{ (\pi_{\boldsymbol{a}}, \pi_{\boldsymbol{b}}) | \pi_{\boldsymbol{a}}(y) \le k, \pi_{\boldsymbol{b}}(y) > k \right\}$$ and $$\size{R} = k (C - k) > 0 $$ Meanwhile, since $$g(\pi_{\boldsymbol{a}}) < g(\pi_{\boldsymbol{b}}) \iff \pi_{\boldsymbol{a}}(y) > k \text{ and } \pi_{\boldsymbol{b}}(y) \le k,$$ we have 
    $$\begin{aligned} S & = \left\{ (\pi_{\boldsymbol{a}}, \pi_{\boldsymbol{b}}) | K \ge \pi_{\boldsymbol{a}}(y) > k \ge \pi_{\boldsymbol{b}}(y), \pi_{\boldsymbol{a}}(y) < \pi_{\boldsymbol{b}}(y) \right\} \\
        & = \varnothing .
    \end{aligned}$$ 
    Therefore, 
    \begin{equation}
        \label{equ:C}
        \mathbf{C} = \frac{\size{R}}{\size{R} + \size{S}} = 1.
    \end{equation}
    \textbf{Proof for discrimination.} Next we show that $\atopk$ is more discriminating than $\topk$. On one hand, we have $$g(\pi_{\boldsymbol{a}}) = g(\pi_{\boldsymbol{b}}) \iff \underbrace{\pi_{\boldsymbol{a}}(y), \pi_{\boldsymbol{b}}(y) \le k}_{(I)} \text{ or } \underbrace{\pi_{\boldsymbol{a}}(y), \pi_{\boldsymbol{b}}(y) > k}_{(II)}. $$ 
    For $(I)$, $$P_1 = \left\{ (\pi_{\boldsymbol{a}}, \pi_{\boldsymbol{b}}) | \pi_{\boldsymbol{a}}(y) < \pi_{\boldsymbol{b}}(y) \le k \right\},$$ thus $$\size{P_1} = (k - 1) + (k - 2) + \cdots + 1 = \frac{1}{2} k(k - 1).$$
    For $(II)$, $$P_2 = \left\{ (\pi_{\boldsymbol{a}}, \pi_{\boldsymbol{b}}) | k < \pi_{\boldsymbol{a}}(y) < \pi_{\boldsymbol{b}}(y), \pi_{\boldsymbol{a}}(y) \le K \right\},$$ thus 
    $$\begin{aligned}
        \size{P_2} & = (C - k - 1) + \cdots + (C - K) \\
        & = \frac{1}{2} (2C - k - K - 1)(K - k).
    \end{aligned} $$
    Combining $(I)$ and $(II)$, we have 
    \begin{equation}
        \label{eq:P}
        \size{P} = \size{P_1} + \size{P_2} > 0,
    \end{equation}
    since $P_1 = 0$ only if $k=1$, and $P_2 = 0$ only if $k=K$. 
    
    On the other hand, we have 
    $$\begin{aligned}
        f(\pi_{\boldsymbol{a}}) = f(\pi_{\boldsymbol{b}}) \iff & \underbrace{\pi_{\boldsymbol{a}}(y) = \pi_{\boldsymbol{b}}(y) \text{ and } \pi_{\boldsymbol{a}}(y), \pi_{\boldsymbol{b}}(y) \le K}_{(III)} \\
        & \text{ or } \underbrace{\pi_{\boldsymbol{a}}(y), \pi_{\boldsymbol{b}}(y) > K}_{(IV)}.
    \end{aligned}$$
    For $(III)$, $Q_1= \varnothing $ since $g(\pi_{\boldsymbol{a}}) > g(\pi_{\boldsymbol{b}})$ means $\pi_{\boldsymbol{a}}(y) \neq \pi_{\boldsymbol{b}}(y)$. For $(IV)$, $$Q_2 = \left\{ (\pi_{\boldsymbol{a}}, \pi_{\boldsymbol{b}}) | K < \pi_{\boldsymbol{a}}(y) \le k, \pi_{\boldsymbol{b}}(y) > K \right\} = \varnothing$$
    Then, Combining $(III)$ and $(IV)$, we have 
    \begin{equation}
        \label{eq:Q}
        \size{Q} = \size{Q_1} + \size{Q_2} = 0.
    \end{equation}
    With Eq.(\ref{eq:P}) and Eq.(\ref{eq:Q}), we have 
    \begin{equation}
        \label{equ:D}
        \mathbf{D} = \frac{\size{P}}{\size{Q}} = \infty.
    \end{equation}
    Finally, Combining Eq.(\ref{equ:C}) and Eq.(\ref{equ:D}), the proof ends.
\end{proof}

\section{Reformulation of the Original Optimization Objective (Proof of Thm. \ref{thm:reformulation_opzero})}
\label{sec_app:reformulation_opzero}
\reformulationofopzero*
\begin{proof}
    According to Assumption \ref{ass:wrongly_break_ties}, we have
    $$\begin{aligned}
        & \sum_{k \le K} { \I{ s_y \le (s_{\setminus y})_{[k]}} } \\
        & = \left\{\begin{array}{cl}
                        -1 + \sum_{k \leq K} { \I{ s_y \le s_{[k]}} } , & s_y > (s_{\setminus y})_{[K]}, \\ [5pt]
                        K, & \text{otherwise}
                    \end{array}\right. \\
        & = -1 + \I{s_y \le (s_{\setminus y})_{[K]}} + \sum_{k \leq K} { \I{ s_y \le s_{[k]}} } \\
        & = -1 + \I{s_y \le s_{[K + 1]}} + \sum_{k \leq K} { \I{ s_y \le s_{[k]}} } \\
        & = -1 + \sum_{k \leq K + 1} { \I{ s_y \le s_{[k]}} }
    \end{aligned}$$
    Then the proof ends.
\end{proof}

\section{Consistency Analysis}
\subsection{The Bayes Optimal Function for AUTKC (Proof of Thm. \ref{thm:bayes_optimal})}
\label{sec_app:bayes_optimal}
\bayesoptimal*
\begin{proof}
    Since the samples are independently drawn, we define the conditional risk:
    \begin{equation}
        \mathcal{R}_K(f, \eta(\boldsymbol{x})) = \E{y | \boldsymbol{x} \sim \eta(\boldsymbol{x})}{aerr_K(y, f(\boldsymbol{x}))}.
    \end{equation}
    Then, we only need to show that given any sample $\boldsymbol{x}$, if $\mathsf{RP}_K(f^*(\boldsymbol{x}), \eta(\boldsymbol{x}))$ and $\lnot \mathsf{RP}_K(f(\boldsymbol{x}), \eta(\boldsymbol{x}))$, then $\mathcal{R}_K(f^*, \eta(\boldsymbol{x})) < \mathcal{R}_K(f, \eta(\boldsymbol{x}))$. For the sake of conciseness, \ul{we denote $\eta (\boldsymbol{x})$ as $\eta$ in the following discussion}, if there exists no ambiguity.
    According to the definition of conditional risk, we have:
    \begin{equation}
        \begin{split}
            & \mathcal{R}_K(\boldsymbol{s}, \eta) = \E{y | \boldsymbol{x} \sim \eta}{ \frac{1}{K} \sum_{k=1}^{K} \I{\pi_{\boldsymbol{s}}(y) > k}} \\
            & = \frac{1}{K} \sum_{y} \eta_y \left[\sum_{k=1}^{K} \I{\pi_{\boldsymbol{s}}(y) > k} \right] \\
            & = \frac{1}{K} \sum_{y} \eta_y \cdot \min\left\{ \pi_{\boldsymbol{s}}(y) -1 , K \right\} \\
        \end{split}
    \end{equation}
    Without loss of generality, we assume that $\eta_1 > \eta_2 > \cdots > \eta_C$. Then, for any $\boldsymbol{s}^*$ such that $\mathsf{RP}_K(\boldsymbol{s}^*, \eta)$, we have 
    $$\mathcal{R}_K(\boldsymbol{s}^*, \eta) = \frac{1}{K} \eta^T P_{\boldsymbol{s}^*}$$
    where
    $$P_{\boldsymbol{s}^*} := [ 0, 1, \cdots, K-1, \underbrace{K, \cdots, K}_{(C - K)\text{ elements}} ] \in \mathbb{R}^{C}$$ 
    Then, for any score $\boldsymbol{s}$ with $\lnot \mathsf{RP}_K(\boldsymbol{s}, \eta)$, we have
    $$\mathcal{R}_K(\boldsymbol{s}, \eta) = \frac{1}{K} \eta^T P_{\boldsymbol{s}},$$
    where $P_{\boldsymbol{s}}$ is a permutation of $P_{\boldsymbol{s}^*}$ according to $\pi_{\boldsymbol{s}}(y)$. Notice that $P_{\boldsymbol{s}}$ can be represented as the combination of a series of swapping operated on $P_{\boldsymbol{s}^*}$, which is defined as follows:

    \begin{definition}[Top-$K$ ranking-swapped]
        For any $y_1, y_2 \in \mathcal{Y}$ such that $\eta_{y_1} > \max \{\eta_{y_2}, \eta_{[K]}\}$, given $\boldsymbol{s}_1$ such that $\pi_{\boldsymbol{s}_1}(y_1) < \pi_{\boldsymbol{s}_1}(y_2)$, if $\boldsymbol{s}_2$ satisfies:
        \begin{itemize}
            \item $\pi_{\boldsymbol{s}_2}(y_1) = \pi_{\boldsymbol{s}_1}(y_2)$ and $\pi_{\boldsymbol{s}_2}(y_2) = \pi_{\boldsymbol{s}_1}(y_1)$;
            \item $\pi_{\boldsymbol{s}_2}(y) = \pi_{\boldsymbol{s}_1}(y)$ for any $y \notin \{y_1, y_2\}$,
        \end{itemize}
        then we say $\boldsymbol{s}_2$ is \ul{top-$K$ ranking-swapped} with respect to $\boldsymbol{s}_1$, denoted as $\mathsf{RS}_K(\boldsymbol{s}_2, \boldsymbol{s}_1)$.
    \end{definition}

    Next, we only need to show $\mathcal{R}_K(\boldsymbol{s}_1, \eta) \le \mathcal{R}_K(\boldsymbol{s}_2, \eta)$ if $\mathsf{RS}_K(\boldsymbol{s}_2, \boldsymbol{s}_1)$, with equality if and only if $\mathsf{RP}_K(\boldsymbol{s}_2, \boldsymbol{s}_1)$. In other words, any swapping operation for such $y_1$ and $y_2$ will not lead to a smaller conditional risk than $\mathcal{R}_K(\boldsymbol{s}_1, \eta)$. Since $s$ is the combination of some $\mathsf{RS}$ with respect to $\boldsymbol{s}^*$, we could conclude that for any $\boldsymbol{s}$ satisfying $\lnot \mathsf{RP}_K(\boldsymbol{s}_2, \eta)$, $\mathcal{R}_K(\boldsymbol{s}^*, \eta) < \mathcal{R}_K(\boldsymbol{s}, \eta)$ holds.

    To show $\mathcal{R}_K(\boldsymbol{s}_1, \eta) \le \mathcal{R}_K(\boldsymbol{s}_2, \eta)$ if $\mathsf{RS}_K(\boldsymbol{s}_2, \boldsymbol{s}_1)$, we consider the following three situations:
    \begin{itemize}
        \item \textbf{S1: } If $\pi_{\boldsymbol{s}_1}(y_1) < \pi_{\boldsymbol{s}_1}(y_2) \le K$, we have
            \begin{equation}
                \begin{split}
                    & K \cdot \left[\mathcal{R}_K(\boldsymbol{s}_1, \eta) - \mathcal{R}_K(\boldsymbol{s}_2, \eta)\right] \\
                    & = \left[ \eta_{y_1}(\pi_{\boldsymbol{s}_1}(y_1) - 1) + \eta_{y_2}(\pi_{\boldsymbol{s}_1}(y_2) - 1) \right] \\
                    & \phantom{\eta_y(\pi)} - \left[\eta_{y_1} {\color{orange} \min\left\{ \pi_{\boldsymbol{s}_2}(y_1) - 1 , K \right\}} + \eta_{y_2} {\color{blue} \min\left\{ \pi_{\boldsymbol{s}_2}(y_1) - 1 , K \right\}} \right]   \\
                    & = \left[ \eta_{y_1}(\pi_{\boldsymbol{s}_1}(y_1) - 1) + \eta_{y_2}(\pi_{\boldsymbol{s}_1}(y_2) - 1) \right] \\
                    & \phantom{\eta_y(\pi)} - \left[ \eta_{y_2}( {\color{orange} \pi_{\boldsymbol{s}_1}(y_2) - 1} ) + \eta_{y_2}( {\color{blue} \pi_{\boldsymbol{s}_1}(y_1) - 1}) \right] \\
                    & = (\eta_{y_1} - \eta_{y_2})[\pi_{\boldsymbol{s}_1}(y_1) - \pi_{\boldsymbol{s}_1}(y_2)] \\
                    & < 0.
                \end{split}
            \end{equation}

        \item \textbf{S2: } If $\pi_{\boldsymbol{s}_1}(y_1) \le K < \pi_{\boldsymbol{s}_1}(y_2)$, we have 
            \begin{equation}
                \begin{split}
                    & K \cdot \left[\mathcal{R}_K(\boldsymbol{s}_1, \eta) - \mathcal{R}_K(\boldsymbol{s}_2, \eta)\right] \\
                    & = \left[ \eta_{y_1}(\pi_{\boldsymbol{s}_1}(y_1) - 1) + \eta_{y_2} K \right] - \left[ \eta_{y_2} {\color{orange} K} + \eta_{y_2}( {\color{blue} \pi_{\boldsymbol{s}_1}(y_1) - 1}) \right] \\
                    & = (\eta_{y_1} - \eta_{y_2})[\pi_{\boldsymbol{s}_1}(y_1) - K] \\
                    & < 0.
                \end{split}
            \end{equation}

        \item \textbf{S3: } If $K < \pi_{\boldsymbol{s}_1}(y_1) < \pi_{\boldsymbol{s}_1}(y_2)$, we have 
            \begin{equation}
                \begin{split}
                    & K \cdot \left[\mathcal{R}_K(\boldsymbol{s}_1, \eta) - \mathcal{R}_K(\boldsymbol{s}_2, \eta)\right] \\
                    & = \left[ \eta_{y_1} K + \eta_{y_2} K \right] - \left[ \eta_{y_2} {\color{orange} K} + \eta_{y_2} {\color{blue} K} \right] \\
                    & = 0. \\
                \end{split}
            \end{equation}
    \end{itemize}
    Note that in \textbf{S3}, $\mathsf{RP}_K(\boldsymbol{s}_2, \boldsymbol{s}_1)$ holds, then the proof ends with the transitivity of $\mathsf{RP}_K$.
\end{proof}

\subsection{Sufficient Condition for AUTKC Calibration (Proof of Thm. \ref{thm:condition_for_consistency})}
\label{sec_app:condition_calibration}
\conditionforconsistency*
\begin{proof}
    We first define the conditional risk and the optimal conditional risk of the surrogate loss:
    \begin{equation}
        \mathcal{R}_K^{\ell}(\boldsymbol{s}, \eta) = \E{y | \boldsymbol{x} \sim \eta}{ \frac{1}{K} \sum_{k=1}^{K+1} \ell \left(s_y - s_{[k]}\right) }.
    \end{equation}
    \begin{equation}
        \mathcal{R}_K^{\ell, *}(\eta) := \inf_{\boldsymbol{s} \in \mathbb{R}^C} \mathcal{R}_K^{\ell}(\boldsymbol{s}, \eta)
    \end{equation}
    Then, we prove the theorem by the following steps.

    \noindent \rule[2pt]{\linewidth}{0.1em}\\
    \textbf{Claim 1.} If $\boldsymbol{s}^* \in \arg \inf_{\boldsymbol{s}} \mathcal{R}_K^\ell(\boldsymbol{s}, \eta)$, then $\mathsf{RP}(\boldsymbol{s}^*, \eta)$.\\
    \noindent \rule[2pt]{\linewidth}{0.1em}

    We show this claim by showing that if $\lnot \mathsf{RP}(\boldsymbol{s}, \eta)$, then $\mathcal{R}_K^{\ell}(\boldsymbol{s}, \eta) > \mathcal{R}_K^{\ell, *}(\eta)$. Note that two cases will lead to $\lnot \mathsf{RP}(\boldsymbol{s}, \eta)$:

    \textbf{Case 1:} If ties exist in $\{s_{[1]}, \cdots, s_{[K]}, s_{[K + 1]}\}$, the chaim is obtained by acontradiction. Given $\boldsymbol{s}$ such that $\pi_{\boldsymbol{s}}(y_1), \pi_{\boldsymbol{s}}(y_2) \le K + 1$ and $s_{y_1} = s_{y_2}$, we assume that $\mathcal{R}_K^{\ell}(\boldsymbol{s}, \eta) = \mathcal{R}_K^{\ell, *}(\eta)$. According to the first-order condition, we have:
    $$
        \frac{\partial}{\partial s_y} \mathcal{R}_K^{\ell}(\boldsymbol{s}, \eta) = 0, y = y_1, y_2.
    $$
    That is, 
    $$\begin{aligned}
        \underbrace{\eta_{y_1} \sum_{k=1, [k] \neq y_1}^{K+1} \ell'(s_{y_1} - s_{[k]})}_{(I)} & = \underbrace{\sum_{y \neq y_1} \eta_{y} \ell'(s_y - s_{y_1})}_{(II)}; \\
        \underbrace{\eta_{y_2} \sum_{k=1, [k] \neq y_2}^{K+1} \ell'(s_{y_2} - s_{[k]})}_{(III)} & = \underbrace{\sum_{y \neq y_2} \eta_{y} \ell'(s_y - s_{y_2})}_{(IV)},
    \end{aligned}$$
    where $[k] \neq y_1$ means that the calculation of the derivative will be skipped when $\pi_{\boldsymbol{s}}(y_1) = k$. Since $s_{y_1} = s_{y_2}$, we have
    $$\begin{aligned}
        (I) & = \eta_{y_1} \ell'(0) + \eta_{y_1} \sum_{k=1, [k] \notin \{y_1, y_2\}}^{K+1} \ell'(s_{y_1} - s_{[k]}); \\
        (II) & = \eta_{y_2}\ell'(0) + \sum_{y \notin \{y_1, y_2\}} \eta_{y} \ell'(s_y - s_{y_1}); \\
        (III) & = \eta_{y_2} \ell'(0) + \eta_{y_2} \sum_{k=1, [k] \notin \{y_1, y_2\}}^{K+1} \ell'(s_{y_1} - s_{[k]}); \\
        (IV) & = \eta_{y_1}\ell'(0) + \sum_{y \notin \{y_1, y_2\}} \eta_{y} \ell'(s_y - s_{y_1}); \\
    \end{aligned}$$
    Then, since $(I) - (III) = (II) - (IV)$, we have 
    $$
        (\eta_{y_1} - \eta_{y_2}) \left[ \ell'(0) + \sum_{y \notin \{y_1, y_2\}} \eta_{y} \ell'(s_y - s_{y_1}) \right]= (\eta_{y_2} - \eta_{y_1})\ell'(0), 
    $$
    which leads to a contradiction since $\eta_{y_2} \neq \eta_{y_1}$ and $\ell'(t) < 0$.

    \textbf{Case 2:} If no ties in $\{s_{[1]}, \cdots, s_{[K]}, s_{[K + 1]}\}$, for any $\boldsymbol{s}$ with $\lnot \mathsf{RP}_{K}(\boldsymbol{s}, \eta)$, we can construct an $\boldsymbol{s}^*$ with $\mathsf{RP}_{K}(\boldsymbol{s}^*, \eta)$ by a series of top-$K$ ranking resuming operation, that is, the inverse operation of the top-$K$ ranking-swapping operation. Since $\mathcal{R}_K^{\ell} (\boldsymbol{s}^*, \eta) \ge \mathcal{R}_K^{\ell, *}(\eta)$, we only need to show any top-$K$ ranking resuming operation will lead to a smaller conditional risk. Specifically, for any $\boldsymbol{s}_1, \boldsymbol{s}_2$ satisfying
    \begin{itemize}
        \item $\eta_{y_1} > \eta_{y_2}$ and $s_{1, y_1} > s_{1, y_2}$;
        \item $s_{2, y_1} = s_{1, y_2}$ and $s_{2, y_2} = s_{1, y_1}$;
        \item $s_{2, y} = s_{1, y}$ for any $y \notin \{y_1, y_2\}$.
    \end{itemize}
    we need to show $\mathcal{R}_K^{\ell}(\boldsymbol{s}_1, \eta) < \mathcal{R}_K^{\ell}(\boldsymbol{s}_2, \eta)$. On top of these properties, we have  $s_{1, [k]} = s_{2, [k]}$ for any $k \in [C]$. Then,
    $$\begin{aligned}
        & K \cdot \left[\mathcal{R}_K^{\ell}(\boldsymbol{s}_1, \eta) - \mathcal{R}_K^{\ell}(\boldsymbol{s}_2, \eta)\right] \\
        & = \left[ \eta_{y_1} \sum_{k=1}^{K+1} \ell(s_{1, y_1} - s_{1, [k]}) + \eta_{y_2} \sum_{k=1}^{K+1} \ell(s_{1, y_2} - s_{1, [k]}) \right] - \left[ \eta_{y_1} {\color{orange} \sum_{k=1}^{K+1} \ell(s_{2, y_1} - s_{2, [k]})} + \eta_{y_2} {\color{blue} \sum_{k=1}^{K+1} \ell(s_{2, y_2} - s_{2, [k]}) }  \right] \\
        & = \left[ \eta_{y_1} \sum_{k=1}^{K+1} \ell(s_{1, y_1} - s_{1, [k]}) + \eta_{y_2} \sum_{k=1}^{K+1} \ell(s_{1, y_2} - s_{1, [k]}) \right] - \left[ \eta_{y_1} {\color{orange} \sum_{k=1}^{K+1} \ell(s_{1, y_2} - s_{1, [k]})} + \eta_{y_2} {\color{blue} \sum_{k=1}^{K+1} \ell(s_{1, y_1} - s_{1, [k]})} \right] \\
        & = \left( \eta_{y_1} - \eta_{y_2} \right) \sum_{k=1}^{K+1} \left[ \ell(s_{1, y_1} - s_{1, [k]}) - \ell(s_{1, y_2} - s_{1, [k]}) \right] \\
        & < 0
    \end{aligned}$$
    The inequality holds since $s_{1, y_1} > s_{1, y_2}$ and the surrogate loss $\ell$ is strictly decreasing.

    Then, the \textbf{Claim 1} is completed by \textbf{Case 1} and \textbf{Case 2}.

    \noindent \rule[2pt]{\linewidth}{0.1em}\\
    \textbf{Claim 2.} 
    $$
        \inf_{\boldsymbol{s}: \lnot \mathsf{RP}(\boldsymbol{s}, \eta)} \mathcal{R}_K^\ell(\boldsymbol{s}, \eta) > \inf_{\boldsymbol{s}: \mathsf{RP}(\boldsymbol{s}, \eta)} \mathcal{R}_K^\ell(\boldsymbol{s}, \eta)
    $$
    \noindent \rule[2pt]{\linewidth}{0.1em}

    It is clear that \textbf{Claim 2} follows \textbf{Claim 1}.

    \noindent \rule[2pt]{\linewidth}{0.1em}\\
    \textbf{Claim 3.} For any sequence $\{\boldsymbol{s}_t\}_{t \in \mathbb{N}_+}$, $\boldsymbol{s}_t \in \mathbb{R}^C$,
    $$
        \mathcal{R}_{K}^{\ell}(\boldsymbol{s}_t, \eta) \to \inf_{\boldsymbol{s} \in \mathbb{R}^C} \mathcal{R}_{K}^{\ell}(\boldsymbol{s}, \eta) \Rightarrow \mathcal{R}_{K}(\boldsymbol{s}_t, \eta) \to \inf_{\boldsymbol{s} \in \mathbb{R}^C} \mathcal{R}_{K}(\boldsymbol{s}, \eta).
    $$
    \noindent \rule[2pt]{\linewidth}{0.1em}

    On top of \textbf{Claim 1}, we only need to prove $\mathsf{RP}(\boldsymbol{s}_t, \eta)$ when $t \to \infty$. Define 
    $$
        \delta := \inf_{\boldsymbol{s}: \lnot \mathsf{RP}(\boldsymbol{s}, \eta)} \mathcal{R}_K^\ell(\boldsymbol{s}, \eta) - \inf_{\boldsymbol{s} \in \mathbb{R}^C} \mathcal{R}_K^\ell(\boldsymbol{s}, \eta).
    $$
    According to \textbf{Claim 2}, we have $\infty > \delta > 0$. Suppose that when $t \to \infty$, $\lnot \mathsf{RP}(\boldsymbol{s}_t, \eta)$. Then, there exists a large enough $T$ such that
    $$
        \mathcal{R}_K^\ell(\boldsymbol{s}_t, \eta) - \inf_{\boldsymbol{s} \in \mathbb{R}^C} \mathcal{R}_K^\ell(\boldsymbol{s}, \eta) > \delta,
    $$
    which is contradicts with $\mathcal{R}_{K}^{\ell}(\boldsymbol{s}_t, \eta) \to \inf_{\boldsymbol{s} \in \mathbb{R}^C} \mathcal{R}_{K}^{\ell}(\boldsymbol{s}, \eta)$. This shows that \textbf{Claim 3} holds.

    \noindent \rule[2pt]{\linewidth}{0.1em}\\
    \textbf{Claim 4.} For any sequence $\{f_t\}_{t \in \mathbb{N}_+}$, $f_t \in \mathcal{F}$,
    $$
        \mathcal{R}_{K}^{\ell}(f_t) \to \inf_{f \in \mathcal{F}} \mathcal{R}_{K}^{\ell}(f) \Rightarrow \mathcal{R}_{K}(f_t) \to \inf_{f \in \mathcal{F}} \mathcal{R}_{K}(f).
    $$
    \noindent \rule[2pt]{\linewidth}{0.1em}

    It is clear that \textbf{Claim 4} holds with \textbf{Claim 3} and 
    $$
        \mathcal{R}_{K}^{\ell}(f_t) = \E{\boldsymbol{x}}{ \mathcal{R}_{K}^{\ell}(f_t(\boldsymbol{x}), \eta(\boldsymbol{x})) }
    $$

    Then, the proof of Thm.\ref{thm:condition_for_consistency} ends.
\end{proof}

\subsection{Inconsistency of Hinge Loss (Proof of Thm. \ref{thm:hinge})}
We first present the following simple but useful lemma: 
\label{sec_app:hinge_inconsistent}
\begin{lemma}
    \label{lem:hinge_plus}
    Given constants $a, b > 0$, the function $F(t) = b [1 + t]_+ + a [1 - t]_+$ is strictly increasing when $t > 1$. 
\end{lemma}
\begin{proof}
    It is clear that $F(t) = b (1 + t)$ when $t > 1$, then the proof ends.
\end{proof}

Now, we are ready to complete the proof. The main idea is to construct a non-optimal score $s$ for any Bayes optimal score $\boldsymbol{s}^*$ and further show that the surrogate risk gap $\mathcal{R}_K^{\ell}(\boldsymbol{s}^*, \eta) - \mathcal{R}_K^{\ell}(\boldsymbol{s}, \eta) > 0$ for any $\boldsymbol{s}^*$. The rest proof consists of two steps: (1) remove the term $[\cdot]_+$ in the hinge loss by showing that a necessary condition for the infimum of the gap is $s^*_{y} - s^*_{k} \le 1$ for any $y, k \le K + 1$; (2) show that the gap is not smaller than zero, and the equality holds only when $\boldsymbol{s}^* = \boldsymbol{s}$.
\hingenotconsistent*
\begin{proof}
    Without loss of generality, we assume $\eta_1 > \cdots > \eta_C$, where we omit $\boldsymbol{x}$ for the sake of conciseness. For any $\boldsymbol{s}^*$ such that $\mathsf{RP}(\boldsymbol{s}^*, \eta)$, we can construct $s$ such that $s_1 = \cdots = s_{K+1} = s^*_{[K+1]}$ and $s_i = s^*_i$ for any $i \in \{K + 2, \cdots, C\}$. Then, the proof ends if 
    $$
        \mathcal{R}_K^{\ell}(\boldsymbol{s}^*, \eta) - \mathcal{R}_K^{\ell}(\boldsymbol{s}, \eta)  > 0.
    $$
    To show this, we first remove the term $[\cdot]_+$. Note that 
    $$\begin{aligned}
        K \cdot \mathcal{R}_K^{\ell}(\boldsymbol{s}, \eta) & = (K + 1)\sum_{y=1}^{K+1}\eta_y + \sum_{y = K + 2}^{C} \eta_y \sum_{k=1}^{K+1} [1 + s_{[k]} - s_y]_+ \\
        & = (K + 1)\sum_{y=1}^{K+1}\eta_y + (K + 1) \sum_{y = K + 2}^{C} \eta_y [1 + s^*_{[K+1]} - s^*_y]_+ \\
        & = \underbrace{(K + 1)\sum_{y=1}^{K+1}\eta_y}_{(I)} + \underbrace{(K + 1) \sum_{y = K + 2}^{C} \eta_y \left(1 + s^*_{K+1} - s^*_y\right)}_{(II)} \\
        K \cdot \mathcal{R}_K^{\ell}(\boldsymbol{s}^*, \eta) & = \underbrace{\sum_{y=1}^{K+1}\eta_y \sum_{k=1}^{K+1} [1 + s^*_{k} - s^*_y]_+}_{(III)} + \underbrace{\sum_{y = K + 2}^{C} \eta_y \sum_{k=1}^{K+1} \left( 1 + s^*_k - s^*_y \right)}_{(IV)} \\
    \end{aligned}$$
    On one hand,
    \begin{equation}
        \label{eq:IV_II}
        (IV) - (II)  = \sum_{y = K + 2}^{C} \eta_y \sum_{k=1}^{K+1} \left( s^*_{k} - s^*_{K+1} \right)
    \end{equation}
    is increasing as $\Delta_{y, k} := s^*_{y} - s^*_{k}$ increases when $y < k \le K + 1$. On the other hand, according to Lem.\ref{lem:hinge_plus},
    \begin{equation}
        \label{eq:III}
        (III)  =  \sum_{y=1}^{K+1} \eta_y + \sum_{y=1}^{K+1} \sum_{k = y + 1}^{K+1} \left( \eta_y  [1 + s^*_k - s^*_y]_+ + \eta_k  [1 + s^*_y - s^*_k]_+  \right).
    \end{equation}
    is also increasing \textit{w.r.t.} $\Delta_{y, k}$ when $\Delta_{y, k} > 1$. Combining Eq.(\ref{eq:IV_II}) and Eq.(\ref{eq:III}), we conclude that if exists $y < k \le K + 1$ such that $\Delta_{y, k} > 1$, then 
    $$
        \mathcal{R}_K^{\ell}(\boldsymbol{s}^*, \eta) - \mathcal{R}_K^{\ell}(\boldsymbol{s}, \eta) > \inf_{\boldsymbol{s}^*: \mathsf{RP}(\boldsymbol{s}^*, \eta)} \left[ \mathcal{R}_K^{\ell}(\boldsymbol{s}^*, \eta) - \mathcal{R}_K^{\ell}(\boldsymbol{s}, \eta) \right].
    $$  
    Thus, we next assume $\Delta_{y, k} \le 1$ for any $y, k \le K + 1$. In this case, for any $y \le K$ we have 
    $$\begin{aligned}
        \frac{ \partial \left[ \mathcal{R}_K^{\ell}(\boldsymbol{s}^*, \eta) - \mathcal{R}_K^{\ell}(\boldsymbol{s}, \eta) \right] }{\partial s^*_y} & = - \eta_y + \frac{1}{K} \sum_{k=1, k \neq y}^{C} \eta_k \\
        & = - \eta_y + \frac{1}{K} (1 - \eta_y) \\
        & > - \frac{1}{k+1} + \frac{1}{K} (1 - \frac{1}{K+1}) = 0
    \end{aligned}$$
    Note that $s^*_1 > \cdots > s^*_K > s^*_{K+1}$, which means:
    $$
        \mathcal{R}_K^{\ell}(\boldsymbol{s}^*, \eta) - \mathcal{R}_K^{\ell}(\boldsymbol{s}, \eta) \ge \lim_{s^*_y \to s^*_{[K+1]}, y \le K} \left[ \mathcal{R}_K^{\ell}(\boldsymbol{s}^*, \eta) - \mathcal{R}_K^{\ell}(\boldsymbol{s}, \eta) \right] = 0.
    $$
    The equality holds only when $s^*_y = s^*_{[K+1]}, y \le K$, which contradicts with the fact $\mathsf{RP}(\boldsymbol{s}^*, \eta)$. Thus, we have $\mathcal{R}_K^{\ell}(\boldsymbol{s}^*, \eta) > \mathcal{R}_K^{\ell}(\boldsymbol{s}, \eta)$, which completes the proof.
\end{proof}

\section{Generalization Analysis}
\subsection{Lipschitz Property of the Loss Functions (Proof of Thm. \ref{thm:for_lip} and Corollary \ref{coll:lipschitz_surrogate})}
\label{sec_app:lip_surrogate}

\conditionoflipschitz*
\begin{proof}
    According to definition of $L_K$, we have
    $$\begin{aligned}
        & \left| L_K(\boldsymbol{s}, y) - L_K(\boldsymbol{s}', y) \right|\\
        & = \frac{1}{K} \left| \sum_{k=1}^{K+1} \ell\left( s_{y} - s_{[k]} \right) - \sum_{k=1}^{K+1} \ell\left( s'_{y} - s'_{[k]} \right) \right| \\
        & = \frac{1}{K} \left| \sum_{k=1}^{K+1} \ell\left( s_{y} - \boldsymbol{s} \right)_{[k]} - \sum_{k=1}^{K+1} \ell\left( s'_{y} - \boldsymbol{s}'\right)_{[k]} \right| \\
        & = \frac{1}{K} \left| \max_{k \le K+1} \sum_{k=1}^{K+1} \ell\left( s_{y} - s_{k} \right) - \max_{k \le K+1} \sum_{k=1}^{K+1} \ell\left( s'_{y} - s'_{k} \right) \right| \\
        & \le \frac{1}{K} \max_{k \le K+1} \left| \sum_{k=1}^{K+1} \left[ \ell\left( s_{y} - s_{k} \right) - \ell\left( s'_{y} - s'_{k} \right) \right] \right| \\
        & \le \frac{1}{K} \max_{k \le K+1} \sum_{k=1}^{K+1} \left| \ell\left( s_{y} - s_{k} \right) - \ell\left( s'_{y} - s'_{k} \right) \right| \\
    \end{aligned}$$ 
    The second equality holds since $\ell$ is strictly decreasing; the last equality holds since
    $$\sum_{k=1}^{K}t_{[k]} = \max_{k \le K} \sum_{k=1}^{K}t_{k}, \forall \boldsymbol{t} \in \mathbb{R}^{C};$$
    The first inequality holds since
    \begin{equation}
        \label{eq:triangle_inequality}
        \left|\max \left\{a_{1}, \ldots, a_{K}\right\}-\max \left\{b_{1}, \ldots, b_{K}\right\}\right| \le \max \left\{\left|a_{1}-b_{1}\right|, \ldots,\left|a_{K}-b_{K}\right|\right\}, \quad \forall \boldsymbol{a}, \boldsymbol{b} \in \mathbb{R}^{K}
    \end{equation}
    Since $\ell$ is $L_\ell$-Lipschitz continuous, the last term is bounded by 
    $$\begin{aligned}
        & \frac{ L_\ell }{K} \max_{k \le K+1} \sum_{k=1}^{K+1} \left| \left( s_{k} - s'_{k} \right) - \left( s_{y} - s'_{y} \right) \right| \\
        & \le \frac{ L_\ell }{K} \max_{k \le K+1} \sum_{k=1}^{K+1} \left| \left( s_{k} - s'_{k} \right) \right| + \frac{ L_\ell (K+1) }{K} \left| s_{y} - s'_{y} \right|  \\
        & \le \frac{ L_\ell \sqrt{K+1} }{K} \max_{k \le K+1} \left[ \sum_{k=1}^{K+1} \left( s_{k} - s'_{k} \right)^2 \right]^{\frac{1}{2}} + \frac{ L_\ell (K+1) }{K} \left| s_{y} - s'_{y} \right|  \\
        & \le \frac{ L_\ell \sqrt{K+1} }{K} \left[ \sum_{k=1}^{C} \left( s_{k} - s'_{k} \right)^2 \right]^{\frac{1}{2}} + \frac{ L_\ell (K+1) }{K} \left| s_{y} - s'_{y} \right|  \\
    \end{aligned}$$ 
    Then, the proof ends according to Def.\ref{ass:lipschitz}.
\end{proof}

To obtain Corollary \ref{coll:lipschitz_surrogate}, the following lemma is necessary.
\begin{lemma}[Lipschitz Constant for Softmax Function \cite{DBLP:journals/corr/abs-2107-13171}]
    \label{lem:softmax}
    Given $\boldsymbol{t} \in \mathbb{R}^{C}$, the softmax function $$ \mathsf{soft}_i(\boldsymbol{t}) = \frac{t_i}{\sum_{j=1}^{C}\exp(t_j)} $$ is $\frac{\sqrt{2}}{2}$-Lipschitz continuous \textit{w.r.t.} vector $\ell_2$ norm. 
\end{lemma}

\lipschitzofsurrogate*

\subsubsection{Lipschitz Property of the Square Loss}
\begin{proof}
    It is clear that $\ell_{sq}(t)$ is strictly decreasing in the range $[-1, 1]$, and $\left| \ell'_{sq}(t) \right| \le 4$. Then, the proof ends by Lem.\ref{thm:for_lip} and Lem. \ref{lem:softmax}.
\end{proof}

\subsubsection{Lipschitz Property of the Exp Loss}
\begin{proof}
    It is clear that $\ell_{exp}(t)$ is strictly decreasing in the range $[-1, 1]$, and $\left| \ell'_{exp}(t) \right| \le e$. Then, the proof ends by Thm.\ref{thm:for_lip} and Lem. \ref{lem:softmax}.
\end{proof}

\subsubsection{Lipschitz Property of the Logit Loss}
\begin{proof}
    It is clear that $\ell_{logit}(t)$ is strictly decreasing in the range $[-1, 1]$, and $\left| \ell'_{logit}(t) \right| \le \frac{1}{e \ln 2}$. Then, the proof ends by Lem.\ref{thm:for_lip} and Lem. \ref{lem:softmax}.
\end{proof}

\subsection{Sub-Gaussian Property of Gaussian Complexity (Proof of Proposition \ref{prop:sub_gaussian})}
\label{sec_app:subgaussian}
\subgaussian*
\begin{proof}
    It is clearly $Z$ is a Gaussian distribution with $\mathbb{E}[Z] = 0$. Meanwhile, we have
    $$\begin{aligned}
        Var[Z] & = Var[\frac{1}{\sqrt{nC}} \sum_{i=1}^{n} \sum_{j=1}^{C} (s_{i, j} - s'_{i, j}) g_{i, j} ] \\
        & = \frac{1}{nC} \sum_{i=1}^{n} \sum_{j=1}^{C} (s_{i, j} - s'_{i, j})^2 Var[g_{i, j}]\\
        & = \frac{1}{nC} \sum_{i=1}^{n} \sum_{j=1}^{C} (s_{i, j} - s'_{i, j})^2\\
        & \le \max_{i, j} \left( s_{i, j} - s'_{i, j} \right) ^ 2 \\
        & = d_{\infty, \mathcal{S}}^2 (\boldsymbol{s}, \boldsymbol{s}'),
    \end{aligned}$$
    where the second equality holds since if ${g_{i, j}}, i \in [n], j \in [C] $ are independent from each other, we have
    $$
        Var[\sum_i g_i t_i] = \sum_i t_i^2 Var[g_i];
    $$
    the third equality is due to $Var[g_{i, j}] = 1$. Further, the moment-generating function of $Z$ is
    $$
        \mathbb{E}[e ^ {\lambda Z}] = \mathbb{E}[e ^ {\lambda (Z - \mathbb{E}[Z])}] \le e^{\frac{\lambda^2}{2} Var[Z]} \le e^{\frac{\lambda^2}{2} d_{\infty, \mathcal{S}}^2 (\boldsymbol{s}, \boldsymbol{s}')}
    $$ 
    According to Def.\ref{def:subgaussian}, we could conclude that $Z$ is a sub-Gaussian process with metric being $d_{\infty, \mathcal{S}} (\boldsymbol{s}, \boldsymbol{s}')$.
\end{proof}

\subsection{Generalization Bound of Convolutional Neural Networks (Proof of Thm. \ref{thm:final_bound})}
\label{sec_app:bound_cnn}
According to Proposition \ref{prop:generalization_gaussian} and Proposition \ref{prop:chaining_gaussian}, to obtain the final generalization bound, our task is to bound the covering number of $\mathcal{F}_{\beta, \nu}$. To this end, we first present the Lipschitz property of $\widetilde{\mathcal{F}}_{\beta, \nu}$ in Corollary \ref{coll:tilde_F_beta_nu}. Then, the practical generalization bound for convolutional neural networks is obtained by Lem.\ref{lem:metric_entropy} and Proposition \ref{prop:chaining_bound}.
\subsubsection{The Lipschitz Property of $\widetilde{\mathcal{F}}_{\beta, \nu}$}
\begin{definition}
    We say a function set $\mathcal{F}$ is $(B, d)$-Lipschitz parameterized if there is a norm $||\cdot||$ in $\mathbb{R}^d$ and a mapping $\phi$ from the unit ball \textit{w.r.t.} $||\cdot||$ in $\mathbb{R}^d$ to $\mathcal{F}$ such that: for any $\boldsymbol{x} \in \mathcal{X}$, and $\theta$ and $\theta'$ with $||\theta|| \le 1$ and $||\theta'|| \le 1$, the following inequality holds:
    $$
        | \phi(\theta)(\boldsymbol{x}) - \phi(\theta')(\boldsymbol{x}) | \le B ||\theta - \theta'||
    $$
\end{definition}

\begin{lemma}[\cite{DBLP:conf/iclr/LongS20}]
    \label{lem:F_beta_nu}
    $\mathcal{F}_{\beta, \nu}$ is $(B_{\beta, \nu, \chi}, W)$-Lipschitz parameterized, where $B_{\beta, \nu, \chi} := \chi \beta (1 + \nu + \beta / L_{a}) ^ {L_{a}}$, $L_{a} = L_{c} + L_{f}$, and $W$ is the number of parameters in $\mathcal{F}_{\beta, \nu}$.
\end{lemma}

\begin{corollary}
    \label{coll:tilde_F_beta_nu}
    $\widetilde{\mathcal{F}}_{\beta, \nu}$ is also $(B_{\beta, \nu, \chi}, W)$-Lipschitz parameterized.
\end{corollary}

\begin{proof}
    According to Lem.\ref{lem:F_beta_nu}, there is a norm $||\cdot||$ in $\mathbb{R}^d$ and a mapping $\phi$ from the unit ball \textit{w.r.t.} $||\cdot||$ in $\mathbb{R}^d$ to $\mathcal{F}_{\beta, \nu}$ such that: for any $\boldsymbol{x} \in \mathcal{X}$, and $\theta$ and $\theta'$ with $||\theta|| \le 1$ and $||\theta'|| \le 1$, the following inequality holds:
    $$| \phi(\theta)(\boldsymbol{x}) - \phi(\theta')(\boldsymbol{x}) | \le B_{\beta, \nu, \chi} ||\theta - \theta'||.$$
    Let $\widetilde{\phi}  = \boldsymbol{e}_y \circ \phi$, and we have 
    $$\begin{aligned}
        | \widetilde{\phi}(\theta)(\boldsymbol{z}) - \widetilde{\phi}(\theta')(\boldsymbol{z}) | & = | \boldsymbol{e}_y \phi(\theta)(\boldsymbol{x}) - \boldsymbol{e}_y \phi(\theta')(\boldsymbol{x}) | \\
        & \le | \boldsymbol{e}_y | \cdot | \phi(\theta)(\boldsymbol{x}) - \phi(\theta')(\boldsymbol{x}) | \\
        & \le B_{\beta, \nu, \chi} ||\theta - \theta'||
    \end{aligned}$$
    Then the proof ends.
\end{proof}

\subsubsection{The Practical Generalization Bound of $\widetilde{\mathcal{F}}_{\beta, \nu}$}
The following lemma bridges $(B, d)$-Lipschitz parameterized and covering number, and Proposition \ref{prop:chaining_bound} further provides the upper bound of the Gaussian complexity in Eq.(\ref{eq:bound_lip}).
\begin{lemma}[Upper bound of Covering Number \cite{wainwright2019high}]
    \label{lem:metric_entropy}
    If the score function set $\mathcal{F}$ is $(B, d)$-Lipschitz parameterized, then the metric entropy defined on the metric space $(\mathcal{F}, d_{\infty, \mathcal{S}})$ is bounded:
    \begin{equation}
        \log \mathfrak{C}(\mathcal{F}, \epsilon, d_{\infty, \mathcal{S}}) \le d \log{\frac{3B}{\epsilon}}
    \end{equation}
\end{lemma}

\begin{proposition}
    \label{prop:chaining_bound}
    Given the score function $f \in \widetilde{\mathcal{F}}$, if $\widetilde{\mathcal{F}}$ is $(B, d)$-Lipschitz parameterized, then we have:
    \begin{equation}
        \mathfrak{G}_{\tilde{\mathcal{S}}}(\widetilde{\mathcal{F}}) \precsim \mathcal{O}\left(\frac{d \log{BnC}}{\sqrt{nC}}\right),
    \end{equation}
    and 
    \begin{equation}
        \mathfrak{G}_{\tilde{\mathcal{S}}'}(\widetilde{\mathcal{F}}) \precsim \mathcal{O}\left(\frac{d \log{Bn}}{\sqrt{n}}\right),
    \end{equation}
    where $\mathcal{O}$ is the complexity.
\end{proposition}

\begin{proof}
    Applying Lem. \ref{lem:metric_entropy} to Eq. (\ref{equ:chaining_gaussian}), we have
    $$\begin{aligned}
        \mathfrak{G}_{\tilde{\mathcal{S}}}(\widetilde{\mathcal{F}}) & \le \inf_{\alpha \ge 0} \left( C_1\alpha + \frac{C_2}{\sqrt{nC}} \int_\alpha^{1} d \log{\frac{3B}{\epsilon}} d\epsilon \right) \\
        & \le \inf_{\alpha \ge 0} \left( C_1\alpha + \frac{C_2}{\sqrt{nC}} (1 - \alpha) d \log{\frac{3B}{\alpha}} \right) \\
    \end{aligned}$$
    The last equation holds since $d \log{\frac{3B}{\epsilon}}$ is  non-increasing \textit{w.r.t.} $\epsilon$. By setting $\alpha = \frac{1}{\sqrt{nC}}$, we have 
    \begin{equation}
        \mathfrak{G}_{\tilde{\mathcal{S}}}(\widetilde{\mathcal{F}}) \le \frac{1}{\sqrt{nC}} \left[C_1 + C_{2}d(1 - \frac{1}{\sqrt{nC}}) \log{3B\sqrt{nC}}\right].
    \end{equation}
    Similarly, we have
    \begin{equation}
        \mathfrak{G}_{\mathcal{S}}(\widetilde{\mathcal{F}}) \le \frac{1}{\sqrt{n}} \left[C_1 + C_{2}d(1 - \frac{1}{\sqrt{n}}) \log{3B\sqrt{n}}\right].
    \end{equation}
    Then the proof ends.
\end{proof}

\noindent Finally, the practical generalization bound of the convolutional neural networks we discuss is clear:
\boundcnn*
\begin{proof}
    The proof ends by Proposition \ref{prop:generalization_gaussian}, Proposition \ref{prop:chaining_bound} and Corollary \ref{coll:tilde_F_beta_nu}.
\end{proof}

\end{document}